%% file: paper.tex
\documentclass[]{fairmeta}

\usepackage[all]{hypcap}
\usepackage{hyperref}[citecolor=magenta]

\hypersetup{
    colorlinks = true,
    citecolor = {magenta},
}

\usepackage{microtype}
\usepackage{graphicx}
\usepackage{booktabs}
\usepackage{float}

\usepackage{amsmath}
\usepackage{amssymb}
\usepackage{mathtools}
\usepackage{amsthm}
\usepackage{mathrsfs}
\usepackage{nicefrac}
\usepackage{dsfont}
\usepackage{enumitem}
\usepackage{tikz}
\usepackage{wrapfig}
\usepackage{float}
\usepackage[subtle]{savetrees}

\usepackage{xspace}

\usepackage{amsmath}
\usepackage{amssymb}
\usepackage{mathtools}
\usepackage{amsthm}
\usepackage{bbm}
\usepackage{fleqn, tabularx}
\usepackage{multirow}
\usepackage[export]{adjustbox}
\usepackage{colortbl}

\usepackage{algpseudocode}
\usepackage{setspace}

\usepackage{color}
\definecolor{deepblue}{rgb}{0,0,0.5}
\definecolor{deepred}{rgb}{0.6,0,0}
\definecolor{deepgreen}{rgb}{0,0.5,0}
\definecolor{boost_correct_to_correct}{HTML}{66C2A5}
\definecolor{default_correct_to_correct}{HTML}{fc8d62}
\definecolor{dup_correct_to_correct}{HTML}{8da0cb}
\definecolor{new_correct_to_correct}{HTML}{e78ac3}

\newcommand\pythonstyle{\lstset{
basicstyle=\ttfamily\footnotesize,
language=Python,
morekeywords={self, clip, exp, mse_loss, uniform_sample, concatenate, logsumexp},              
keywordstyle=\color{deepblue},
emph={MyClass,__init__},          
emphstyle=\color{deepred},   
stringstyle=\color{deepgreen},
frame=single,                       
showstringspaces=false
}}

\lstnewenvironment{python}[1][]
{
\pythonstyle
\lstset{#1}
}
{}

\newcommand\pythoninline[1]{{\pythonstyle\lstinline!#1!}}

\makeatletter
\def\mathcolor#1#{\@mathcolor{#1}}
\def\@mathcolor#1#2#3{%
  \protect\leavevmode
  \begingroup
    \color#1{#2}#3%
  \endgroup
}
\makeatother

\usepackage[textsize=tiny]{todonotes}
\usepackage{algorithm}
\usepackage{amssymb}

\usepackage{crossreftools}
\usepackage[suppress]{color-edits}

\usepackage{dsfont}
\usepackage{nicefrac}

\newtcolorbox{analysisbox}[1][]{
    enhanced jigsaw,
    colback=white,
    colframe=blue!75!black,
    fonttitle=\bfseries,
    boxsep=5pt,
    left=5pt,
    right=5pt,
    top=5pt,
    bottom=5pt,
    title=#1,
}
\newtcolorbox{promptbox}[2][]{  
listing only,
enhanced,
breakable,
colback=rliableolive!13!white,
colframe=black,
fontupper=\ttfamily,
title=#2,
#1}

\newcommand{\bcirc}[1]{%
  \tikz[baseline=(char.base)]{
    \node[
      circle,
      inner sep=0.4pt, %
      fill=black
    ] (char) {\color{white}\bfseries #1};
  }%
}

\definecolor{rliableolive}{HTML}{BBCC33}
\definecolor{rliableblue}{HTML}{77AADD}
\definecolor{rliablered}{HTML}{EE8866}

\tcbset{
  aibox/.style={
    width=\linewidth,
    top=8pt,
    bottom=4pt,
    colback=rliableolive!13!white,
    colframe=black,
    colbacktitle=black,
    enhanced,
    center,
    attach boxed title to top left={yshift=-0.1in,xshift=0.15in},
    boxed title style={boxrule=0pt,colframe=white,},
  }
}
\newtcolorbox{AIbox}[2][]{aibox,title=#2,#1}
\definecolor{lightblue}{rgb}{0.22,0.45,0.70}%

\usepackage{pifont}
\usepackage{soul}
\definecolor{highlightmistake}{RGB}{255, 179, 179}
\definecolor{highlightcorrect}{RGB}{179, 255, 179}

\newcommand{\highlight}[1]{\textbf{\textit{\textcolor{lightblue}{#1}}}}

\title{
Reuse your FLOPs: Scaling RL on Hard Problems by Conditioning on Very Off-Policy Prefixes
}

\author[1,2,*]{Amrith Setlur}
\author[1]{Zijian Wang}
\author[1, *]{Andrew Cohen}
\author[1,3, *,\dagger]{Paria Rashidinejad}
\author[1,*,\dagger]{Sang Michael Xie}

\affiliation[1]{Meta}
\affiliation[2]{Carnegie Mellon University}
\affiliation[3]{University of Southern California}
\contribution[*]{Work done at Meta}
\contribution[\dagger]{Authors jointly advised this work.}

\input{macro}

\abstract{
\input{sections_v2/abstract}
}

\date{\today}
\correspondence{\email{asetlur@cs.cmu.edu}, \email{zijianwang@meta.com}}

\begin{document}

\maketitle

\input{sections_v2/intro}

\input{sections_v2/prelim}
\input{sections_v2/method}

\input{sections_v2/backgen}

\input{sections/experiments}

\input{sections_v2/relwork}

\input{sections_v2/conclusion}

\bibliographystyle{plainnat}
\bibliography{paper}

\clearpage
\newpage
\beginappendix

\input{sections_v2/appendix}

\end{document}

%% file: macro.tex
\newcommand{\E}{\mathbb{E}}

\newcommand{\cF}{\mathcal{F}}

\newcommand{\by}{\mathbf{y}}
\newcommand{\bx}{\mathbf{x}}

\newcommand{\bs}{\mathbf{s}}

\newcommand{\pit}{\pi^{t}}

\newcommand{\methodname}{PrefixRL}

\newcommand{\Doff}{\mathcal{D}_\mathrm{off}}
\newcommand{\Dpre}{\mathcal{D}_\mathrm{pre}}
\definecolor{blanchedalmond}{rgb}{1.0, 0.92, 0.8}
\definecolor{carmine}{rgb}{0.59, 0.0, 0.09}
\definecolor{lightblue}{rgb}{0.22,0.45,0.70}%

\newtheorem{theorem}{Theorem}[section]
\newtheorem{lemma}[theorem]{Lemma}

\newtheorem{proposition}[theorem]{Proposition}

\newtheorem{assumption}[theorem]{Assumption}

\renewcommand{\mathbf}{\boldsymbol}

\makeatletter
\def\Ddots{\mathinner{\mkern1mu\raise\p@
\vbox{\kern7\p@\hbox{.}}\mkern2mu
\raise4\p@\hbox{.}\mkern2mu\raise7\p@\hbox{.}\mkern1mu}}
\makeatother

\newcommand{\paren}[1]{\left( #1 \right)}

\newcommand{\brck}[1]{\left [ #1 \right ] }
\newcommand{\yix}{\by_{i}^{\bx}}

\newcommand{\KL}{\mathrm{KL}}
\newcommand{\pibase}{\pi^{0}}

\numberwithin{equation}{section}

\definecolor{amaranth}{rgb}{0.9, 0.17, 0.31}
\definecolor{antiquebrass}{rgb}{0.8, 0.58, 0.46}
\definecolor{antiquefuchsia}{rgb}{0.57, 0.36, 0.51}
\definecolor{chromeyellow}{rgb}{0.31, 0.47, 0.26}

\usepackage[dvipsnames]{xcolor}
\definecolor{maj5}{HTML}{2b8cbe}
\definecolor{maj5Imp}{HTML}{084081}

\definecolor{seq5wo}{HTML}{d95f0e}
\definecolor{seq5woImp}{HTML}{662506}

\definecolor{seq5w}{HTML}{6a51a3}
\definecolor{seq5wImp}{HTML}{3f007d}

\definecolor{selfwo}{HTML}{d95f0e}
\definecolor{selfwoImp}{HTML}{662506}

\definecolor{selfw}{HTML}{6a51a3}
\definecolor{selfwImp}{HTML}{3f007d}

\definecolor{glorewo}{HTML}{d95f0e}
\definecolor{glorewoImp}{HTML}{662506}

\definecolor{glorew}{HTML}{6a51a3}
\definecolor{glorewImp}{HTML}{3f007d}

\definecolor{vstar}{HTML}{d95f0e}
\definecolor{vstarImp}{HTML}{662506}

%% file: sections_v2/abstract.tex
Typical reinforcement learning (RL) methods for LLM reasoning waste compute on hard problems, where correct \emph{on-policy} traces are rare, policy gradients vanish, and learning stalls. To bootstrap more efficient RL, we consider reusing old sampling FLOPs (from prior inference or RL training) in the form of \emph{off-policy} traces. Standard off-policy methods supervise against off-policy data, causing instabilities during RL optimization. We introduce \textbf{PrefixRL}, where we \emph{condition} on the prefix of successful off-policy traces and run on-policy RL to complete them, side-stepping off-policy instabilities. PrefixRL boosts the learning signal on hard problems by modulating the difficulty of the problem through the off-policy prefix length. We prove that the PrefixRL objective is not only consistent with the standard RL objective but also more sample efficient. Empirically, we discover \textbf{back-generalization}:  training \emph{only} on prefixed problems generalizes to \emph{out-of-distribution} unprefixed performance, with learned strategies often differing from those in the prefix.  In our experiments, we source the off-policy traces by rejection sampling with the base model, creating a self-improvement loop. On hard reasoning problems, PrefixRL reaches the same training reward 2$\times$ faster than the strongest baseline (SFT on off-policy data then RL), even after accounting for the compute spent on the initial rejection sampling, and increases the final reward by 3$\times$. The gains transfer to held-out benchmarks, and PrefixRL is still effective when off-policy traces are derived from a different model family, validating its flexibility in practical settings.

%% file: sections_v2/intro.tex
\vspace{-0.1in}
\section{Introduction}
\label{sec:introduction}
\vspace{-0.1in}

Reinforcement learning (RL) is the de facto method to boost large language model (LLMs) reasoning, especially for math and coding~\citep{Polaris2025,liu2025prorlprolongedreinforcementlearning,guo2025deepseek}.
Most successful RL recipes~\citep{ahmadian2024basicsrevisitingreinforcestyle,yu2025dapoopensourcellmreinforcement} are \emph{on-policy}: sample multiple reasoning traces (rollouts) from the current model and derive updates from correct (and incorrect) traces. This paradigm breaks down on \emph{hard problems} with low pass@$k$ (e.g., pass@2k $\approx$ 0), where the model rarely samples a correct trace.
In this regime, learning stalls as RL spends enormous amounts of sampling FLOPs without receiving any learning signal, and RL rewards plateau.

In practice, we are rarely solving these problems for the first time 
since earlier RL runs or inference on previous models may have spent compute on the same (or similar) hard problems.
The natural question now is how to reuse this ever-growing dataset of \emph{off-policy} traces, which often contains some correct traces even for very hard problems, in order to guide the online RL policy towards higher-rewarding states and accelerate on-policy RL.

A straightforward approach is to treat the off-policy traces as supervision: perform supervised fine-tuning (a.k.a., mid-training or continued pretraining) on the correct off-policy traces followed by standard on-policy RL~\citep{wang2025octothinkermidtrainingincentivizesreinforcement}. However, SFT on a small set of correct traces can lead to memorization~\citep{chu2025sft} and entropy collapse, which hurts exploration during subsequent RL~\citep{zhang2025onpolicyrlmeetsoffpolicy}.
Alternatively, we can use off-policy traces directly in RL via importance weighting, but this is often unstable due to high-variance gradient estimates~\citep{liuli2025,yan2025learningreasonoffpolicyguidance}. Both options use off-policy traces as target supervision, and since these off-policy traces are very low probability under the RL policy, this leads to suboptimal RL optimization.

\begin{figure}
    \centering
    \includegraphics[width=0.95\linewidth]{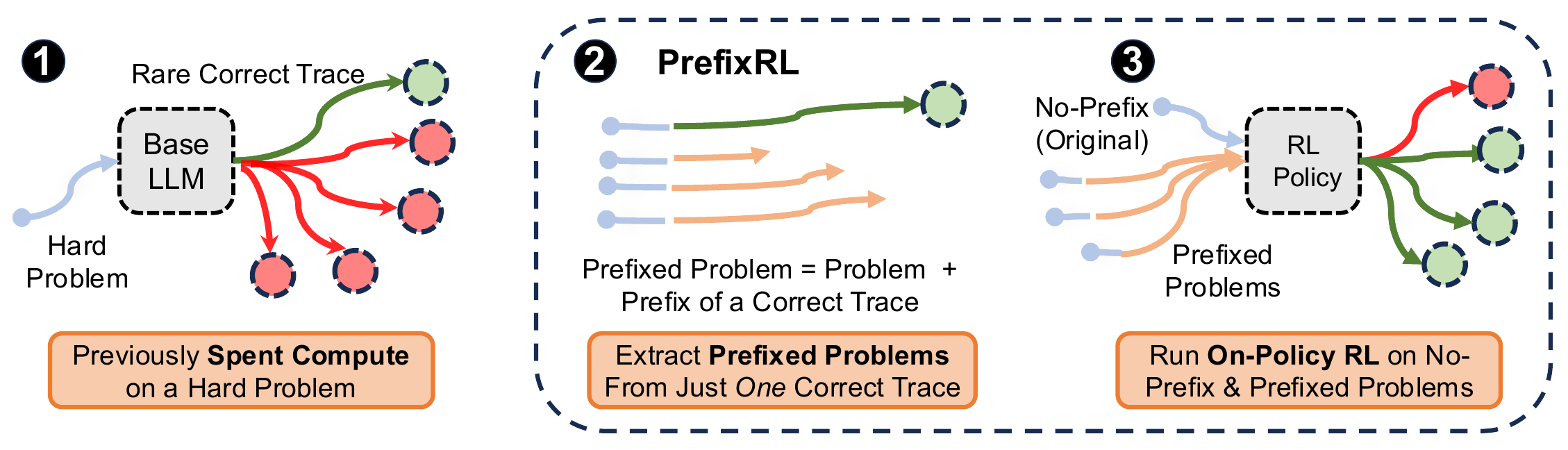}
    \vspace{-0.1in}
    \caption{\textbf{PrefixRL: On-Policy RL Conditioned on Off-Policy Prefixes.} 
    We leverage previously spent compute  (\protect\bcirc{1}) on hard problems in the form of correct off-policy traces rejection sampled from the base LLM we start RL from. Off-policy traces could also come from other model families or previous RL  runs.  We append prefixes of a single correct off-policy trace per problem to the original problem, creating prefixed problems (\protect\bcirc{2}). Then, we simply run on-policy RL on prefixed and no-prefix (original) problems (\protect\bcirc{3}). PrefixRL places the RL policy in higher-rewarding states, which boosts the learning signal. Performance transfers from the prefixed to no-prefix problems via a phenomenon we call back-generalization.    
    } 
    \label{fig:intro_figure}
\end{figure}

To avoid these pitfalls, we propose \highlight{PrefixRL}: run on-policy RL \emph{conditioned} on prefixes of correct off-policy traces instead of directly supervising on them (Figure~\ref{fig:intro_figure}). \emph{First}, we extract and fix a few off-policy prefixes and append them to the original problem to create \emph{prefixed problems}.
\emph{Second}, we run on-policy RL on both \emph{no-prefix} (original) problems and prefixed problems, where gradients are masked on the off-policy prefix. 
The prefixes extracted from correct off-policy traces place the current RL policy in states that are more likely to lead to a correct answer on hard problems, reducing gradient variance and increasing the strength of the learning signal.

\highlight{PrefixRL is consistent with and more sample-efficient than standard RL.} However, it is not immediately clear what the effect on the bias is.
In Section~\ref{subsec:theory}, we prove that when the prefixes are correct and realizable in the model class, (i) maximizers of the PrefixRL objective also maximize performance on the standard RL objective; and (ii) since the prefixes lessen the exploration burden, PrefixRL reduces suboptimality gap with less samples compared to standard RL (by a factor of context length). Overall, PrefixRL changes the on-policy RL objective by using off-policy prefixes solely to guide exploration and unblock training on hard problems.

\highlight{Back-generalization.} 
Beyond the theory, we empirically find an additional phenomenon behind the gains in PrefixRL we call \emph{back-generalization}, where on-policy RL on \emph{only} prefixed problems substantially boosts test performance on the original no-prefix problems, which were never trained on.
Beyond the generalization in the face of train/test mismatch, back-generalization is distinctive for two reasons.
First, back-generalization is a type of \textit{generalization via shared parameters} because it alters the next-token distribution on prefixes it was never trained on (impossible in the tabular RL setting).  
Second, we find that \emph{back-generalization can be even more powerful than standard generalization} in RL (transfer across related problems or environments).
We show this in an in-context learning setup, where we run RL on problems prefixed with another problem and reasoning trace in context. We find that training on a problem P1 conditioned on a related problem P2 in context improves generalization from P1 to P2 considerably more than directly running RL on the problem P1 (see Section~\ref{subsec:backgen_icl}).

\highlight{PrefixRL can discover and learn strategies beyond what is provided in the prefix.} 
Interestingly, the model does not simply back-generalize by imitating the off-policy prefix it is conditioned on.
Through controlled experiments, we find that PrefixRL is more compute-efficient than standard RL at amplifying successful strategies and rejecting suboptimal ones, \emph{even} when the suboptimal strategy is {explicitly present} in the off-policy prefix. As a result of observing non-zero rewards (and advantages) more often, we hypothesize that PrefixRL allows the model to more quickly identify the flaws in the suboptimal strategy and use this insight to find a better strategy (see Section~\ref{subsec:backgen_synthetic}).

\begin{figure}
    \centering
    \includegraphics[width=0.44\linewidth]{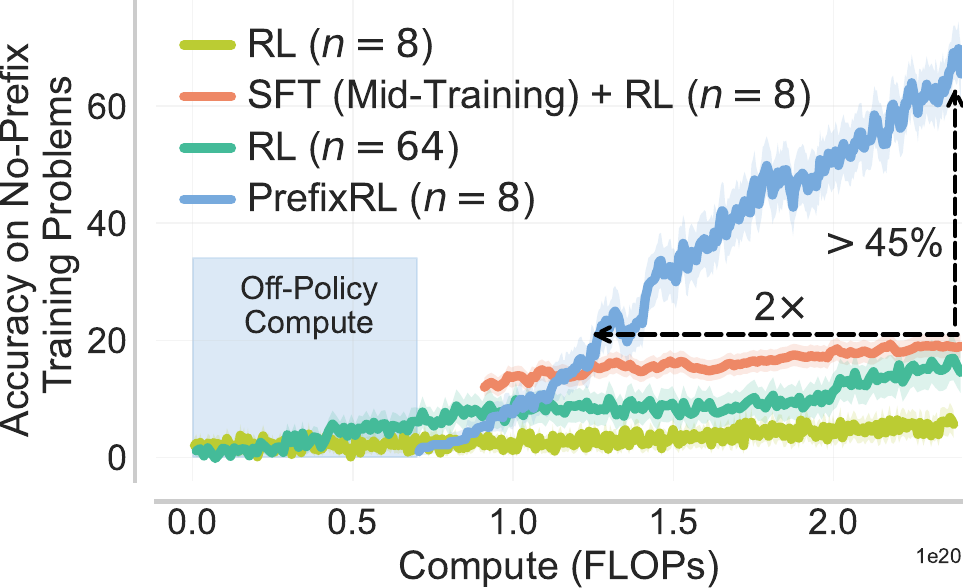}
    \hspace{0.5cm}
    \includegraphics[width=0.44\linewidth]{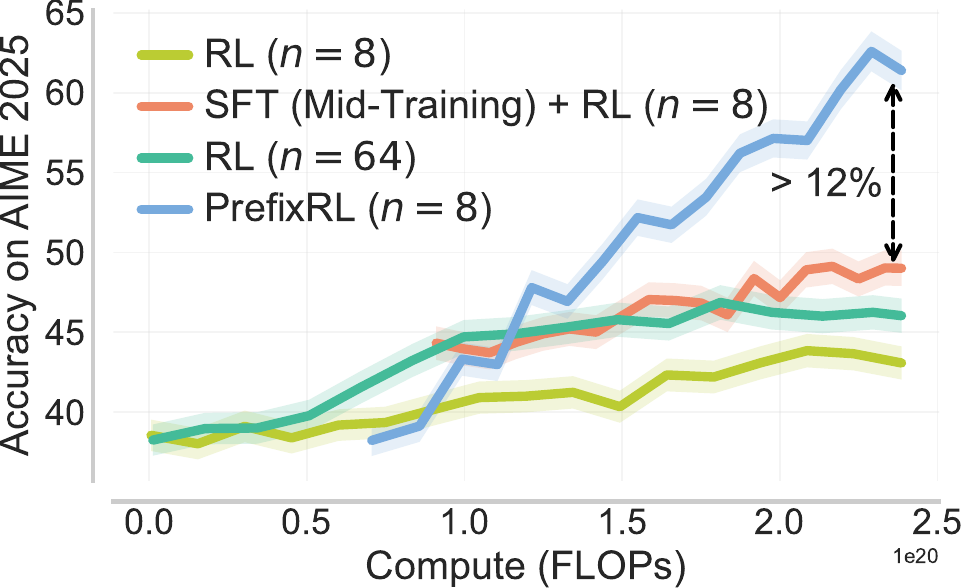}
    \vspace{-0.06in}
    \caption{\textbf{PrefixRL affords a self-improvement pipeline that recycles RL flops on hard problems.}  
    We instantiate PrefixRL for {self-improvement} by collecting a dataset of off-policy traces through large-scale rejection sampling on the base LLM (distilled Llama3.1-8B). 
    In FLOPs-matched training, PrefixRL outperforms the strongest baseline (SFT on rejection-sampled data + RL): 2$\times$ higher compute efficiency (including rejection-sampling cost) and $>$45$\%$ (over 3$\times$ relative)
    higher final training accuracy on no-prefix training problems \emph{(left)}, with gains transferring to standardized evals such as AIME ’25 \emph{(right)}.  
    } 
    \vspace{-0.15in}
    \label{fig:intro_results_figure}
\end{figure}

\highlight{PrefixRL improves both compute efficiency and final performance.} 
In our experiments, we instantiate PrefixRL in a \emph{self-improvement setting} by collecting a dataset of off-policy traces through large-scale rejection sampling on the base model.
On hard problems in training, PrefixRL improves compute-efficiency over the strongest mid-training baseline (SFT on the off-policy data followed by on-policy RL) by $2\times$, even when we account for the initial compute spent on collecting the off-policy traces via rejection sampling, and training accuracy by $>45\%$ (over 3$\times$ relative) on the original no-prefix problems (Figure~\ref{fig:intro_results_figure} \textit{(left)}).
These gains transfer to held-out benchmarks: for example, on AIME '25, PrefixRL improves pass@1 by 12\% over the strongest mid-training baseline in a compute-matched comparison.
Finally, we find that PrefixRL is still effective when off-policy prefixes are sourced from Qwen3-4B-instruct while the RL policy is a distilled Llama-3.1-8B-instruct. This setting provides similar compute and accuracy gains, demonstrating the flexibility of PrefixRL to the off-policy data source and model size.

%% file: sections_v2/prelim.tex
\vspace{-0.1in}
\section{Preliminaries}
\label{sec:prelim}
\vspace{-0.1in}

\highlight{Setup.} We use $\bx$ to denote a problem and $\by=(y_1,\ldots,y_{H})$ for a response of $H$ tokens sampled auto-regressively from an LLM or policy $\pi$ in policy class $\Pi$, where $\by_{:h}$ refers to the prefix consisting of first $h$ tokens in $\by$. We use $\pibase$ to denote the base  (pre-trained) LLM that we want to post-train  on a dataset of hard problems $\mathcal{D} =: \{\bx_i\}_{i=1}^{|\mathcal{D}|}$ with verifiable rewards.  For each $\bx\in \mathcal{D}$, we have an outcome reward  $r(\bx_i, \by)$ which is 1 when $\by$ is correct and 0 when incorrect.
We say $\bx$ is a  \emph{hard problem} for policy $\pi^0$ if pass@512 under $\pi^0$ is $\approx 0$ for $\bx$.  We use $J(\pi) := \E_{\bx\sim \rho} \E_{\by \sim \pi(\cdot \mid \bx)} r(\bx, \by)$ to denote the performance of $\pi$ on $\rho$ which is the empirical distribution over $\cal{D}$. 
For full set of notations, see Appendix~\ref{sec:additional_notation}. Our goal is to train  $\pibase$ on $\cal{D}$,  
to  maximize the return $J(\pi)$  with access to previously spent inference compute on $\pibase$ (or on models fine-tuned from it), available  in the form of correct off-policy traces $\Doff$. 

\highlight{Source of off-policy traces.}
To situate ourselves in a \emph{self-improvement} setup, we mainly source the off-policy traces via rejection sampling done on top of the base model $\pibase$. Concretely, for each $\bx \in \mathcal{D}$, we collect a single correct trace by running rejection sampling on  $\pibase$ until we see a correct trace. 
Therefore, if the pass@1 under $\pibase$ is $p_{\bx}$ on  $\bx$, then in expectation we will sample $\nicefrac{1}{p_{\bx}}$ traces to get a correct one. Doing this for every $\bx \in \cal{D}$ gives us $\Doff$, where $|\Doff| = |\cal{D}|$.   
In theory, we assume that the empirical distribution defined by $\Doff$ can be \emph{realized} (perfectly fitted) by some  $\mu \in \Pi$. 
In practice, $\Doff$ can also be curated with sophisticated inference algorithms that may not be representable by using greater compute depth than the architecture (by scaffolding sequential and parallel compute) or having oracle access. We show the flexibility of the off-policy source in  Section~\ref{sec:experiments}.

\highlight{RL algorithms.} For RL that does not directly train on $\Doff$, we use the REINFORCE~\citep{ahmadian2024basicsrevisitingreinforcestyle} update. For others like importance-weighted off-policy RL see CISPO update in Appendix~\ref{app:additional_baselines}. Starting from $\pi^0$, REINFORCE iteratively updates $\pi^t$ by ascending $J(\pi)$ using the gradient: $\E_{\bx \sim \rho} \E_{\by \sim \pit(\cdot \mid \bx)} \brck{A_{\pit}(\bx, \by) \cdot \log\; \pit(\by \mid \bx)}$.
Following \cite{guo2025deepseek}, we sample $n$ \textit{i.i.d.} traces from $\pit(\by \mid \bx)$ to estimate the expectation in the policy gradient, and use the group baseline $\hat A(\bx,\by_i)=r(\bx,\by_i)-\tfrac{1}{n}\cdot \sum_{j=1}^n r(\bx,\by_j)$ for the advantage. 
On hard problems with pass@$n\approx 0$, all $n$ samples typically fail, so $\hat A$ (and the gradient) is usually zero, yielding the \emph{stalling regime}.

%% file: sections_v2/method.tex
\vspace{-0.1in}
\section{PrefixRL: On-Policy RL on Very Off-policy Prefixes}
\label{sec:method}
\vspace{-0.1in}

\begin{figure}[!t]
    \centering
    \includegraphics[width=0.9\linewidth]{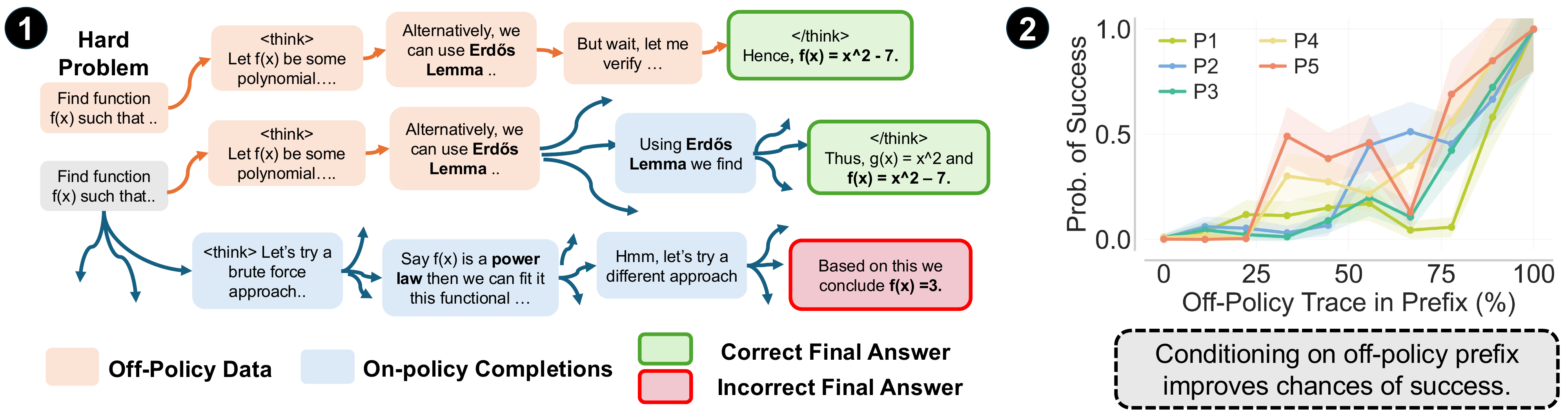}
    \vspace{-0.1in}
    \caption{\textbf{Prefixing on off-policy trace improves probability of future success.} 
    When we condition on prefixed problems, we increase the accuracy by placing the policy at key strategy-revealing states (Erd\"os Lemma in example  \protect\bcirc{1}).
    For five problems (P1-P5) we plot accuracy when conditioning on prefixes of varying lengths, as a proportion of the full off-policy prefix length (\protect\bcirc{2}). 
    The accuracy is near zero until a key state is visited, after which it jumps sharply.
    }
    \label{fig:prefixrl-fig}
    \vspace{-0.2in}
\end{figure}

In this section, we introduce the \highlight{PrefixRL} framework,
which conditions on prefixes from $\Doff$ to guide on-policy exploration, boosting success on prefixed problems and transferring to the original no-prefix problems. We also prove some theoretical properties of the PrefixRL trained model.  Before that, we discuss a more direct approach first. 

\highlight{Directly training on $\Doff$.}
Using $\Doff$ directly via (i) SFT (mid-training) or (ii) importance-weighted off-policy RL~\citep{degris2012off} can degrade exploration and stability in our setting. Warm-starting the RL run with SFT can boost pass@1 but quickly collapses entropy, leading to a drop in pass@64 after $\sim$100 RL iterations (Figure~\ref{fig:midtrain_and_offpolicyrl}(a,b) in Appendix~\ref{app:offpolicy_as_supervision}). While early-stopping SFT may preserve randomness~\citep{wang2025octothinkermidtrainingincentivizesreinforcement}, for these issues to go away we would eventually need a much larger $|\Doff|$ which would raise upfront cost. Off-policy RL on $\Doff$ is often unstable due to large-variance or clipped importance weights, yielding biased gradients~\citep{agarwal2021theory} and potential reward collapse/memorization (Figure~\ref{fig:midtrain_and_offpolicyrl}(c,d))~\citep{kang2024learning,setlur2024rl}.

\highlight{Creating prefixed problems in \methodname{}.} 
Instead of imitating off-policy traces, PrefixRL runs on-policy RL conditioned on off-policy prefixes (in addition to the original no-prefix problems).
Crucially, the gradients are always masked on the off-policy prefix, avoiding the instability of policy gradients on very off-policy tokens.
We create a dataset of prefixed problems $\Dpre$ by taking prefixes from a correct trace $\by^{\bx} \in \Doff$ (for a problem $\bx$) and appending the first $h$ tokens ($\by^{\bx}_{1: h}$) in $\by^{\bx}$ to the original problem $\bx$, creating the prefixed problem $\mathrm{concat}(\bx, \by^{\bx}_{1:h})$ (see Appendix~\ref{subsec:example_prefixed} for an example).
We create multiple prefixed problems for every original problem by choosing multiple values of the prefix lengths $h$. 
We choose the values of $h$ such that, conditioned on the prefix, the base LLM has a reasonable accuracy under base LLM (see Section~\ref{sec:experiments}).
Typically these are states revealing a high-level problem-solving strategy which the base LLM has little probability of sampling on its own (see Figure~\ref{fig:prefixrl-fig}).

\highlight{PrefixRL training objective.}
The \methodname{} objective in \eqref{eq:prefix-rl-obj} optimizes rewards on both prefixed problems in $\cal{D}_\mathrm{pre}$ and no-prefix ones in $\mathcal{D}$ within a maximum context length of $H$ tokens.
Note that the reward $r(\bx, \cdot)$ for any prefixed problem $\bx_\mathrm{pre}$ is identical to its no-prefix counterpart.
{
\setlength{\abovedisplayskip}{4pt}
\setlength{\belowdisplayskip}{4pt}
\begin{align}
    \textbf{\textcolor{lightblue}{(PrefixRL)}} \quad \;\; \max_{\pi} \;\bigg( \underbrace{\sum_{ \bx_\mathrm{pre} \in \Dpre} \;  \E_{\by \sim \pi(\cdot \mid \bx_\mathrm{pre})} \left[r(\bx_\mathrm{pre}, \by)\right]}_{\textbf{prefixed problems}} \;\; + \;\; \underbrace{\sum_{\bx \in \cal{D}} \; \E_{\by \sim \pi(\cdot \mid \bx)} \left[r(\bx, \by)\right]}_{ \textbf{\highlight{Standard RL}: No-Prefix Problems}} \bigg)   
    \label{eq:prefix-rl-obj}
\end{align}
}

\begin{AIbox}{Takeaway: PrefixRL conditions on off-policy traces instead of using them as supervision.}
The correct traces in $\Doff$ are very off-policy and serve as poor supervision targets for SFT or off-policy RL. PrefixRL never trains the RL policy on $\Doff$ and instead runs on-policy RL  conditioned on prefixed problems (original problems appended with off-policy prefixes in $\Doff$) along with the original problems.
\end{AIbox}

\subsection{PrefixRL Objective-Consistency and Sample-Efficiency Guarantees}
\label{subsec:theory}

Prefixing problems with partial correct traces   should place the policy in a higher-rewarding state (Figure~\ref{fig:prefixrl-fig}).
The central theoretical question, however, is not whether learning the prefixed-problems are easier, but instead whether optimizing over such \emph{augmented} problem sets provably improves the performance on the \emph{original} RL objective $J(\pi)$. 
In general, training on an altered input distribution could change the objective away from maximizing $J(\pi)$.

We show that this is \emph{not} the case for \methodname{} as long as the prefixes come from correct traces generated by a realizable policy. Concretely, we prove: (i) \textit{objective consistency:} every maximizer of the PrefixRL objective is also a maximizer of $J(\pi)$; and (ii) \textit{sample complexity guarantees and improvement over online RL:} for a natural policy gradient variant, PrefixRL achieves a smaller suboptimality bound, which translates to a smaller number of on-policy samples required to reach a given reward $J(\pi)$. In other words, we formally show that \textbf{PrefixRL reuses your FLOPs}: it converts information already paid for in logged prefixes into sample-complexity advantages over standard RL.

\highlight{PrefixRL objective is consistent with standard RL.}
We make the following assumption that the prefixes are taken from correct traces and that they are realized by some policy $\mu$, \textit{i.e.}, there is a policy $\mu$ that can perfectly fit $\Doff$.
\begin{assumption}[\textbf{Realizability~and~ correctness~~of~~off-policy~~traces}]
\label{assm:realizable} For any $(\bx, \by) \in \Doff$: (i) the trace is correct: $r(\bx, \by) = 1$, and (ii) all traces are realized by an optimal $\mu$, \textit{i.e.}, there exists an optimal policy $\mu \in \Pi$ s.t. ${\mu}(\by\mid \bx) = 1$.
\end{assumption}
The next theorem states that as long as prefixes are taken from the correct traces realized by a policy $\mu \in \Pi$, optimizing the PrefixRL objective preserves optimality on $J(\pi)$ (performance on no-prefix problems only). 

\begin{theorem}[\textbf{Consistency~~~of~~~the~~~PrefixRL~~~objective}]
    \label{thm:consistency} Assume the realizability and correctness of off-policy traces (Assm. \ref{assm:realizable}). Then, the maximizer of the PrefixRL objective~\eqref{eq:prefix-rl-obj} also maximizes standard RL objective $J(\pi)$. 

\end{theorem}
Intuitively, a good policy (yields high $J(\pi)$) on the no-prefix problems should be able to complete its prefix counterparts and get the same reward.  This is because the prefixes in $\Doff$ are from correct traces and the realizability assumption tells us that there should exist at least one such good policy that can remain optimal on the prefixes  in $\Doff$; thus the two terms in the objective need not be conflicting at optima. 
Note that while \methodname{} does not change the global solution, it does not produce the same gradients as the standard RL objective.

\highlight{PrefixRL is more sample-efficient than standard RL.}
Having established that PrefixRL does not bias policy optimization, we now quantify the benefits of prefixing in terms of the number of on-policy samples needed to reach a near-optimal policy.
We analyze PrefixRL by instantiating the policy update to be natural policy gradient~\citep{kakade2001natural} (PrefixRL-NPG). See Algorithm~\ref{alg:npgalg} and discussion in Appendix~\ref{sec:sample_complexity_proof}.
Concretely, each RL iteration alternates between: (i) \emph{policy evaluation} by fitting critic or $Q$-function $\in \mathcal{F}$ using $N$ on-policy  completions conditioned on the off-policy prefix, and (ii) \emph{policy improvement} using the fitted critic via a mirror-descent update. 
In practice, algorithms like PPO, GRPO approximate the NPG update~\citep{schulman2017ppo}.  
Theorem~\ref{thm:suboptgap} bounds the suboptimality of the policy $\bar \pi_T$ returned by PrefixRL-NPG  (Algorithm~\ref{alg:npgalg}) 
in terms of the number of policy updates $T$, on-policy completions $N$, and a single distribution shift quantity between  base policy $\pi_0$ and policy $\mu$ that realizes $\Doff$.

\begin{theorem}[\textbf{Suboptimality gap of PrefixRL}]
\label{thm:suboptgap} 
Under Assumption~\ref{assm:realizable}, let $\Doff$ be realized by $\mu \in \Pi$. 
For any $\delta \in [0,1]$, with probability at least $1-\delta$, policy $\bar{\pi}_T$ returned by PrefixRL-NPG (Algorithm~\ref{alg:npgalg}) satisfies:
{
\setlength{\abovedisplayskip}{4pt}
\setlength{\belowdisplayskip}{4pt}
\begin{align*}
\max_{\pi \in \Pi} \;\; J(\pi) \; - \; J(\bar{\pi}_T) \;\; \leq \;\; \; \mathcal{O}\paren{ \sqrt{\frac{\mathrm{KL}(\mu ||  \pi_0) }{T}}  \; + \;  \sqrt{\frac{1}{N} \cdot \log\paren{\frac{T|\cal{F}|}{\delta}}}}.
\end{align*}}
\end{theorem}

The bound decomposes into (i) an \emph{optimization} term giving a $1/\sqrt{T}$ convergence rate (with an oracle evaluator, $N\to\infty$), whose constant depends only on the initial shift between the base policy $\pi_0$ and the behavior policy $\mu$ realizing $\Doff$ and does not accumulate over iterations; in particular, if $\Doff$ comes from rejection sampling $\pi_0$ with at most $R$ attempts per problem, then $\mu=\pi_0$ conditioned on success and $\mathrm{KL}(\mu||\pi_0)=\mathcal{O}(\log R)$ so this dependence grows only logarithmically in $R$; and (ii) a \emph{critic approximation} term from fitting $Q_{\pi_t}\in\cF$ using $N$ on-policy traces, which in PrefixRL does not pay an extra distribution-shift penalty because sampling and evaluation use the same reset distribution induced by prefixes from $\Doff$ (unlike standard RL, which still suffers mismatch relative to $\pi^\star$). 

\begin{proposition}[Worst-case separation with standard RL]
    \label{prp:worstcase_gap}
    Let $\bar{\pi}_T^{\mathrm{pre}}$ be the policy returned after $T$ PrefixRL iterations of Algorithm~\ref{alg:npgalg} and $\bar{\pi}_T^{\mathrm{std}}$ be the policy after $T$ iterations of standard RL (states in Algorithm~\ref{alg:npgalg}  $\sim \pi_t$).  
    Then, there exists a reward function $r$ and base LLM $\pibase$ such that $J(\bar{\pi}_T^\mathrm{pre}) - J(\bar{\pi}_T^\mathrm{std}) \geq 1-\paren{TN\cdot e^{-H}}$ for $TN=o(e^{H})$.  
\end{proposition}

\highlight{Worst-case separation with standard RL.} Proposition~\ref{prp:worstcase_gap} shows there exists a reward function and base LLM where on-policy NPG and PrefixRL-NPG separate sharply. Intuitively, this happens when NPG is bottlenecked by exponentially rare (in $H$) rewarding traces, whereas PrefixRL-NPG only needs polynomial (in $H$) samples to learn the sequence of optimal actions in reverse, starting from the longest prefix to the shortest one. Here the exploration needed to learn optimal actions on longer prefixes is outsourced to the off-policy prefix, and by the time the policy is presented with the shorter prefixes it has already learned to complete optimally on longer ones.

%% file: sections_v2/backgen.tex
\begin{figure*}[!t]
    \centering
    \begin{subfigure}[t]{0.32\linewidth}
        \centering
        \includegraphics[width=\linewidth]{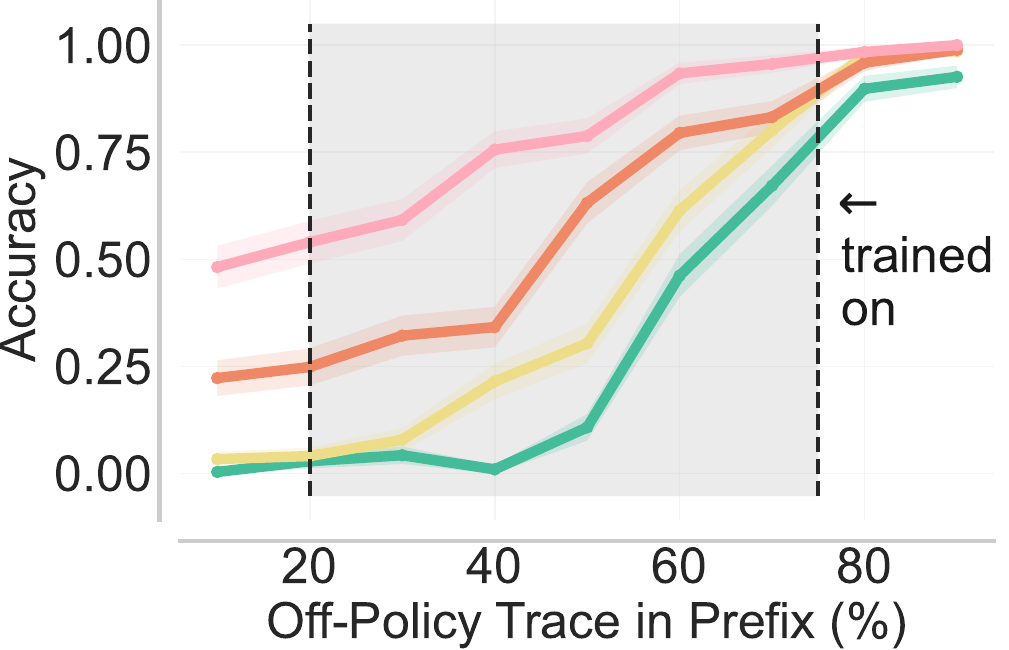}
        \caption*{(a)}
    \end{subfigure}
    \begin{subfigure}[t]{0.32\linewidth}
        \centering
        \includegraphics[width=\linewidth]{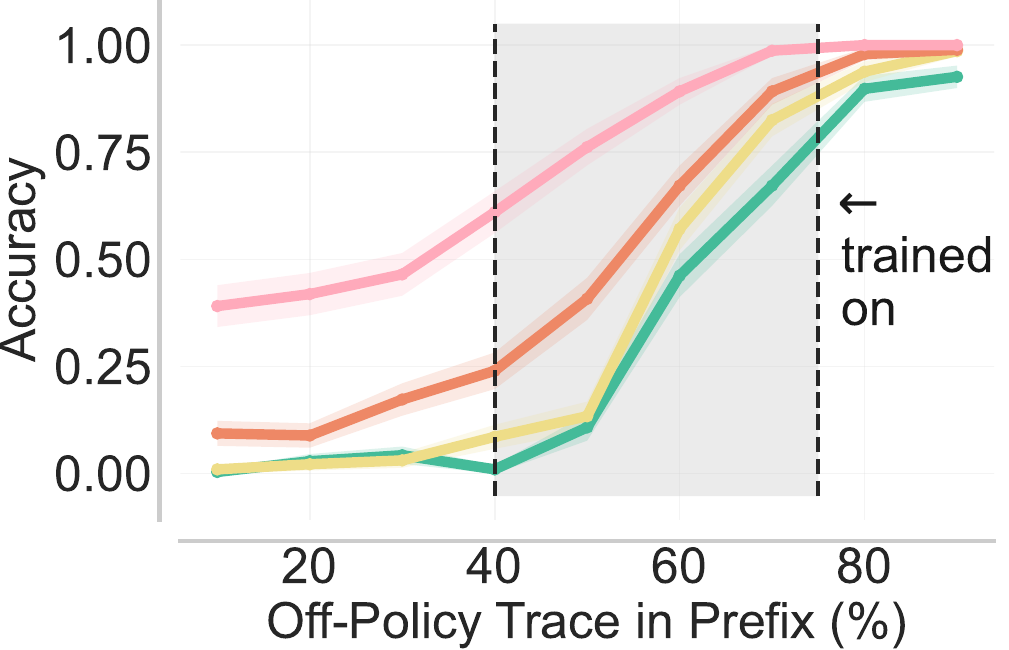}
        \caption*{(b)}
    \end{subfigure}
    \begin{subfigure}[t]{0.32\linewidth}
        \centering
        \includegraphics[width=\linewidth]{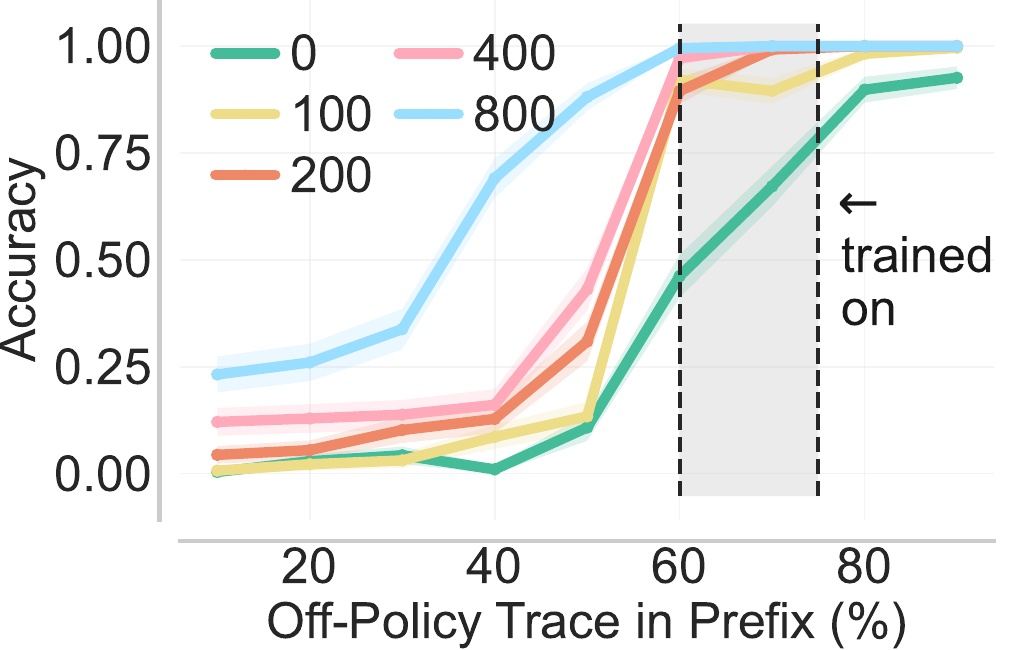}
        \caption*{(c)}
    \end{subfigure}
    \vspace{-0.1in}
    \caption{\textbf{Back-generalization (train/test mismatch):} We run RL \emph{only} on prefixed problems where the prefix length of the off-policy trace lies in the shaded interval. We evaluate different RL training checkpoints across the full range of prefix lengths, including no-prefix problems (0\% prefixing). Training on longer prefixes improves performance on shorter prefixes and can eventually improve no-prefix, indicating back-generalization (a,b). When training uses very long prefixes only there is a severe train/test mismatch, and back-generalization to no-prefix problems takes more training steps (c).}
    \label{fig:back-gen}
    \vspace{-0.15in}
\end{figure*}

\vspace{-0.1in}
\section{Back-Generalization Boosts the Learning Signal in PrefixRL} 
\label{sec:back_generalization}
\vspace{-0.05in}

In Section~\ref{subsec:theory}, we showed that PrefixRL is  consistent with standard RL except more sample-efficient. 
Now, we show that an empirical phenomenon we call \highlight{back-generalization} is a strong source of the gains behind PrefixRL and is unexplained by our theory.
Back-generalization is defined as the performance improvement on no-prefix problems when we train \emph{only} on their prefixed counterparts. Next, we study this with controlled experiments.

\vspace{-0.1in}
\subsection{PrefixRL Improves No-Prefix Performance Even When Training Only on Prefixed-Problems} 
\vspace{-0.05in}

We run {on-policy RL} only on prefixed-problems where the prefix lengths are distributed uniformly between a fixed band of token-length percentiles of the full off-policy trace, but we evaluate accuracy across the full spectrum of prefix lengths, including the no-prefix endpoint (0\% prefixing). 
In Figure~\ref{fig:back-gen}, we see generalization to no-prefix problems despite not having trained on them. This transfer from prefixed to no-prefix problems is particularly notable since the prefixes are highly unlikely under the base policy.
When the training mixture includes relatively short prefixes, the mismatch is moderate (Figure~\ref{fig:back-gen} (a,b)). In this case, performance increases first near the trained band and then progressively improves for shorter prefixes, eventually lifting no-prefix accuracy. 
When training is restricted to very long prefixes (Figure~\ref{fig:back-gen}(c)), the train/test mismatch with no-prefix problems is severe and the transfer is slower, but longer training (\textit{e.g.}, 800 steps) still yields measurable no-prefix gains.

\vspace{-0.05in}
\subsection{PrefixRL can Discover New Strategies Beyond What is Present in the Prefixed Problem}
\label{subsec:backgen_synthetic}
\vspace{-0.05in}

Clearly, back-generalization improves performance on unseen shorter prefixes, but the mechanism behind this is unclear.
To understand this better, we create a simplified setup where we run RL on the prefixes of a single problem:

\highlight{Setup.} On two hard problems, we run PrefixRL on the prefixes of a single off-policy trace in $\Doff$ (prefixed on $>$40\% of the tokens in the trace). 
We use a keyword heuristic to extract a salient ``strategy'' present in the off-policy prefix and track its presence in (i) the prefixed-problem (PP) itself, (ii) the model’s response when conditioned on that prefixed-problem (response for PP), and (iii) the early part of the trace on the original problem (untrained states in the model's response). Figure~\ref{fig:synthetic_expt} probes how strategy usage evolves when running PrefixRL, where the plots in (a) correspond to one hard problem, and (b,c) correspond to a different one. The prefix is sampled from a fixed pool, so frequency in PP is constant (dashed line). In contrast, the response patterns change over training and reveal:

\begin{figure*}[b]
    \centering
    \begin{subfigure}[t]{0.3\linewidth}
        \centering
        \includegraphics[width=\linewidth]{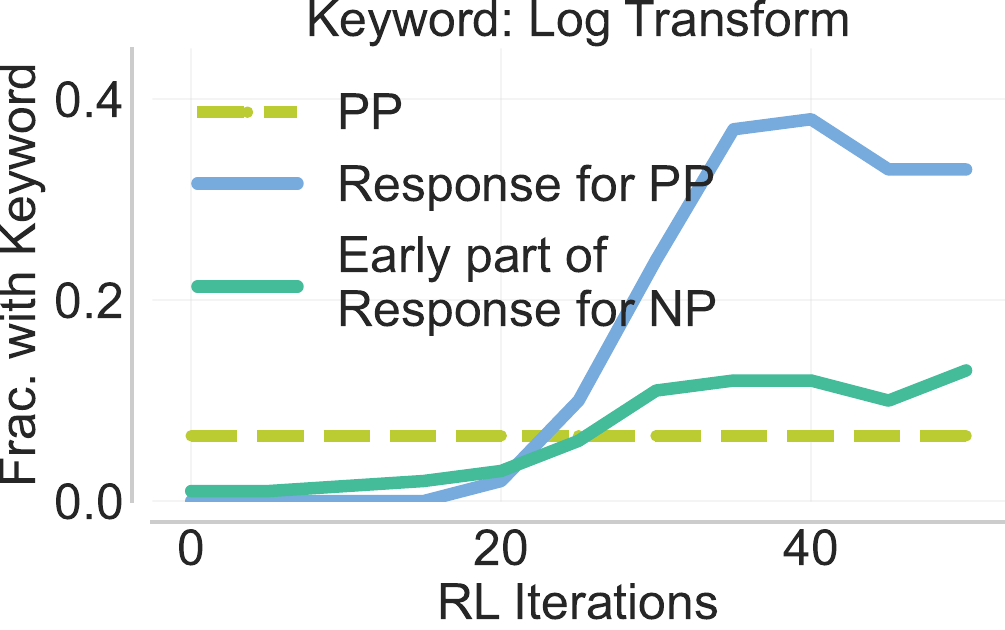}
        \caption*{\vspace{-0.05in}(a)}
    \end{subfigure}
    \begin{subfigure}[t]{0.3\linewidth}
        \centering
        \includegraphics[width=\linewidth]{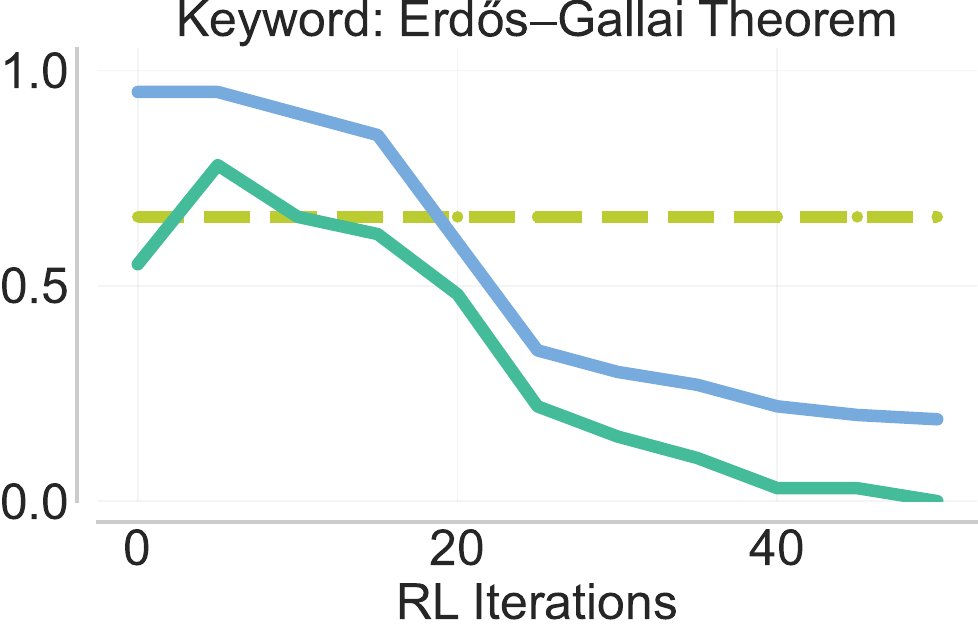}
        \caption*{\vspace{-0.05in}(b)}
    \end{subfigure}
    \begin{subfigure}[t]{0.3\linewidth}
        \centering
        \includegraphics[width=\linewidth]{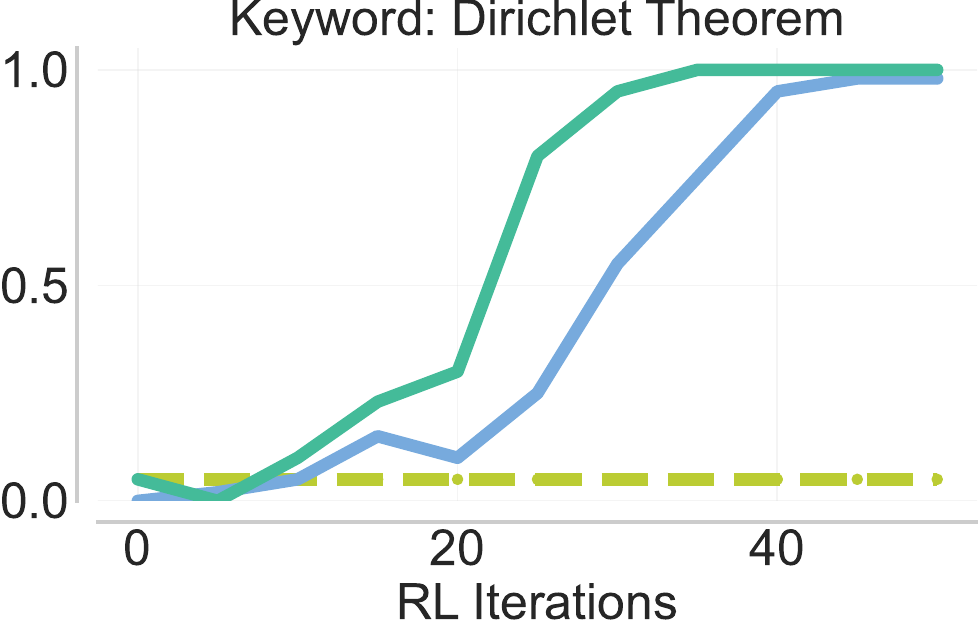}
        \caption*{\vspace{-0.05in}(c)}
    \end{subfigure}
    \vspace{-0.1in}
    \caption{\textbf{\textit{Strong coupling between responses for prefixed and no-prefix problems.}}: We train only on prefixed-problems (PP)  and track the frequency of a  strategy-indicating keyword  in: (i) PP derived from off-policy trace, (ii)   model’s response to PP, and (iii) and the early part (prefixes or states not trained on) of the response to the no-prefix (NP) problem. There is a tight coupling of the strategies present in the responses for PP and NP (througout RL), yet not purely imitative of the strategy explicitly provided in the prefixed-problem itself: the model can learn new strategies or suppress prefixed ones (\textit{e.g.}, Erdős–Gallai). The plots in (a) correspond to one hard problem, and (b,c) correspond to a different one. 
    }
    \label{fig:synthetic_expt}
    \vspace{-0.1in}
\end{figure*}

\highlight{Strategy usage is strongly correlated between prefixed and no-prefix responses.}
There is a tight coupling between strategy use on prefixed and no-prefix responses, which is difficult to explain since no-prefix problems are never trained on and many prefixed problems ($>$90\% in Figure~\ref{fig:synthetic_expt}(a)) may not even contain the keyword. This suggests that PrefixRL updates internal state representations shared across responses for the prefixed and no-prefix problems. This further helps show that back-generalization indeed arises from favorable function approximation in LLMs and is not simply because solutions learned from solving the prefixed problem can be ``stitched'' into the no-prefix version (more in Appendix~\ref{app:backgen_is_not_stitching}), as hypothesized by prior work that prefixes on    hints~\citep{zhang2025bread}.

\highlight{PrefixRL can \emph{unlearn} strategies in the prefixed-problem and discover new ones.} 
In Figure \ref{fig:synthetic_expt}(b), we see that the policy at initialization uses the ``Erdős–Gallai theorem'' close to 90\% of the time on the prefixed-problems since $>$50\% of the prefixed-problems contain references to it.
Throughout training, the frequency of traces mentioning ``Erdős–Gallai'' decreases steadily on responses to PP, indicating that \methodname{} can downweight suboptimal strategies in the off-policy prefixes. For the same hard problem as the one in Figure~\ref{fig:synthetic_expt}(b), in Figure~\ref{fig:synthetic_expt}(c) we note that despite being conditioned on ``Erdős–Gallai'', RL upweights the rare ($<$2\%) strategy at initialization (``Dirichlet Theorem''). This  reweighting also transfers  to the model's behavior on the no-prefix problem. 

\begin{AIbox}{Takeaway: Back-generalization can transfer strategies different from the one in the off-policy prefix.}
    PrefixRL can discover strategies beyond what is present in the prefixed-problems. 
    Benefitting from function approximation, PrefixRL alters the next-token distribution on unseen states; strategies learned and unlearned on the prefixed-problems are quickly reflected in the responses for the no-prefix ones.
\end{AIbox}

\begin{figure*}[t]
    \centering
    \begin{subfigure}[t]{0.42\linewidth}
        \centering
        \includegraphics[width=\linewidth]{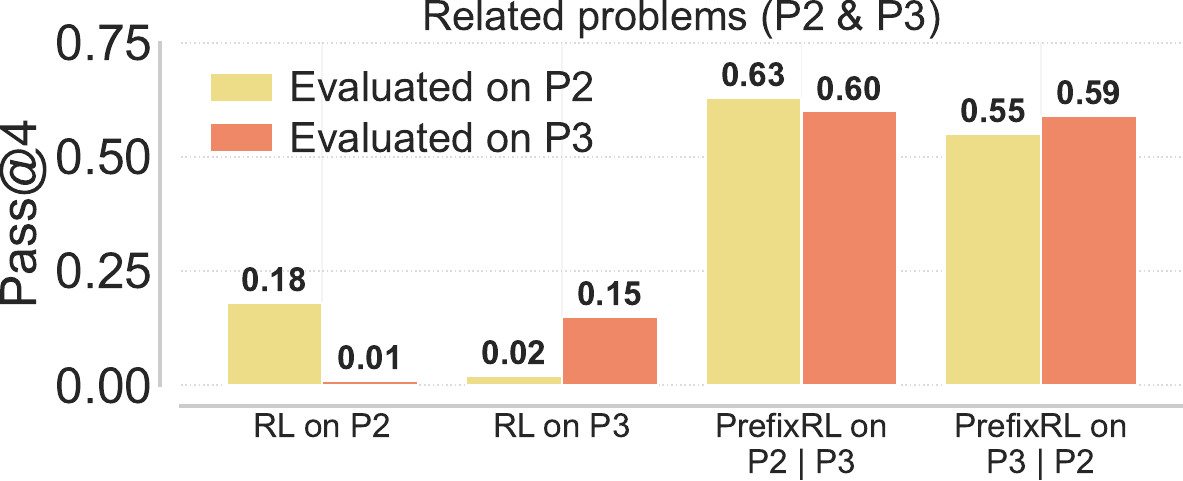}
        \caption*{ \vspace{-0.05in}  (a)}
    \end{subfigure}\hspace{0.6em}
    \begin{subfigure}[t]{0.42\linewidth}
        \centering
        \includegraphics[width=\linewidth]{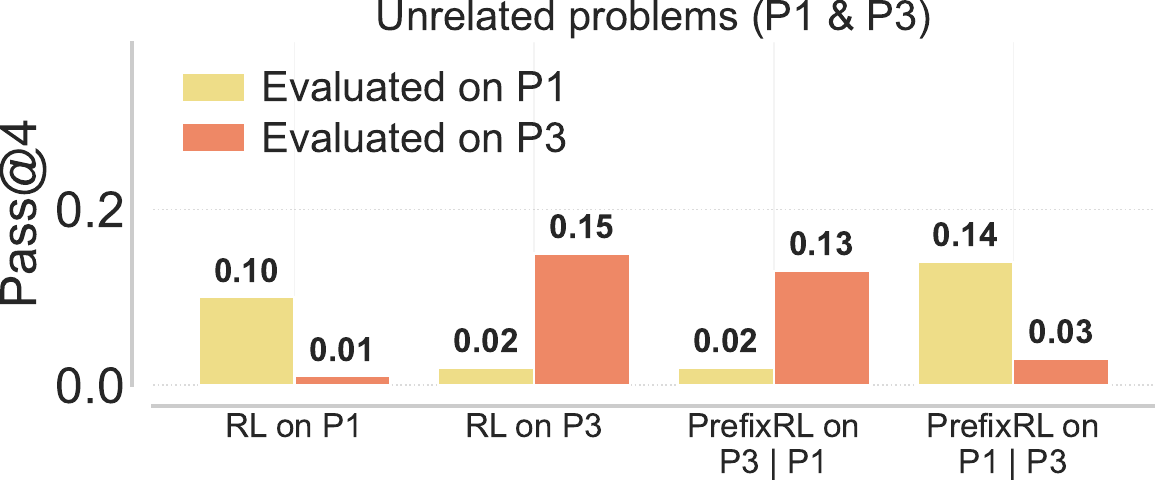}
        \caption*{ \vspace{-0.05in} (b)}
    \end{subfigure}
    \vspace{-0.1in}
    \caption{\textbf{Performance transfer via back-generalization can be stronger than typical generalization in RL:}  When we prefix on the problem and full solution trace of one (in-context) problem (P2), and run Prefix RL  to solve a different but related problem P3 | P2, we are able to improve performance on both P2 and P3 individually, and the performance is much higher compared to running RL either problem individually. 
    The same holds in the opposite direction, when we run PrefixRL on P2 | P3.
    We do not see these gains when the in-context problem is unrelated in the case of P1 and P3. 
    \vspace{-0.15in} 
    }
    \label{fig:icl_expt}
\end{figure*}

\vspace{-0.05in}
\subsection{Which Prefixes Back-Generalize the Most in PrefixRL? Analysis via In-Context Learning}
\label{subsec:backgen_icl}
\vspace{-0.05in}

To study when back-generalization is effective, we analyze it in \emph{in-context learning} setup. We run RL on problems prefixed with another in-context problem, letting us cleanly ablate the relationship between the off-policy prefix and the on-policy suffix based on how related the in-context problem is to the problem being asked to solve.

\highlight{Setup.} We run RL  on a given problem with an entirely different problem (and its solution trace) in  its context or prefix.  
We consider two problem sets: (P1, P3) where P1 and P3 are unrelated problems, and (P2, P3) where P2 and P3 are related and can be solved with the same high level strategy (see Appendix~\ref{app:additional_backgen} for details on P1, P2 and P3). \textit{E.g.}, when we say we run PrefixRL on P1 $\mid$ P3, this means that we prompt the model to solve the problem P1 with P3 along with P3's solution trace in its context.   
All three problems are hard for the base model, with $<$1\% pass@32.

{}

\highlight{Back-generalization occurs when the prefix and suffix are sufficiently related.}
From Figure~\ref{fig:icl_expt}, when the problems are related (P2 and P3), PrefixRL on P2 given P3 achieves 63\% pass@4 on P2 and 60\% pass@4 on the untrained in-context problem P3.
Running standard RL on P2 alone predictably improves the pass@4 of P2 to 18\% but the performance transfer to P3 is limited.
 In contrast, PrefixRL on unrelated problems (P3 $\mid$ P1 or P1 $\mid$ P2) performs similarly to doing RL on just P3 and P1 respectively. 
 This suggests that \emph{back-generalization is more effective when the components of the prefixed problem are related}, and in the in-context learning setting, back-generalization can be stronger than standard generalization across the two related problems. 
This also sheds light on a back-generalization mechanism in standard reasoning, where LLMs often make $k$ attempts at the problem in a single trace. Since the internal representations can be shared across related in-context problems so as to enable back-generalization, this could also be the case across attempts at a problem (see Appendix~\ref{app:mental_model} for a mental model).  

\vspace{-0.02in}
\begin{AIbox}{Takeaway: Back-generalization can be stronger than standard generalization in RL.}
 PrefixRL conditioned on a related problem and its trace transfers performance more effectively than running RL directly on the related problem, perhaps due to responses for prefixed  problems backtracking into states in the off-policy prefix, visiting states closer to those visited when solving the original problem. 
\end{AIbox}

%% file: sections/experiments.tex
\vspace{-0.1in}
\section{Experiments and Results}
\label{sec:experiments}
\vspace{-0.05in}

\begin{figure}[!t]
    \centering
    \begin{subfigure}[t]{0.42\linewidth}
        \centering
        \includegraphics[width=\linewidth]{figures/compute_plot_llama_training.pdf}
        \vspace{-0.25in}
    \end{subfigure}\hspace{1em}
    \begin{subfigure}[t]{0.42\linewidth}
        \centering
        \includegraphics[width=\linewidth]{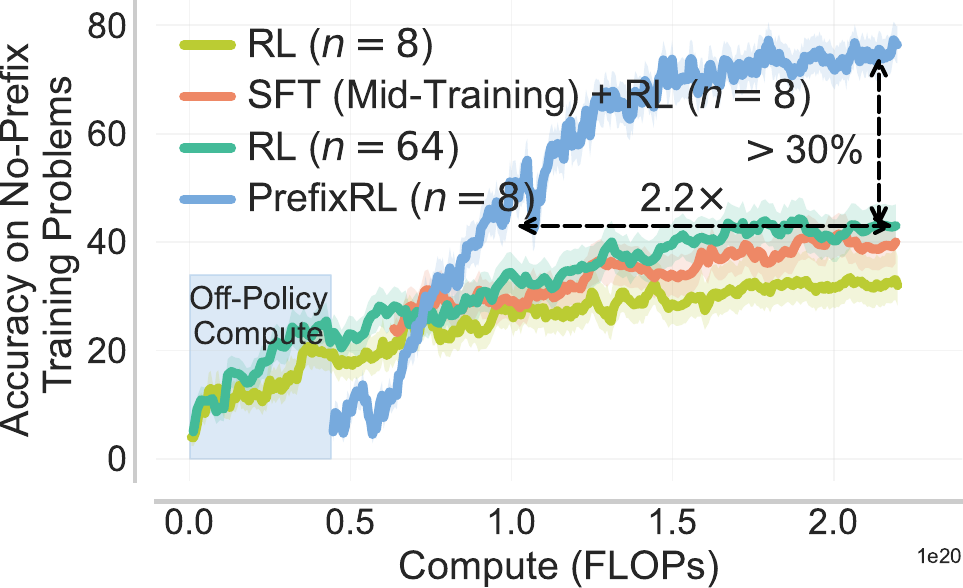}
    \end{subfigure}
    \caption{\textbf{Training compute matched evaluation of \methodname{}.}  
    Using prefixed-problems sourced by rejection sampling the base model, we run \methodname{} on the dataset of hard training problems (low pass rate under the base model) and measure accuracy on the no-prefix training problems for  (\textbf{Left}) Llama-3.1-8b-instruct and (\textbf{Right}) Qwen-3-4b-instruct. 
    Even accounting for the initial compute spent on rejection sampling (shaded blue region), 
    \methodname{} improves compute-efficiency by 2$\times$ over the strongest baseline with and an absolute gain of 45\% with Llama and 30\% with Qwen.}
    \vspace{-0.15in}
    \label{fig:compute_matched_training}
\end{figure}

We evaluate \methodname{} on math reasoning benchmarks: in compute-matched hard-problem training it beats RL and SFT+RL baselines, and these gains transfer to standardized evals. We ablate off-policy data sources, showing robustness to prefixes from different model families, and analyze training dynamics versus  off-policy RL.

\highlight{Experimental setup.}
We experiment with two thinking models: {Llama-3.1-8B-instruct} and {Qwen3-4B-instruct}. 
Since Llama-3.1-8B-instruct is not a thinking model, we distill it on  OpenThoughtsV3 before all experiments, but still refer to it as Llama-3.1-8B ~\citep{guha2025openthoughtsdatarecipesreasoning}.
For training, we select 1k hard problems from {DAPO}~\citep{yu2025dapoopensourcellmreinforcement}  and {OMNI-MATH} (levels 6-8)~\citep{gao2024omnimathuniversalolympiadlevel}, where pass@512 of Llama-3.1-8B is zero.
We compare against on-policy RL~\citep{ahmadian2024basicsrevisitingreinforcestyle} and off-policy baselines that use $\Doff$: SFT (mid-training) on $\Doff$ followed by standard RL (SFT+RL), 
importance-weighted off-policy RL~\citep{mahmood2014weighted}, and LUFFY~\citep{yan2025learningreasonoffpolicyguidance}. All evaluation results, unless noted, are on no-prefix problems.  
See Appendix~\ref{app:additional_experiments} for implementation details. Details on FLOPs accounting for our compute-matched plots are in Appendix~\ref{app:additional_compute_accounting}.

\highlight{Off-policy Dataset $\Doff$ and prefixed-problems.} 
For each base LLM, we produce $\Doff$ using large-scale rejection sampling until there is one correct off-policy trace per problem in $\mathcal{D}$.
For each problem, we generate three prefixed variants by cutting three prefixes from its off-policy trace at random cut points chosen uniformly between 40\% and 80\% of the tokens in the trace. The 3k prefixed and 1k no-prefix problems constitute the training data for \methodname{}.

\vspace{-0.05in}
\subsection{\methodname{} is More Compute-Efficient on Hard Problems Than Standard RL}
\label{subsec:experiments_compute_matched}

\highlight{\methodname{} is $2\times$ more compute-efficient and achieves higher training accuracy.} Figure~\ref{fig:compute_matched_training} shows that with the same compute, \methodname{} achieves higher accuracy on no-prefix problems compared to baselines for Llama-3.1-8B ($45\%$ greater)  and Qwen-3-4B ($30\%$ greater) respectively. Even after accounting for the initial rejection-sampling cost, \methodname{} improves compute-efficiency by roughly $2\times$ over the strongest baseline (SFT+RL). 
In contrast, standard RL and SFT+RL only slowly improve accuracy despite increasing the number of parallel samples per problem $n$  from 8 to 64. Thus, \methodname{} effectively reallocates wasted sampling FLOPs in standard RL towards productively improving training rewards.
Our iteration-matched plots (Figure~\ref{fig:compute_matched_llama}(c)) show that our baselines have stable training curves, and increasing samples per problem $n$ unsurprisingly attains higher accuracy in the iteration matched plot. Thus, \methodname{}'s gains are not explained by degenerate baseline runs.

\begin{figure}[!t]
    \centering
    \begin{subfigure}[t]{0.32\linewidth}
        \centering
        \includegraphics[width=\linewidth]{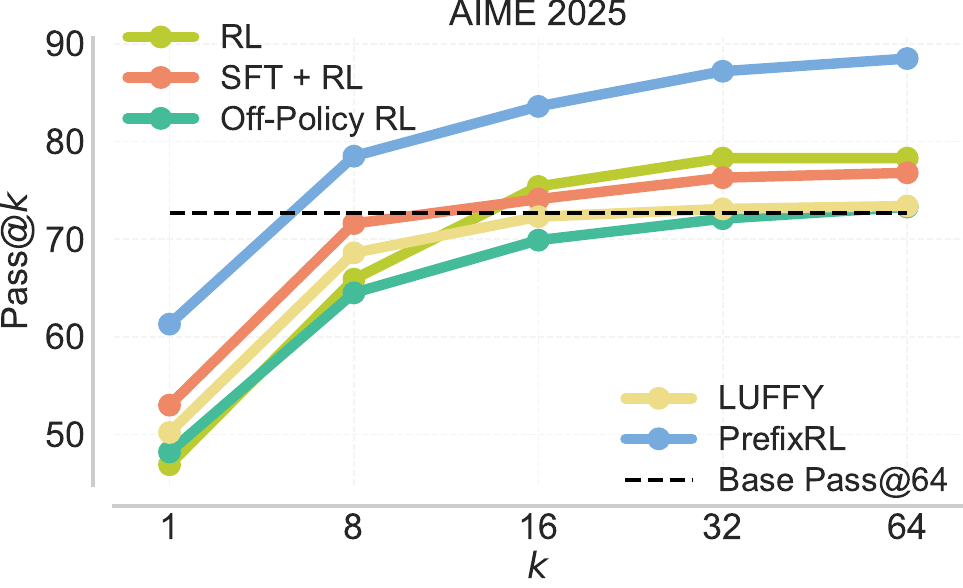}
    \end{subfigure}\hfill
    \begin{subfigure}[t]{0.32\linewidth}
        \centering
        \includegraphics[width=\linewidth]{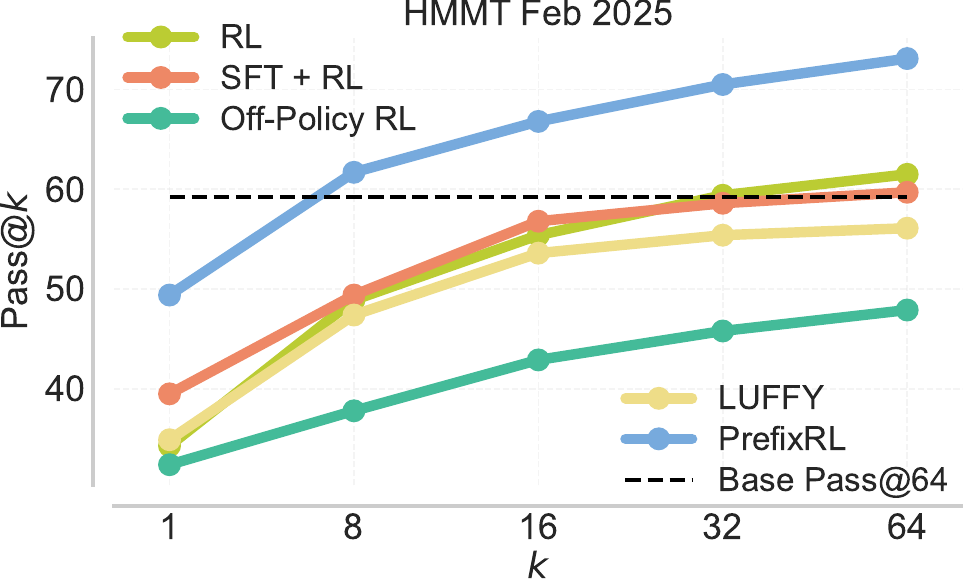}
    \end{subfigure}\hfill
    \begin{subfigure}[t]{0.32\linewidth}
        \centering
        \includegraphics[width=\linewidth]{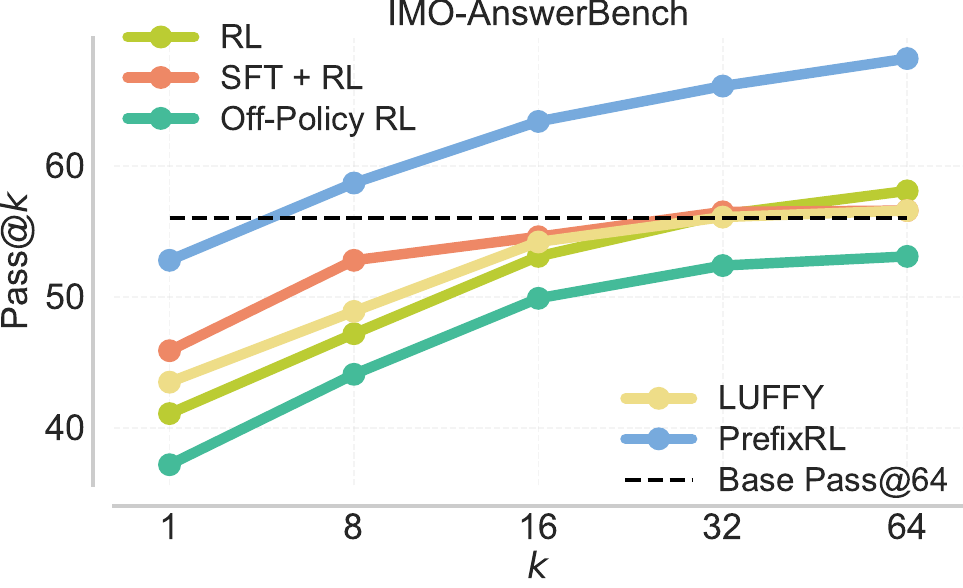}
    \end{subfigure}
    \vspace{-0.05in}
    \caption{\textbf{Held-out pass@$k$ of \methodname{}  trained Llama models}: We measure \texttt{pass@k} (\%) on AIME'25, HMMT'25 and 200 problems from IMO-AnswerBench for our distilled Llama-3.1-8B (base model) finetuned with standard on-policy RL,  and \methodname{} (with and without off-policy suffix injection). Note that IMO-AnswerBench is sampled from a similar distribution to our training problems, but is still held-out. The horizontal dashed line marks the base model's pass@64.}
    \label{fig:passk}
\end{figure}

\highlight{\methodname{} improves both pass@1 and pass@$k$ on held-out benchmarks.} Figure~\ref{fig:passk} shows that across AIME'25, HMMT'25, and IMO-AnswerBench, \methodname{} consistently improves \texttt{pass@k} (for $k\leq64$) over the baselines by over 10\% absolute, including off-policy RL methods such as importance-weighting (off-Policy RL) and LUFFY. The gains are visible already at $k{=}1$ and widen as $k$ increases.
This gain is notable since we only train on hard problems, so there is substantial transfer from the better training accuracy on hard problems to solving both easy and hard problems. On AIME'25 with Llama-3.1-8B, \methodname{} improves \texttt{pass@1} from 38.2 to 61.3, a sizable absolute gain given the benchmark’s difficulty. We observe a similar effect on HMMT'25, where \texttt{pass@1} increases from 29.2 to 49.4. These small-$k$ improvements are diagnostically important: they suggest that the model is more likely to instantiate the right high-level plan earlier in the trajectory, aligning with our “backward generalization” hypothesis (Section~\ref{sec:back_generalization}).
As $k$ grows, the performance gap generally widens (\textit{e.g.}, on AIME' 25: $+18$ points at $k{=}8$ and $+28$ at $k{=}64$). This pattern indicates that \methodname{} improves the diversity of the overall search distribution: additional samples explore more promising subspaces rather than repeating low-value trajectories. In other words, \methodname{} enhances both the \emph{mean} performance (seen in $k{=}1$) and the \emph{tail} (seen as $k$ increases) of the trace distribution.

\highlight{\methodname{} increases the support of solvable problems over the course of training.} 
Now we ask how much of the gains in Figures~\ref{fig:compute_matched_training},~\ref{fig:passk} come from ``sharpening'' the model's output distribution on problems that were already solvable with more samples (pass@$k$) versus an actual expansion in the support of solvable problems.
Figure~\ref{fig:compute_matched_llama}(b) shows that \methodname{} not only improves pass@1 but also steadily improves compute-matched pass@32, while competing baselines largely saturate over the course of training. This suggests that \methodname{} increases the set of problems with non-trivial success probability rather than only converting pass@$k$ into higher pass@1. 

\highlight{Uniform improvement across training problems.} The above pass@32 trend is mirrored by the pass@1 histograms in Figure~\ref{fig:compute_matched_llama}(a): vanilla RL improves a narrow subset of problems (likely those that were lucky enough to see positive samples early in the training) and stagnates elsewhere, while increasing $n$ helps by raising the chance of observing reward across the batch. Prior work attributes such non-uniform progress to training plateaus via \emph{ray interference}~\citep{schaul2019ray}. In contrast, PrefixRL improves pass@1 on a broader fraction of problems by collecting reward on prefixed versions and leveraging back-generalization to advance the no-prefix tasks.

\begin{figure}[b]
    \centering

    \begin{subfigure}[t]{0.32\linewidth}
        \centering
        \includegraphics[width=\linewidth]{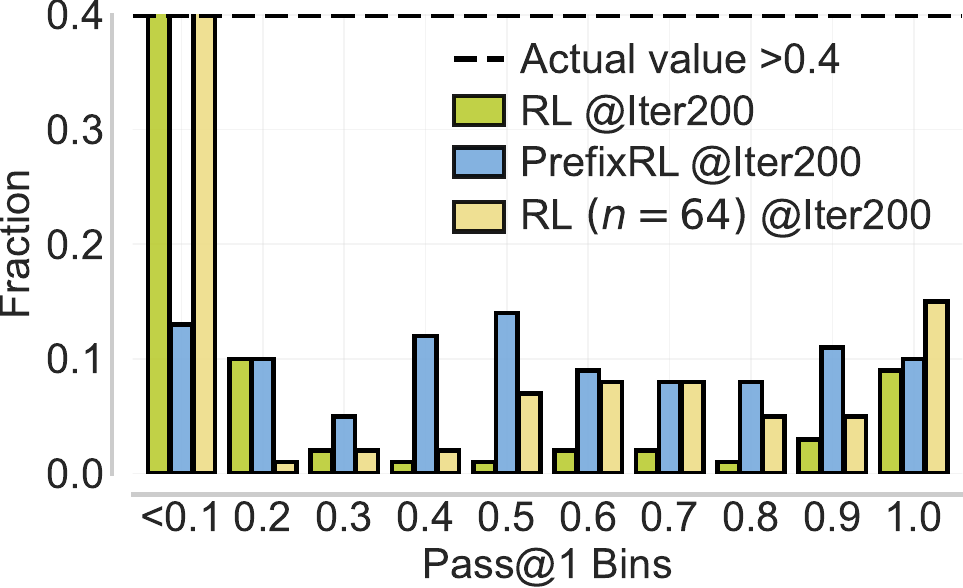}
        \caption*{(a)\vspace{-0.1in}}
    \end{subfigure}
    \begin{subfigure}[t]{0.32\linewidth}
        \centering
        \includegraphics[width=\linewidth]{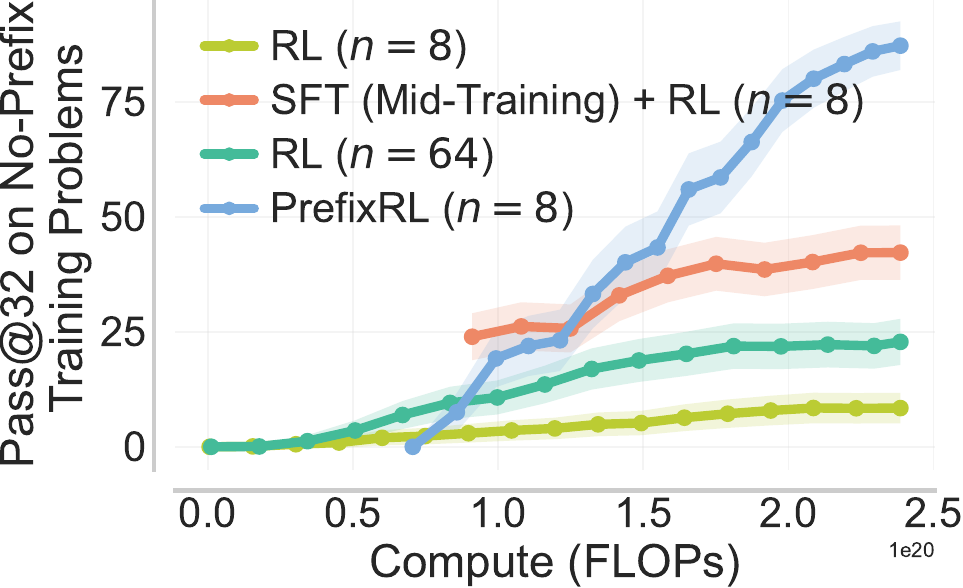}
        \caption*{(b)\vspace{-0.1in}}
    \end{subfigure}
    \begin{subfigure}[t]{0.32\linewidth}
        \centering
        \includegraphics[width=\linewidth]{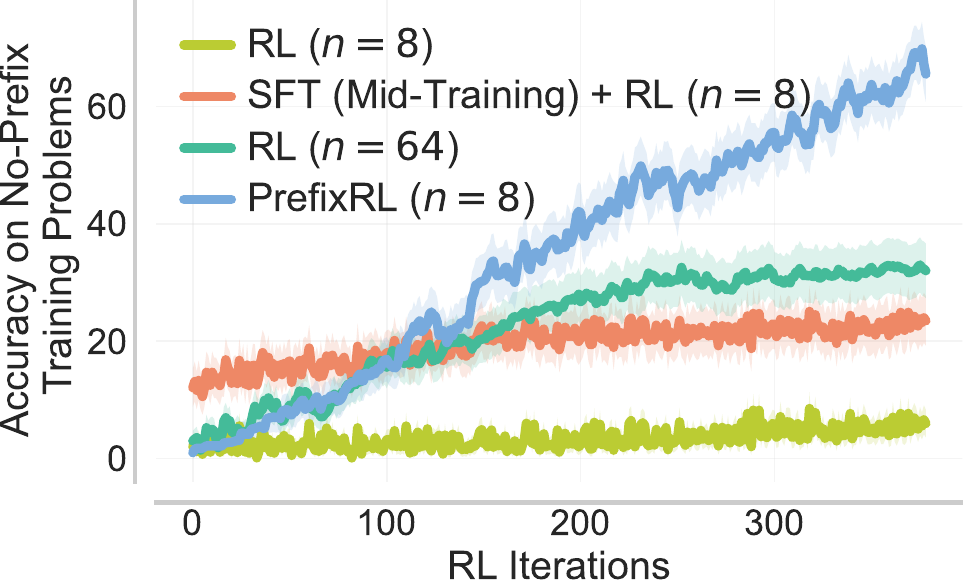}
        \caption*{(c)\vspace{-0.1in}}
    \end{subfigure}
    \caption{\textbf{Pass@$k$ evolution and iteration-matched comparison on training problems.}  (\textbf{a}) \textit{Uniform pass@1 improvement:} By design, the base LLM places all training problems in the pass@1 bin <0.1. After 200 iterations, PrefixRL yields the most uniform gains across problems, while RL concentrates improvements on a small subset with rare successes; increasing  $n$ partially mitigates this. (\textbf{b}) \emph{New problems solved:} compute-matched pass@32 plots  indicate that \methodname{} steadily expands the set of solvable problems rather than merely converting a fixed pass@$k$ (for small $k$) into higher pass@1, whereas the baselines largely saturate on pass@32. (\textbf{c}) \emph{Fair baselines:} iteration-matched reward curves confirm stable training across methods, so the compute-matched gains are not explained by unstable or degenerate baseline runs.
    }
    \label{fig:compute_matched_llama}
\end{figure}

\begin{figure}[t]
    \centering
    \includegraphics[width=0.42\linewidth]{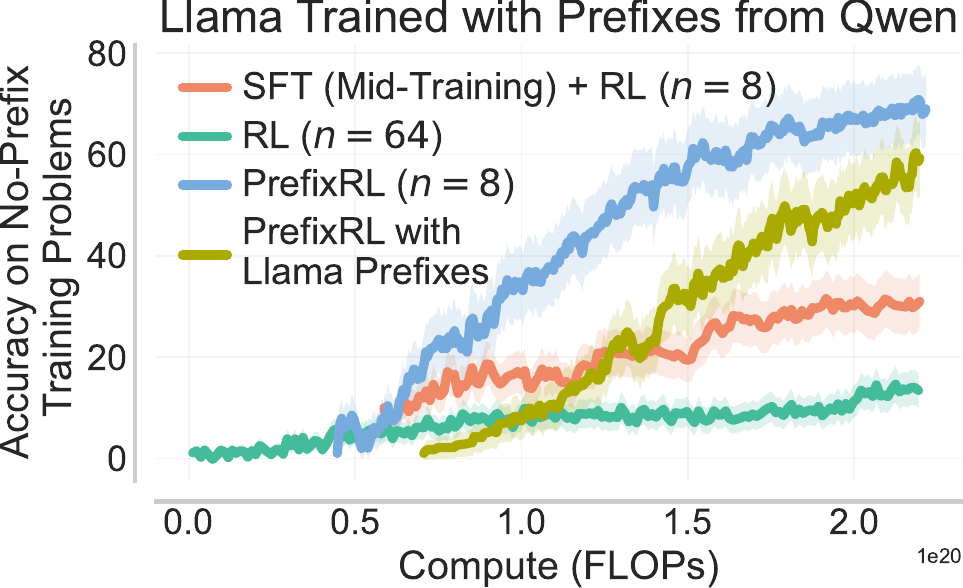}
    \hspace{1em}
    \includegraphics[width=0.42\linewidth]{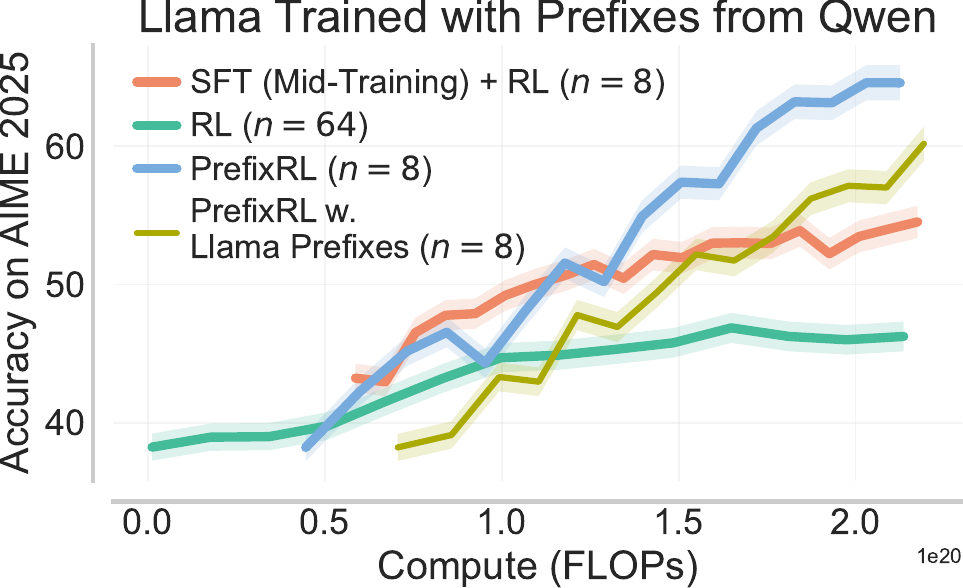}
    \vspace{-0.05in}
    \caption{\textbf{\methodname{} is still effective with off-policy prefixes from a different model family.} 
    We train Llama3.1-8b-instruct on prefixed-problems constructed using prefixes of rejection sampled traces from the Qwen-3-4b-instruct \textit{(left)}. Interestingly, even though prefixes are more out-of-distribution than those rejection sampled from Llama3.1-8b itself, they are still equally effective in improving on hard problems compared to when the prefixes are sampled from Llama3.1-8b (olive green line). We also plot the performance on AIME when we train Llama with prefixes from Qwen  \textit{(right)}.}
    \label{fig:prefix_switch}
\end{figure}

\highlight{\methodname{} is still effective with off-policy prefixes from a different model family.} 
\label{subsec:experiments_switch_prefix}
Up to now, we have used \methodname{} with prefixes from the same base LLM that we ran RL with. In many practical settings, it is easier to source off-policy prefixes from another model that is substantially different in training data or architecture; for example, from the open-source community. 
Motivated by this, we compare \methodname{} with mid-training and standard RL when the off-policy data is sourced from a policy (Qwen3-4b-instruct) that is different from the base LLM (Llama3.1-8B-instruct) in Figure~\ref{fig:prefix_switch}.
Despite the Qwen prefixes being more off-policy due to having a different training dataset and architecture, they end up being similarly effective on both train and test problems as off-policy prefixes from the Llama base model we run RL on. Note that the Llama base model required more initial compute for rejection sampling to collect the off-policy prefixes, which accounts for most of the difference in the curves.
In Appendix~\ref{app:qwen_with_llama_prefixes}, we also find that the reverse direction is effective, where we use Llama to generate off-policy prefixes for Qwen.

\begin{figure*}[!t]
    \centering
    \includegraphics[width=0.32\linewidth]{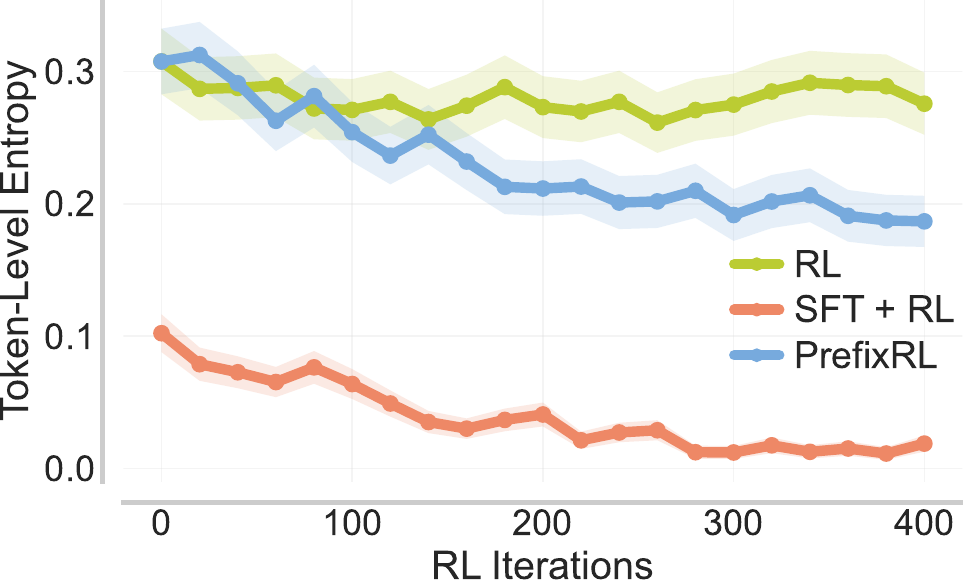}
    \includegraphics[width=0.32\linewidth]{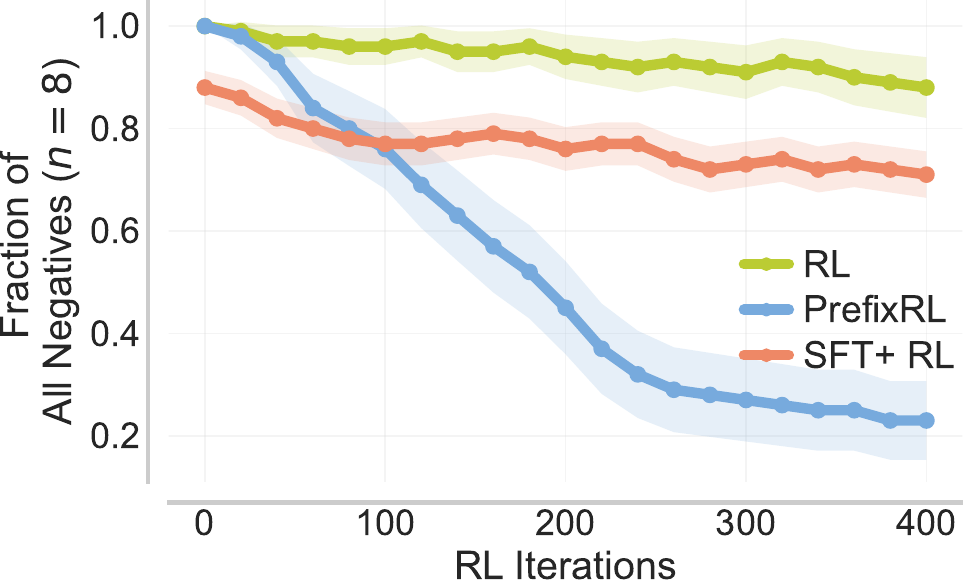}
    \includegraphics[width=0.32\linewidth]{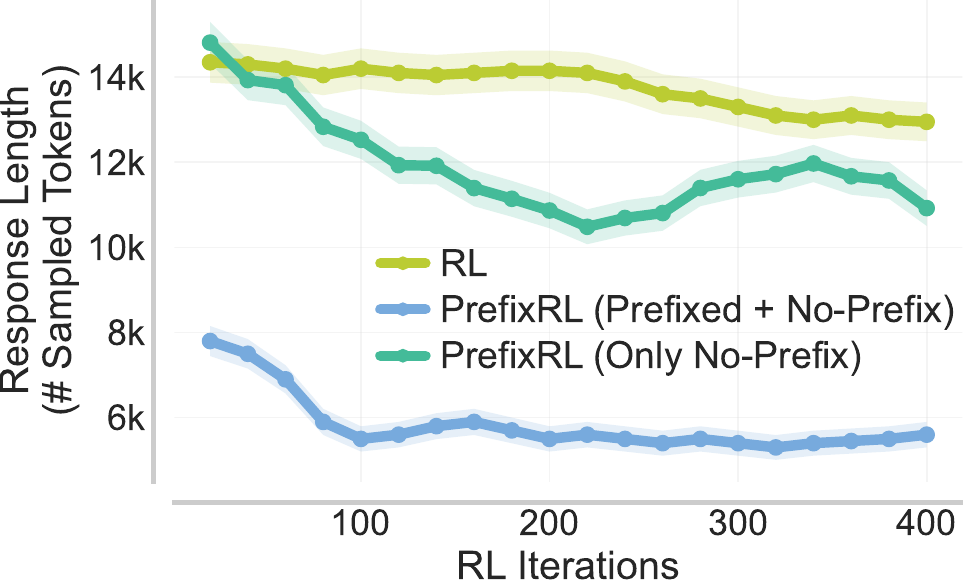}
    \caption{\textbf{RL Training Dynamics.} (\textbf{Left}) Compared to RL and PrefixRL, supervised learning on off-policy data (SFT+RL) leads to much lower token-level entropy during RL, which may hurt exploration. (\textbf{Middle}) \methodname{} has much lower ``all negative ratio'', or the number of prompts with all zero rewards (zero advantage) during training. Note that the ratio is measured only on problems without any prefixes.  (\textbf{Right}) When we look at the average response length on only no-prefix problems (green), we see \methodname{} generates shorter responses, mostly because correct solutions are usually shorter and standard RL has worse accuracy. Blue shows the average length across all problems, which is lower due to prefixing. This shortening is one way that \methodname{} reduces gradient variance and gains on compute efficiency compared to standard RL.  }
    \label{fig:training-dynamics}
    \vspace{-0.2in}
\end{figure*}

\vspace{-0.05in}
\subsection{Training Dynamics of \methodname{} are More Stable Compared to Typical Off-Policy Methods}
\label{subsec:experiments_dynamics}
\vspace{-0.05in}

In this section, we investigate RL training dynamics (like gradient variance) underlying the instabilities of off-policy RL. We also discuss the higher signal-to-noise ratio of policy gradients observed by PrefixRL as a result of the off-policy prefixes of correct traces placing the current RL policy in states with non-zero advantages.

\begin{figure*}[t]
    \centering
    \includegraphics[width=0.38\linewidth]{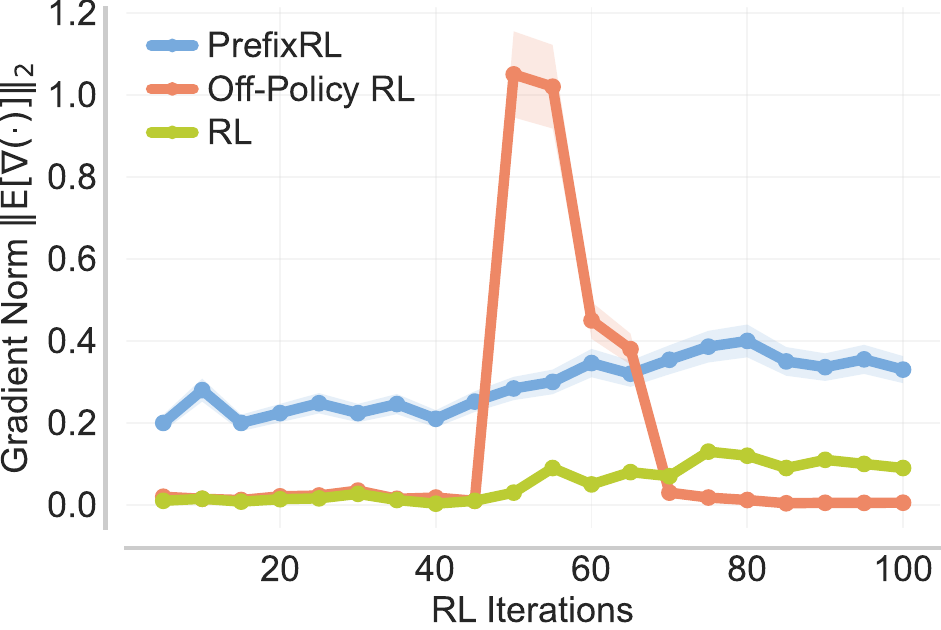}
    \hspace{2em}\includegraphics[width=0.38\linewidth]{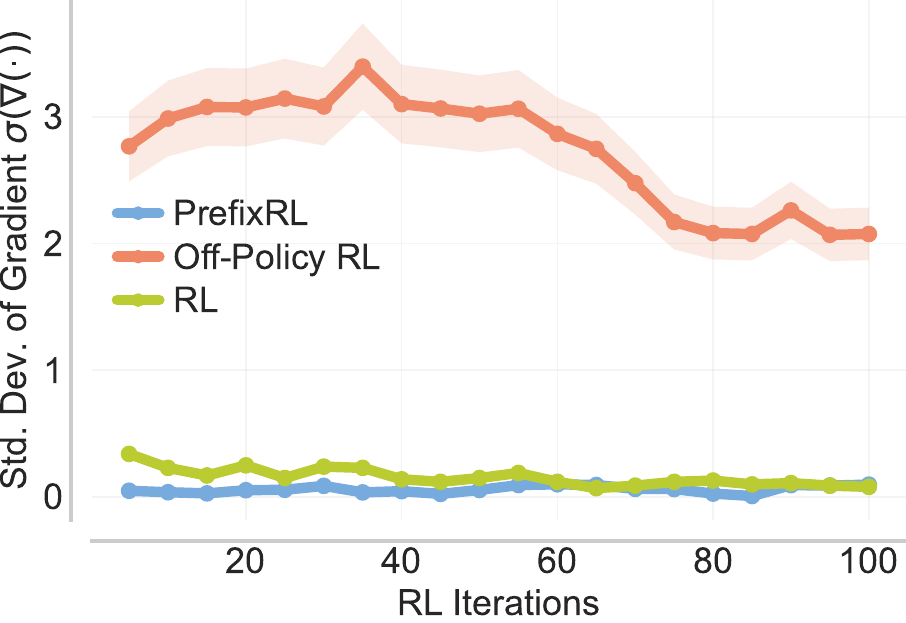}
    \caption{\textbf{\methodname{} signal-to-noise ratio and safely training on off-policy suffixes.} \methodname{} simultaneously has higher gradient norm (left) and lower gradient variance (right) than both standard RL and importance-weighted off-policy RL. For off-policy RL in particular, the importance weighting causes the gradient variance to be much higher and causes a gradient norm spike, which is a sign of training instability. \methodname{} avoids these issues while still leveraging off-policy data.}
    \label{fig:snr}
    \vspace{-0.2cm}
\end{figure*}

\highlight{\methodname{} leverages off-policy data while preserving entropy for RL exploration.} It is common practice to mid-train models on reasoning traces from more capable models or its own traces from base model inference or past RL runs~\citep{zelikman2022star}. This is typically done to prime the RL initialization and improve coverage over high-reward regions. But  this typically comes at the cost of the model's token-level entropy which impacts its ability to explore during RL.  Figure~\ref{fig:training-dynamics} (left) shows the average token-level entropy of the model's next-token distributions during the  RL run. Doing SFT on off-policy data causes the entropy to dramatically decrease and this only continues to drop further to values as low as $0.01$ during RL, suggesting that reward maximization during RL is mostly just sharpening the distribution over correct traces the base model can already sample (after running SFT). In contrast, \methodname{} preserves most of the token-level entropy while still making use of off-policy data. 

\highlight{Fewer all-negative batches on hard problems.}
Figure~\ref{fig:training-dynamics} (middle) plots the fraction of \emph{all-negative} problems (i.e., problems in a training batch for which all $n$ on-policy traces on that problem receive zero reward), measured \emph{only} over no-prefix problems. \methodname{} consistently reduces this ratio relative to on-policy RL. This reduction reveals an underlying shift in the unconditional policy: as training on strategy-revealing states proceeds, the model becomes more likely to \emph{enter} regions where non-zero advantages are attainable (either due to the prefix revealing ``useful'' strategies that are further  reinforced with positive rewards, or revealing likely but incorrect strategies that are quickly  down-weighted and unlearned (Section~\ref{subsec:backgen_synthetic})), thereby breaking the stalling regime (Section~\ref{sec:prelim}).

\highlight{\methodname{} achieves better accuracies with fewer sampled tokens.}
Figure~\ref{fig:training-dynamics} (right) tracks the average sampled tokens per trace across all prompts (with and without prefixes). On the no-prefix problems, \methodname{} eventually maintains a lower sampled token budget per trace while achieving higher reward rates, implying better \emph{iteration efficiency}. This is perhaps expected since correct traces are biased to be shorter and qualitatively, once the model internalizes the strategy, it reaches decisive steps earlier, which reduces ``unproductive wandering'' later in the horizon. Moreover, since PrefixRL trains on 3:1 mixture of prefixed to no-prefix problems, the average number of tokens sampled (across all problem types) per batch is less than $\nicefrac{1}{2}$ of the RL run. As a result of short length RL, the gradient variance for PrefixRL is much lower compared to standard RL, despite the PrefixRL not being biased (Section~\ref{subsec:theory}).

\highlight{\methodname{} policy gradients have a higher signal-to-noise ratio.}
Figure~\ref{fig:snr} reports gradient norm and gradient standard deviation for \methodname{},  off-policy RL, and standard RL. For off-policy RL, in addition to the on-policy samples, we also compute the batch gradient using the correct off-policy traces (in $\Doff$) for the problems in the batch, whenever the problem observes no correct on-policy sample in the $n$ attempts. We correct the sampling policy mismatch via importance weighting: while an unbiased sequence-level, unclipped correction is possible~\citep{tan2025gtpo}, it has prohibitive variance, so we follow prior work and use clipped token-level weighting (biased)~\citep{liuli2025} (Appendix~\ref{app:additional_baselines}). We estimate the expected gradient norm and sampled-gradient standard deviation by maintaining moving averages of coordinate-wise first/second moments and summing across coordinates (equivalently, tracing the gradient covariance; Appendix~\ref{app:grad_norm_std}). Off-policy RL exhibits noisy, biased gradients and even a norm spike (instability). Because rejection-sampled traces are extremely unlikely under $\pi_0$ and later $\pi_t$, token-level weights can be tiny (e.g., $10^{-3}$), shrinking the expected gradient so its norm stays comparable to on-policy RL until the spike (Figure~\ref{fig:snr} (left)). In contrast, PrefixRL samples fewer tokens (lower variance) and its prefixes roll the policy into higher-advantage states (higher expected norm), yielding simultaneously higher gradient norm and lower variance (Figure~\ref{fig:snr} (right)) \textit{i.e.}, a better gradient signal-to-noise ratio.

%% file: sections_v2/relwork.tex
\vspace{-0.15cm}
\section{Related Work and Discussion}
\label{sec:relwork}

\highlight{Learning from off-policy LLM rollouts.} When on-policy search stalls due to over-sharpening or “over-thinking,” a common approach is to supervise on human or oracle-provided traces~\citep{lightman2023lets,corrado2024guideddataaugmentationoffline}, but teacher-driven methods inherit the teacher’s capacity limit~\citep{agarwal2024onpolicydistillationlanguagemodels} and often require reward shaping~\citep{yan2025learningreasonoffpolicyguidance}, entropy control~\citep{wang2025beyond}, and heavy hyperparameter tuning~\citep{zhang2025onpolicyrlmeetsoffpolicy}; moreover, for hard problems, long model-compatible chains of thought are scarce and mismatches can collapse response diversity~\citep{kang2024unfamiliar}. When off-policy data come from "close enough" (in KL divergence) policies as in Async RL, reuse becomes more efficient~\citep{fu2025areallargescaleasynchronousreinforcement,khatri2025artscalingreinforcementlearning}, yet large importance weights and high gradient variance pose instability risks~\citep{agarwal2021theory}, so practical systems cap behavior-policy staleness to only a few RL iterations~\citep{sheng2024hybridflow}.  
More recent works on use off-policy data~\citep{huang2025blendingsupervisedreinforcementfinetuning,zhang2025rlepreinforcementlearningexperience} also use the entire off-policy data or parts of it  as imitation targets on problems where on-policy correct rollouts are rare.
Different from the above, PrefixRL conditions on off-policy prefixes from long-thinking traces, as opposed to training on them before or during RL. 
Thus, instead of suffering from instabilities caused by using off-policy data as direct  targets  (Sections~\ref{sec:method} and ~\ref{subsec:experiments_dynamics}), PrefixRL benefits from the old inference FLOPs spent on off-policy traces indirectly via back-generalization.

\highlight{Conditioning on hints to improve on-policy RL.} The above concerns motivate approaches that do not treat off-policy trajectories as targets; related directions condition on subgoals or plans~\citep{hong2025planning}, higher-level abstractions~\citep{qu2025rladtrainingllmsdiscover}. 
A related line of work augments prompts with hints or partially revealed human solutions to ``guide'' on-policy RL~\citep{chen2025nudgingboundariesllmreasoning,li2025questa,qu2025pope}.   
AdaBack~\citep{amani2025rlreasoningadaptivelyrevealing} adaptively searches for the minimal hint that improves performance over human-written solutions, but is hard to scale to long-context ``thinking models'' and large datasets. Similarly, QuestA~\citep{li2025questa} uses answer-hinted prompts derived from human solutions.
In general, these methods are only feasible when we have access to solution traces written by a human or a more capable teacher model. In contrast, PrefixRL enables a self-improvement loop by not relying on external sources and instead reusing compute from prior models. Moreover, our work also analyzes the back-generalization phenomenon that may be shared across these methods, showing that it cannot be explained by some of the ``stitching'' arguments made in prior work~\citep{zhang2025bread}.

\highlight{Resetting to off-policy states in RL.} The idea of ``resetting'' current RL policy to off-policy states is not new in RL
\citep{kakade2003samplecomplexity,bagnell2003psdp,nair2018overcoming,salimans2018montezuma,yin2022localplanning,uchendu2023jumpstart,silver2016alphago,silver2016mastering,agarwal2019rltheory,daume2005laso,daume2009searn}.
\cite{chang2024dataset} also applied the resetting idea for post-training LLMs with human feedback. While similar in principle, their work does not study computational gains accounting for the initial compute spent on collecting off-policy traces. In fact, in their case the data is human labeled whereas our work lies more in a self-improvement setting.  
Our contribution is to instantiate this perspective for RL of reasoning LLMs.
We show that a relatively small dataset of \emph{correct} off-policy traces is sufficient to enable effective resets that make hard, low-pass-rate problems trainable even when on-policy rollouts almost never succeed.
We also show that PrefixRL yields a strictly better allocation of compute, even after accounting for the inference cost of collecting the off-policy traces.

\highlight{Improving exploration on hard problems in LLM reasoning.}
Small models fine-tuned with RL can outperform much larger base models~\citep{liu2025prorlprolongedreinforcementlearning,deepscaler2025}, largely by reinforcing long chain-of-thought behaviors like self-correction~\citep{qu2024recursive} and reflection~\citep{gandhi2025cognitivebehaviorsenableselfimproving}. Yet, without careful controls, RL often under-explores and leaves hard instances underprobed; empirically this appears as a drop in pass@$k$ versus the base model~\citep{yue2025doesreinforcementlearningreally,zhao2025echochamberrlposttraining}. One response is to regularize training to curb over-sharpening via intrinsic-motivation bonuses~\citep{gao2025navigateunknownenhancingllm}, entropy~\citep{wang2025reinforcementlearningreasoninglarge}, count-based signals~\citep{song2025outcomebasedexplorationllmreasoning}, or objectives that directly optimize pass@$n$~\citep{chow2024inference,balashankar2025infaligninferenceawarelanguagemodel}, but these still inherit sparse-reward limits and depend on easy problems for signal~\citep{he2024olympiadbenchchallengingbenchmarkpromoting}. A complementary thread~\citep{setlur2025e3learningexploreenables} exploits base-model asymmetries, \textit{e.g.}, the verification–generation gap~\citep{setlur2025scalingtesttimecomputeverification,song2024mind}, and can combine with negative-gradient dynamics to chain such asymmetries across updates~\citep{zhu2025surprisingeffectivenessnegativereinforcement}; nevertheless, models often ``under-think''~\citep{wang2025thoughts}, persisting with wrong high-level plans despite more rollouts. 
In contrast, PrefixRL avoids carefully tuned auxiliary exploration objectives by reshaping the start-state distribution directly. 
Empirically, we do not observe the pass@$k$ regressions often induced by over-sharpening or over exploration with token-level entropy regularizers.
In the worst case, uninformative prefixes recover standard on-policy RL (Section~\ref{subsec:backgen_icl}).

%% file: sections_v2/conclusion.tex
\vspace{-0.1cm}
\section{Conclusion}
\label{sec:conclusion}
\vspace{-0.1cm}
\methodname{} is fundamentally different from typical methods that \emph{imitate} off-policy data, relying instead on back-generalization to incorporate off-policy data as \emph{conditioning} context while keeping updates on-policy.
We expect the design space for such algorithms to be broad since conditioning is highly flexible and the mechanisms (and full potential) of back-generalization aren't fully understood.
Crucially, back-generalization leverages the capabilities of the model itself to incorporate off-policy feedback.
Thus, harnessing back-generalization is a crucial ingredient for bootstrapping prior compute and continuous self-improvement.

\textbf{Acknowledgements.} We would like to thank Aviral Kumar, Rishabh Agarwal, Shengjia Zhao, Hongyu Ren, Sean Bell, Ankur Pai, Saurabh Garg, Wen Sun, Max Simchowitz,  Sergey Levine, Yuxiao Qu, Ian Wu, Rohan Maheshwari, and Yuandong Tian for helpful discussions and thoughtful feedback on our work.

%% file: sections_v2/appendix.tex
\section{Full Set of Notations}
\label{sec:additional_notation}

\highlight{Markov decision process.} We use $\bx$ to denote an input problem and $\by=(y_1,\ldots,y_{H})$ for a response of $H$ tokens, and if $\by \sim \pi(\cdot \mid \bx)$, then $\by$ is sampled auto-regressively from the LLM $\pi$ fed with input $\bx$. Each token in this response $\by$ belongs to a set of tokens or actions $\mathcal{A}$.  The state $\bs_h$ at time step $h$ is given by $(\bx, y_1, y_2, \ldots, y_h)$, where the initial state $\bs_0$ is just the problem $\bx$. The set of states across all time steps is denoted by the class $\mathcal{S}$.  We use $d^\pi_h$ to denote the distribution over states $\bs_h$ at time step $h$ by rolling out the policy $\pi$ auto-regressively for $h$ time steps. 
For compactness, we write the trajectory-level log-likelihood $\log \pi(\by\mid \bx)=\sum_{h=1}^{|\by|}\log \pi(y_h\mid \bx,\by_{<h})$. For each problem we have access to outcome reward function $r(\bx_i, \by)$ to check whether the final answer in response $\by$ is correct/incorrect (1/0) for the problem $\bx_i$ (\textit{e.g.}, by matching the boxed answer in the end of $\by$ for math problems).   

\highlight{Dataset of hard problems and off-policy traces.} We use $\cal{D}$ to denote a dataset of $N$ hard problems $\mathcal{D} =: \{\bx_i\}_{i=1}^{N}$. 
We use $\pibase$ to denote the base pre-trained LLM that we initialize the RL algorithm, $\pit$ as the policy after $t$ RL iterations and $\Doff$ as the dataset of off-policy traces. 
Finally, we define the pass rate @$k$ for problem $\bx$ and LLM $\pi$ as $\E_{\by_1, \ldots, \by_k \sim \pi(\cdot \mid \bx)} \max(\{r(\bx, \by_i)\}_{i=1}^k)$. In the main paper, we define the set of hard problems as those with pass@$k$$\approx$$0$ under the base LLM $\pibase$. See Section~\ref{sec:experiments} for how we select these low pass rate hard problems for training.

\section{Using Off-Policy Traces as Supervision Targets}
\label{app:offpolicy_as_supervision}

\begin{figure}[!h]
   \centering
    \includegraphics[width=0.24\linewidth]{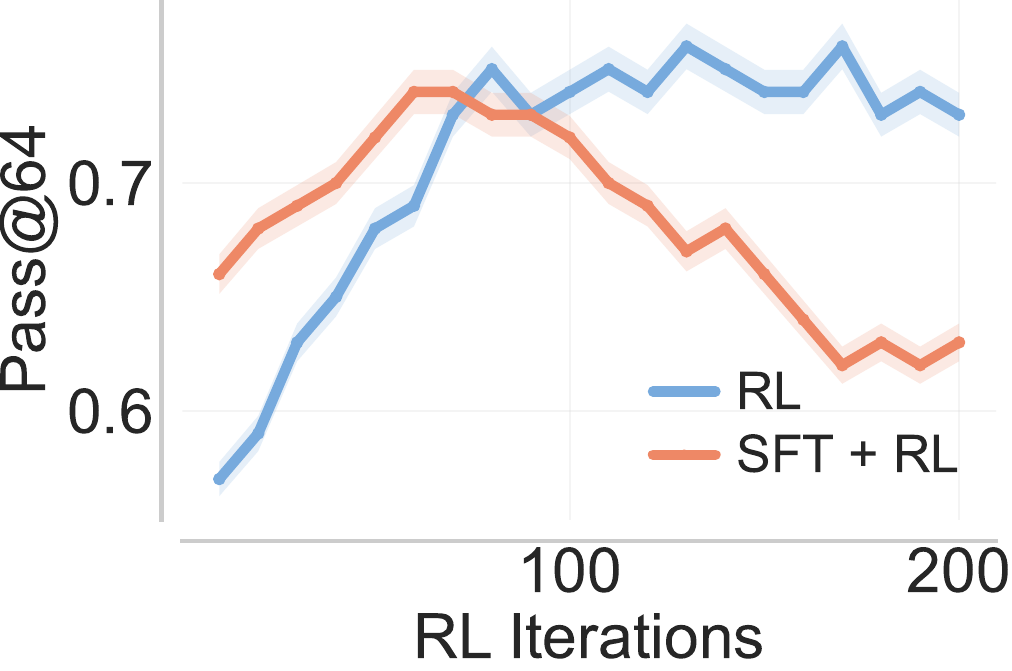}
    \includegraphics[width=0.24\linewidth]{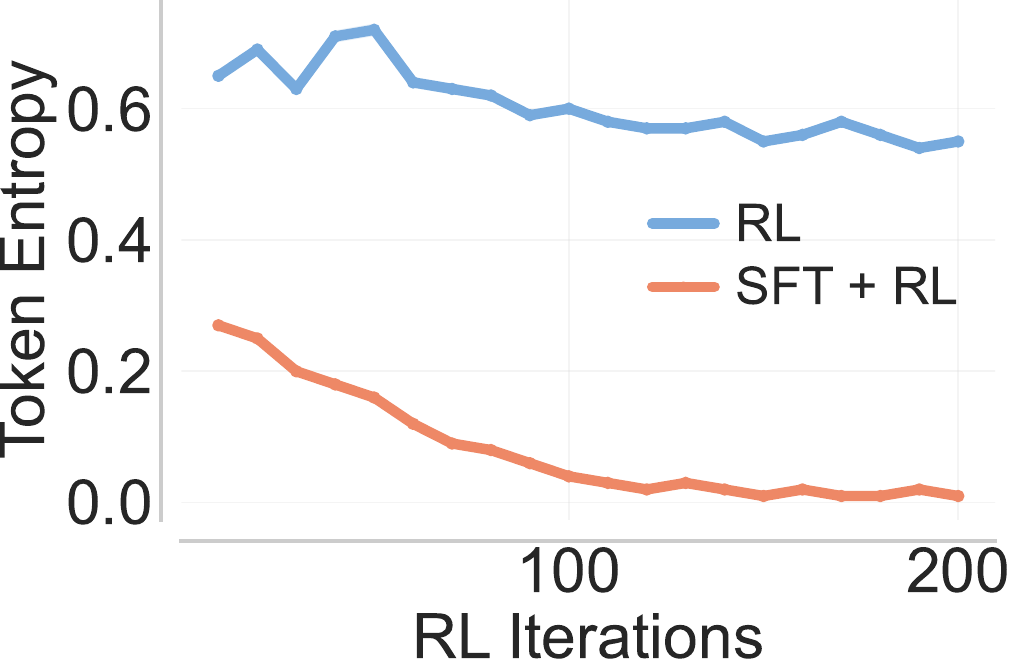}
    \includegraphics[width=0.24\linewidth]{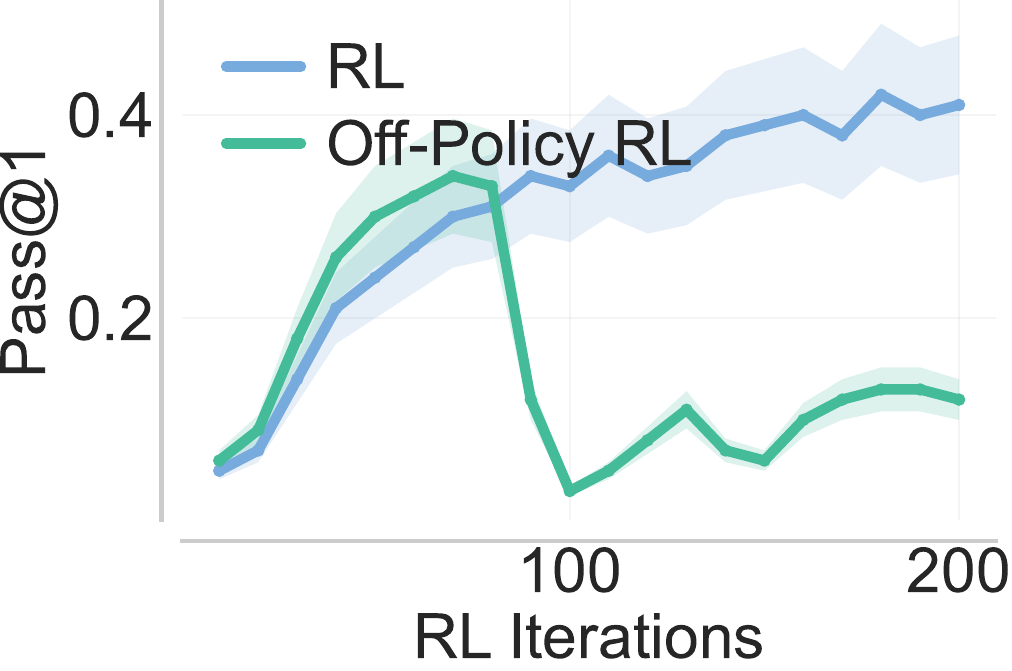}
    \includegraphics[width=0.24\linewidth]{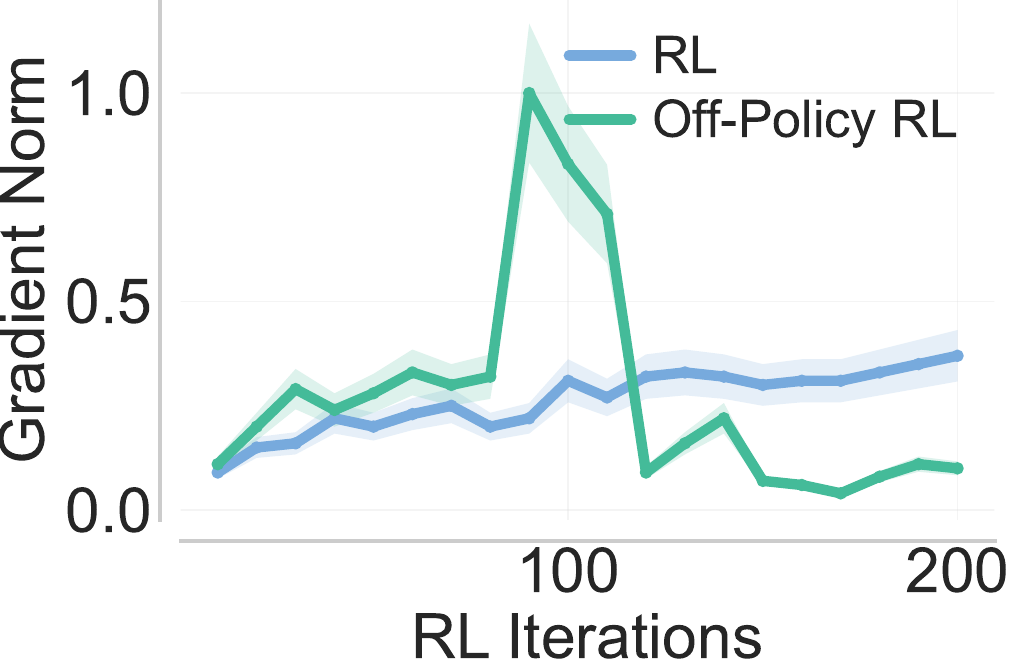}
    \caption{\textbf{Supervising the policy on $\Doff$ can cause diversity collapse or training instabilities during RL:} \textbf{(a, b)} Warm-starting the RL run by running SFT on $\Doff$ before (mid-training) reduces token entropy (a) and hurts exploration during RL (worse Pass@64 performance in (b)). \textbf{(c, d)} Directly using $\Doff$ during online RL by updating the current RL policy importance-weighted off-policy traces (in addition to on-policy traces) leads to training instabilities. We see the gradient norm (clipped at 1.0) blow up during training (d) and this leads to optimization collapse (c).} 
    \label{fig:midtrain_and_offpolicyrl}
\end{figure}

\highlight{SFT on $\Doff$ can hurt exploration during RL.} 
A common way to improve RL on hard problems is to first mid-train (SFT) on traces in $\Doff$ to warm-start RL training. However, while SFT boosts post-RL pass@1, pass@64 drops after about 100 iterations (Figure~\ref{fig:midtrain_and_offpolicyrl}(a)). 
This loss of diversity is driven by a sharp entropy collapse after SFT, which decreases further during RL (Figure~\ref{fig:midtrain_and_offpolicyrl}(b)). This suggests that here, RL mainly sharpens the distribution of responses the model can already produce after SFT. In contrast, an early-stopped SFT checkpoint can underfit and encourage random exploration during RL~\citep{wang2025octothinkermidtrainingincentivizesreinforcement}. One can partially mitigate this by enlarging the SFT dataset ($\Doff$), but doing so increases the upfront cost of the SFT step.

\highlight{Off-policy RL using $\Doff$ leads to training instabilities.}
A more direct way to use the off-policy data $\Doff$ is to do importance-weighted off-policy RL~\citep{degris2012off} (see Appendix~\ref{app:additional_baselines}), which accounts for the distribution shift between the current RL and sampling (rejection sampling on base LLM) policies. However, this suffers from large gradient variance or heavily biased gradients due to clipping and token-level weighting rather than sequence-level~\citep{agarwal2021theory} (Section~\ref{subsec:experiments_dynamics}). This can cause training reward collapse and unstable optimization (Figures~\ref{fig:midtrain_and_offpolicyrl}(c),(d)) as we force updates on very unlikely token sequences under the current RL policy, leading to memorization of $\Doff$~\citep{kang2024learning,setlur2024rl}.

\section{Omitted Proofs}
\label{sec:proofs}

In this section, we present proofs for our theoretical results in Section~\ref{subsec:theory}. We begin with the proof for Theorem~\ref{thm:consistency} which implies that the PrefixRL objective is consistent with standard RL, and any solution for our PrefixRL objective, is also a mazimizer of the standard RL objective which just maximizes $J(\pi)$. Following this, we show the proof for Theorem~\ref{thm:suboptgap} which bounds the suboptimality gap of  an algorithm using natural policy gradient (NPG) to optimize the PrefixRL objective. Note that this is slightly different from the policy gradients we use in practice, but is nevertheless insightful in informing a formal mental model for the gains behind PrefixRL. We also provide a proof for 
Proposition~\ref{prp:worstcase_gap} that lower bounds the performance gap between PrefixRL and standard RL in the worst case.  In the end we list auxiliary lemmas useful for analysis.

\subsection{Proof of Theorem~\ref{thm:consistency}}\label{sec:consistency_proof}

\input{sections_v2/appendix_thm_proof_consistency}

\subsection{Proof of Theorem~\ref{thm:suboptgap}}\label{sec:sample_complexity_proof}

In this section we present our proof for Theorem~\ref{thm:suboptgap} which bounds the performance suboptimality for PrefixRL. In particaly, we bound thisgap for an  algorithm  that conforms to the PrefixRL workflow (Algorithm~\ref{alg:npgalg}) but  uses natural policy gradient (NPG)~\citep{kakade2001natural} to update the policy iteratively (starting from the base LLM $\pi^{0}$). In our  practical implementation use  REINFORCE to compute on-policy gradients.  Next we introduce the setup and describe the key steps in Algorithm~\ref{alg:npgalg}, some differences with practice and the full proof.  

\highlight{Setup.} We use $\pi^{0}$ to denote the base LLM we start RL training with, and $\mu$ as the policy used to IID sample the dataset of off-policy traces $\Doff$, one trace for each problem in $\bx \in \mathcal{D}$. We assume that $\Doff$ is realizable (Assumption~\ref{assm:realizable}) which implies that it only consists of correct off-policy traces. Let $\mathcal{F}$ denote the class of $Q$-functions induced by all policies in the policy class $\Pi$, and $H$ be the maximum context length or horizon $H$ of the auto-regressive Markov decision process (MDP) induced by the policies in $\Pi$ and reward function $r$.

\highlight{Description of PrefixRL with NPG (Algorithm~\ref{alg:npgalg}).} In each iteration of Algorithm~\ref{alg:npgalg} we \emph{first} collect a dataset of state, action and reward triplets $\mathcal{D}_t$. \emph{Second}, we fit a critic or $Q$ function $\hat{Q}^t$ on this dataset of size $N$ (step 9). \emph{Third}, we use the fitted $Q$ function to perform a state-wise mirror ascent or natural policy update (step 11) in order to get the subsequent RL iterate.
We collect the $N$ traces in $\mathcal{D}_t$ by uniformly sampling an off-policy state $\bs_h$ (prefixed problem) from $\Doff$. Then, we rollout the current RL policy $\pit$ conditioned on state $\bs_{h}$ to sample a single action or token $a_h$ (step 6). To estimate the $Q$ function under the current RL policy at this state-action pair we now complete the rollout till time step $H$ and collect a reward (step 7).

\begin{algorithm}[t]
\caption{PrefixRL with Natural Policy Gradients}
\label{alg:npgalg}
\begin{algorithmic}[1]
\Require Base policy $\pi^0$, off-policy data $\Doff$,  horizon $H$, iterations $T$, step size $\eta$,  $Q$ function class $\mathcal{F}$. 
\State Initialize the iterative algorithm with base policy: $\pi^{1} \gets \pi^{0}$.
\For{$t = 1, \dots, T$}
    \State Initialize dataset $\mathcal{D}_t \gets \{\}$.
    \For{$i=1 \ldots n$}
    \State Sample $(\bs_h, a_h^{\mathrm{off}})$ uniformly across state-action pairs in $\Doff$.  \Comment{sample prefixed problem}
    \State $a_h \gets a_h^{\mathrm{off}}$  with probability \nicefrac{1}{2} and $\sim \pit(\cdot \mid \bs_h)$ otherwise. 
    \State Execute $\pit(\cdot \mid \bs_h, a_h)$ from step $h{+}1$ through $H$ to obtain the full trace with reward $r$.
    \State $\mathcal{D}_t \gets \mathcal{D}_t \cup (\bs_h, a_h, r)$. 
    \State \textbf{Critic fit (regression oracle):}
    \State \hspace{0.6cm} $\hat Q^{t} \gets \arg\min_{f \in \mathcal{F}} \ \sum_{(\bs, a,r)\in \mathcal{D}_t}  (f(\bs,a) - r)^2$. 
    \State \textbf{Natural policy update (mirror ascent):} \Comment{performed state-wise}
    \State \hspace{0.6cm} $\pi^{t+1}(\cdot \mid \bs) \gets \arg\min_{p}\ \langle -\hat Q^{t}(\bs,\cdot), p \rangle + \tfrac{1}{\eta} {\mathrm{KL}} \paren{p \,\|\, \pit(\cdot \mid \bs)}$.
\EndFor
\EndFor
\State \Return $\bar{\pi}_T \gets \tfrac{1}{T}\sum_{t=1}^{T}\pi^{t}$. \Comment{return mixture policy}
\end{algorithmic} 
\end{algorithm}

\highlight{Difference with practice: Algorithm~\ref{alg:npgalg} uses $Q$ functions instead of direct rewards.} The update in NPG is similar to REINFORCE except that we use $N$ on-policy samples to first estimate a $Q$ function (step 10) in the $Q$ function class $\mathcal{F}$) for the current RL iterate $\pit$ and then use the estimated $Q$ function to update the policy and get $\pi_{t+1}$ using mirror ascent (step 12). This is a bit different from REINFORCE where we compute the policy gradient using only  the rewards attaind by the $N$ on-policy traces and perform a step of gradient. 

\highlight{Difference with practice: Algorithm~\ref{alg:npgalg} samples new prefixed problems from $\Doff$.} In practice we construct a prefixed problems from a fixed dataset of off-policy traces $\Doff$ and also use the set of prefixed problems in $\cal{D}_\mathrm{pre}$ are fixed throughout RL training. In contrast,  Algorithm~\ref{alg:npgalg} samples off-policy states (prefixed problems) from the dataset of off-policy traces $\Doff$. This difference is pretty minor but perhaps underscores the  performance improvements driven by back-generalization in being able to improve performance on the original no-prefix problems despite PrefixRL only using a small fraction of all possible off-policy prefixes in $\Doff$.

\highlight{Comparison with \citet{chang2024dataset}.} Our proof technique  follows  \citet{chang2024dataset}, adapting to the setting of verifiable rewards with our different ``reset'' policy (which we refer to as prefix policy $\mu$), and removing the requirement of KL divergence between the current and the reset policy. Since our off-policy dataset consists of only realizable correct traces we will need much weaker assumptions. Following are some key differences compared to Algorithm 3 in \cite{chang2024dataset} that allows us to prove the suboptimality gap with weaker assumptions. \emph{First}, we sample the prefix from the comparator policy (in other words the prefix generating policy is realizable and lies in the class of optimal policies).
This ensures sufficient coverage for the distribution of Q-function regression (ensuring small error in fitting the critic) over states visited by a ``good'' policy even though the current RL iterate is far from it. \emph{Second}, we output the mixture policy (standard in self-play literature~\citep{bai2020near,hofbauer2006best}). \emph{Finally}, unlike \citet{chang2024dataset}, we don't require a  bound on the KL divergence against the SFT policy or the policy trained on the off-policy data.

\highlight{Assumptions needed for Theorem~\ref{thm:suboptgap}.} Now, we list the assumptions we make in our analysis of the suboptimality gap of PrefixRL. In general, they are milder than  the assumptions in \citet{chang2024dataset}. 
\begin{itemize}
    \item Assumption~\ref{assm:realizable_q} is pretty standard in the analysis of actor-critic methods~\citep{haarnoja18sac} and only requires that our critic function class is expressive enough to realize the $Q$ function induced by any policy in $\Pi$. Note that since rewards are binary and terminal the $Q$-value at any state $\in [0, 1]$. 
    \item Assumption~\ref{assm:realizable-app} is a milder form of the typical assumption on the coverage over states visited by the optimal policy. Here, we only assume that there is an optimal policy that can fit the dataset $\Doff$ we collected. Typically the coverage assumption places  a uniform bound on the likelihood ratio over the state distributions of the optimal policy and the current RL policy $(d^{\pi^\star}/d^{\pi})$ as in \cite{chang2024dataset}.
    \item Assumption~\ref{assm:finite_KL} is necessary to ensure that the KL between the prefix generating policy (empirical distribution over $\Doff$) and the base policy $\pi^{0}$ is finite. If the size of the dataset $\Doff \rightarrow \infty$ and the samples in $\Doff$ are all drawn IID from a policy $\nu \in \Pi$, then this assumption requires that cross-entropy between $\nu$ and $\pi^{0}$ is finite. 
\end{itemize}

\begin{assumption}[Realizability of $Q$-function class] \label{assm:realizable_q} There is a finite $Q$-function class $\cF \subseteq \{f: \mathcal{S} \times \mathcal{A} \rightarrow [0,1]$, and that $Q$-function induced by any policy is realized in this class, i.e., $Q^{\pi} \in \cF \; \forall \pi \in \Pi$.
\end{assumption}

\begin{assumption}[Correctness and realizability of $\Doff$]
\label{assm:realizable-app} We say that $\Doff$ is correct if it contains a single correct trace $\by$ for every $\bx \in \mathcal{D}$ and realizable if $\exists$  some policy $\mu \in \Pi$ such that  ${\mu}(\by\mid \bx) = 1,\; \forall (\bx, \by) \in \Doff$.
\end{assumption}

\begin{assumption}[Bounded likelihood of $\Doff$ under $\pi^{0}$]
\label{assm:finite_KL} The KL divergence between base LLM $\pi^{0}$ and the policy $\mu \in \Pi$ that perfectly fits the data is
    $\mathrm{KL}(\mu||\pi^{0}) < \infty$. In other words, this assumes that the samples in $\Doff$ have a bounded likelihood under the base LLM $\pi^{0}$, i.e., $\mathrm{KL}(\mu||\pi^{0}) = \frac{1}{|\Doff|} \sum_{(\bx, \by^{\bx}) \in \Doff} - { \log \pi^{0}(\by^{\bx} \mid \bx)} < \infty$.   
\end{assumption}

\input{sections_v2/appendix_thm_proof_subopt}

\subsection{Proof of Proposition~\ref{prp:worstcase_gap}}
\label{subsec:proof_prp_worstcase_gap}
\input{sections_v2/appendix_thm_proof_lb}

\subsection{Auxiliary Lemmas}
\input{sections_v2/appendix_auxiliary_lemmas}

\section{Additional Experiment Details and Results on Back-Generalization}
\label{app:additional_backgen}

\input{sections_v2/appendix_additional_backgen}

\section{Additional Experiments and Details for Results in Section~\ref{sec:experiments}}
\label{app:additional_experiments}

\subsection{Implementation details for PrefixRL and Baselines}
\label{app:additional_baselines}

\subsubsection{Off-policy RL Baselines}

In our work, we use CISPO~\citep{chen2025minimax} to compute the policy gradient following \citet{khatri2025artscalingreinforcementlearning} which found it to work best on long RL runs. Moreover, CISPO can also handle off-policy updates, \textit{i.e.}, update $\pit$ on a trajectory $\tau$ sampled from $\mu \neq \pit$ with an importance weighting term common in off-policy RL~\citep{fujimoto2018off}.
For each $\bx$ in batch $\cal{B}$, CISPO samples $k$ reponses $\{ \yix\}_{i=1}^k$ where $\yix \sim \mu(\cdot \mid \bx)$, then the CISPO policy gradient is given by: 
{
\setlength{\abovedisplayskip}{6pt}
\setlength{\belowdisplayskip}{6pt}
\begin{align}
\text{\textbf{\textcolor{lightblue}{(CISPO)}}} \quad & \frac{1}{\text{token-sum}}  \sum_{\bx\in \mathcal{B}}\sum_{i=1}^k
\sum_{h=1}^{|\yix|} \big(\text{stop-grad} \left(\max\left(w(\bx,y^{\bx}_{i,h}), \varepsilon_\mathrm{high}\right)\right) \cdot A(\bx,\yix) \cdot \nabla_{\pi}\log \pit(y^{\bx}_{i,h} \mid \bx, \by_{i,h}^{\bx})\big), \nonumber\\
&\quad\quad\;\;\;  \text{where} \quad w(\bx, y^{\bx}_{i,h}) = \frac{\pit(y^{\bx}_{i,h} \mid \bx,\by^{\bx}_{i,<h})}{\mu(y^{\bx}_{i, h} \mid \bx,\by^{\bx}_{i,<h})} \quad \text{and} \quad \text{token-sum} = \sum_{\bx \in \cal{B}} \sum_{i \in [k]} |\yix|.
\label{eq:cispo-grad}
\end{align}
}In \eqref{eq:cispo-grad}, the advantage $A(\bx, \yix)$ is computed by subtracting the baseline $\bar r(x)=\tfrac{1}{k}\sum_{i=1}^k r(\bx, \yix)$ from $r(\bx, \yix)$. The per-token importance weight $w(\bx, \by^{\bx}_{i,h})$ accounts for the distribution shift between the current policy $\pit$ and the sampler $\mu$ and is $1$ for on-policy traces where $\pit = \mu$. To reduce gradient variance from importance weights, it is clipped at $\varepsilon_\mathrm{high}$ and in practice we set it to a value of $0.01$.  

In our setup, the off-policy dataset $\Doff$ is constructed by \emph{rejection sampling} the base policy $\pi^0$: for each prompt $\bx$, we repeatedly sample $\by \sim \pi^0(\cdot \mid \bx)$ until we obtain one \emph{correct} trajectory (according to the verifier), and store that successful trajectory in $\Doff$.
This procedure induces an \emph{accepted} (conditional) behavior distribution
\begin{align}
\mu_{\mathrm{acc}}(\by \mid \bx)
\;\;=\;\;
\pi^0(\by \mid \bx,\; r(\bx,\by)=1)
\;\;=\;\;
\frac{\pi^0(\by \mid \bx)\,\mathbf{1}\{r(\bx,\by)=1\}}{a(\bx)},
\qquad
a(\bx)\;:=\;\Pr_{\by \sim \pi^0(\cdot \mid \bx)}\!\big[r(\bx,\by)=1\big].
\label{eq:mu_acc}
\end{align}
Thus, when we treat accepted trajectories as “samples from $\mu$” in \eqref{eq:cispo-grad}, the correct sequence-level importance ratio for an accepted trajectory $\by$ is
\begin{align}
\frac{\pi^t(\by \mid \bx)}{\mu_{\mathrm{acc}}(\by \mid \bx)}
\;=\;
a(\bx)\cdot \frac{\pi^t(\by \mid \bx)}{\pi^0(\by \mid \bx)}
\qquad
\text{(since $\mathbf{1}\{r(\bx,\by)=1\}=1$ for accepted $\by$).}
\label{eq:seq_level_ratio}
\end{align}
The key point is that the acceptance-probability correction $a(\bx)$ is a \emph{trajectory-level} factor: it appears once per accepted sequence, not once per token.
In practice, we estimate $a(\bx)$ directly from the rejection-sampling effort.
Let $R(\bx)$ be the number of rollout attempts required to obtain one correct trace for $\bx$ during dataset construction; then $\widehat a(\bx) \approx 1/R(\bx)$. CISPO, however, uses per-token importance weights $w(\bx,y^{\bx}_{i,h})$ (Eq.~\eqref{eq:cispo-grad}) and aggregates gradients across tokens.
A simple way to incorporate the acceptance correction is to multiply \emph{each token} in an accepted trajectory by $\widehat a(\bx)$:
\begin{align}
\widetilde w(\bx, y^{\bx}_{i,h})
\;:=\;
\widehat a(\bx)\cdot
\frac{\pi^t(y^{\bx}_{i,h}\mid \bx,\by^{\bx}_{i,<h})}{\pi^0(y^{\bx}_{i,h}\mid \bx,\by^{\bx}_{i,<h})}.
\label{eq:token_weight_with_accept}
\end{align}
This heuristic is slightly biased relative to the sequence-level ratio in \eqref{eq:seq_level_ratio}: multiplying every token effectively scales an accepted trajectory’s total contribution by approximately $\widehat a(\bx)\cdot |\by|$ (modulo the global token normalization), whereas the exact correction would apply $\widehat a(\bx)$ once at the trajectory level.
When accepted trajectories have heterogeneous lengths, this introduces a mild length-dependent bias.
Empirically, we found this approximation to be stable, and it preserves the intended qualitative effect: prompts that are harder under $\pi^0$ (larger $R(\bx)$, smaller $\widehat a(\bx)$) receive smaller off-policy gradient mass, reflecting the fact that an accepted sample from $\mu_{\mathrm{acc}}$ is “more selective” for those prompts.

\highlight{LUFFY baseline: mixed-policy GRPO with policy shaping.}
We also compare against LUFFY~\citep{yan2025learningreasonoffpolicyguidance}, which incorporates off-policy reasoning traces by mixing them with on-policy rollouts inside a GRPO-style objective.
LUFFY computes advantages using \emph{group computation} over a mixed set of rollouts: for each prompt it combines $N_{\mathrm{on}}$ on-policy samples with $N_{\mathrm{off}}$ off-policy traces, and normalizes rewards using the mean and standard deviation over the union of the two groups.
In our reproduction, we follow LUFFY’s “fair” setting by using $8$ total samples per prompt with a $1$-off-policy / $7$-on-policy split, rollout batch size $128$, update batch size $64$, and rollout temperature $1.0$.
For optimization, LUFFY uses a constant learning rate of $10^{-6}$ and trains for $500$ steps. We also removes KL regularization (setting $\beta=0$) and uses an entropy-loss coefficient of $0.01$.
Finally, LUFFY introduces \emph{policy shaping} via a regularized importance-sampling transformation controlled by a parameter $\gamma$, and for this we use $\gamma=0.1$, chosen after sweeoing over $\{0.01, 0.1, 0.2\}$.

\subsubsection{Hyperparameter Details for PrefixRL}

We use the REINFORCE~\citep{ahmadian2024basicsrevisitingreinforcestyle} on-policy algorithm for PrefixRL and standard RL. For this, we set the training batch size of 128 with, a constant learning rate of $1\times10^{-6}$. We turn off any KL regularization and also disable entropy regularization (entropy coefficient 0). We also use a gradient clipping of 1.0.  We set the sampling temperature for training at 1.0 for Llama3.1-8b-instruct and 0.7 for Qwen-3-4b-instruct. 
At test-time we sample with a temperature of 0.7 for both models, including the inference to collect data for rejection sampling. 
For all our PrefixRL runs we use $n=8$ rollouts per prompt in the batch. We use the same for standard RL, and off-policy RL, except specified otherwise. In all our runs in Section~\ref{sec:experiments} and Section~\ref{sec:back_generalization} we run RL training for 400 iterations, except for our RL runs in Section~\ref{subsec:backgen_icl} and Section~\ref{subsec:backgen_synthetic} where run the training for 100 iterations.   We set the maximum context length $H$ of the output trace to be $16384$ for both the off-policy traces and during the RL run. We do not change this maximum limit for the prefixed problems, depsite the prefix itself possibly being 12k-14k tokens long in some cases. To account for such cases, we set the maximum sequence length to be 36000 tokens, though in practice the completions for the prefixed-problems with longer prefixes are typically much shorter and this limit of 36000 is hit very rarely.    

Before we run RL, we finetune the Llama3.1-8b-instruct model on OpenThoughtsV3~\citep{guha2025openthoughtsdatarecipesreasoning}. For this, we first filter the dataset to only retain responses of token length $<24192$. Then, we run SFT for 5 epochs on this data at peak learning rate of $8e-5$. We use a batch size of 512 traces per batch and a cosine learning rate schedule that has a linear warm up (for 10\% of the total iterations) followed by a cosine decay to a learning rate of 8e-6. We use a hold out validation set to measure the negative log-likelihood loss during training, and pick the earliest checkpoint with the least validation loss as the final distilled model.

\subsection{Evaluation Protocols and FLOPs accounting.}
\label{app:additional_compute_accounting}

All evaluation results in Section~\ref{sec:experiments} are on the original no-prefix problems. For the plots where we report pass@$k$, we estimate it by drawing 256 samples per problem and using the bootstrapped estimate in~\cite{chen2021evaluating}.  Where possible, we include 95\% confidence intervals across evaluated problems. Details on FLOPs accounting for our compute-matched plots are as follows. 
We compute FLOPs using the standard Transformer approximation: processing $D$ tokens with a model of $N$ trainable parameters costs $\approx 2ND$ FLOPs for a \emph{forward-only} pass (sampling/inference) and
$\approx 6ND$ FLOPs for a \emph{training update} (forward + backward + gradient computation)~\citep{snell2024scalingllmtesttimecompute}.

\highlight{Definitions.} We use $N$ to denote the number of trainable parameters of the model whose compute is being measured and $D$ to denote the total number of tokens processed by that model in the relevant stage, summed over all sequences (prompt + generated tokens).

\highlight{Per-iteration compute.}
At RL iteration $t$, let $D^{(t)}_{\mathrm{samp}}$ be the total number of tokens generated/evaluated during rollout sampling, and let $D^{(t)}_{\mathrm{upd}}$ be the total number of tokens used in gradient-based optimization.
We estimate FLOPs as
\begin{align}
\mathrm{FLOPs}^{(t)} \;\;=\;\; 2N\,D^{(t)}_{\mathrm{samp}} \;+\; 6N\,D^{(t)}_{\mathrm{upd}}.
\label{eq:flops_iter}
\end{align}

\highlight{Cumulative compute.} The x-axis in our compute-matched plots reports cumulative FLOPs after $T$ iterations:
\begin{align}
\mathrm{FLOPs}_{\le T}
\;\;=\;\;
\sum_{t=1}^{T} \;\; \mathrm{FLOPs}^{(t)}.
\label{eq:flops_total}
\end{align}
If a method includes up-front rejection sampling to construct $\mathcal{D}_{\mathrm{off}}$, we add that forward-only cost (also using $2ND$) to \eqref{eq:flops_total}. So, if we need to sample $R$ times before we accept a correct trace for a problem $\bx$, then the total upfront compute spent on the problem is $2RND$.

\subsection{Qwen with Llama Prefixes}
\label{app:qwen_with_llama_prefixes}

\begin{figure}[!h]
    \centering
    \includegraphics[width=0.43\linewidth]{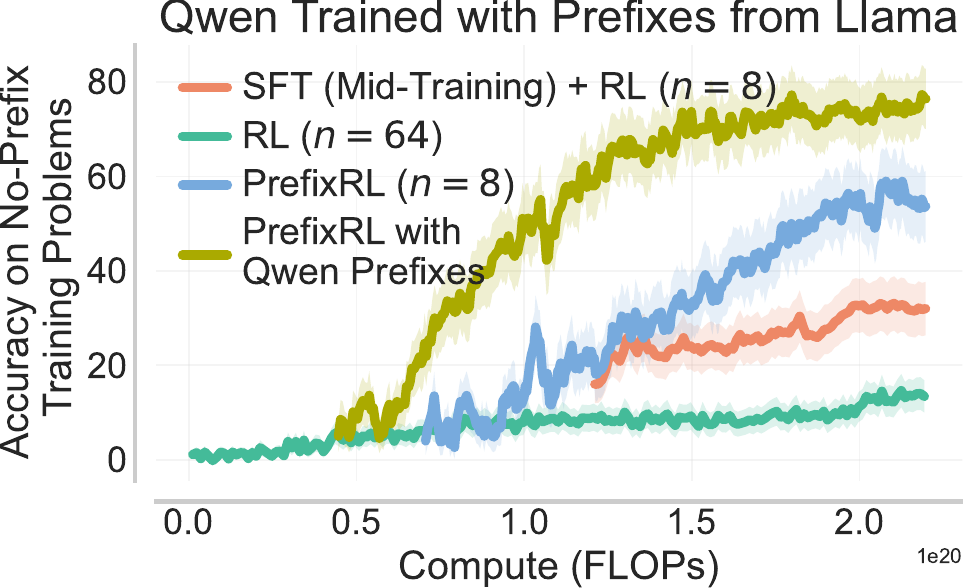}
    \hspace{2em}
    \includegraphics[width=0.43\linewidth]{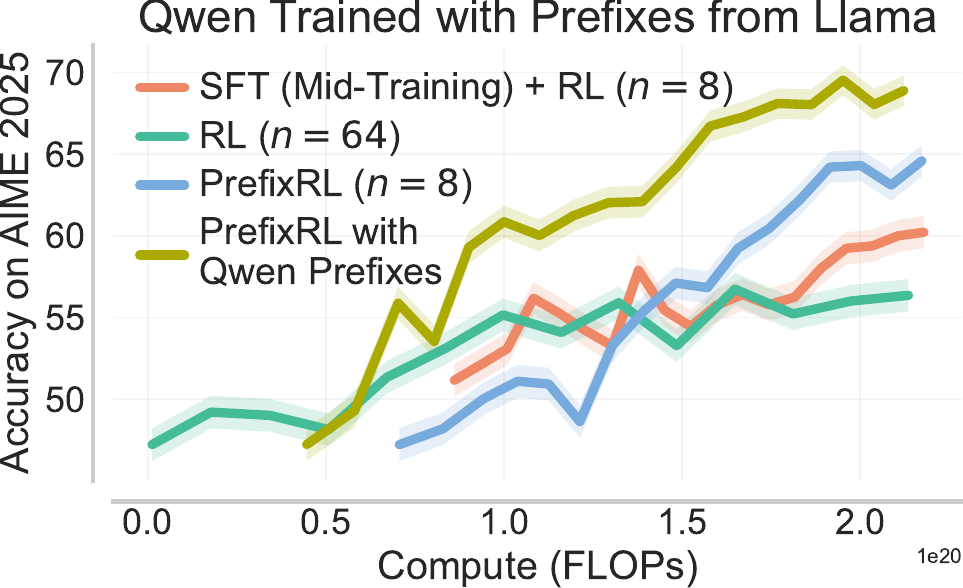}
    \caption{\textbf{Training Qwen model with prefixes from Llama.} Opposite of our experiment in Figure~\ref{fig:prefix_switch}, here we train Qwen3-4b-instruct base LLM using off-policy traces sourced by rejection sampling Llama-3.1-8b-instruct model. We find that while off-policy prefixes from Llama are not as effective at improving Qwen as the other way around. Moreover, we also end up spending more compute upfront on rejection sampling the traces with the less capable Llama base model. We show accuracy on no-prefix training problems \textit{(left)} and AIME 2025 \textit{(right)}.}
    \label{fig:prefix_switch_2}
\end{figure}

Figure~\ref{fig:prefix_switch_2} complements the cross-family back-generalization results in
Figure~\ref{fig:back-gen-qwen} by reversing the direction of prefix transfer: instead of training Llama
using Qwen-sourced off-policy prefixes, we train Qwen using prefixes sourced from Llama.
The key takeaway is asymmetric transfer. While Qwen prefixes can still drive back-generalization when
optimizing Llama (Section~\ref{app:backgen_model_mismatch}), Llama prefixes are noticeably less effective
for improving Qwen in both training accuracy (no-prefix training problems) and standardized evaluation
(AIME 2025).

This asymmetry is consistent with the back-generalization discussion: PrefixRL relies on prefix states that
are injected from an external distribution, and the degree to which learning transfers to no-prefix behavior
depends on how informative and ``compatible'' those prefix states are with the target model's internal
representations and solution strategies. When prefixes are sourced from a more capable model family (here,
Qwen), they tend to encode higher-quality intermediate reasoning states, and RL can more readily propagate
improvements from prefixed states to earlier states and ultimately to the no-prefix setting. In contrast,
prefixes sourced from a less capable model (here, Llama) are both (i) harder to obtain via rejection sampling
and (ii) less likely to contain strategy-revealing intermediate states that meaningfully constrain the
continuation policy, resulting in weaker transfer when training Qwen.

A second practical implication is compute efficiency. Since PrefixRL amortizes training over a fixed pool of
off-policy traces, the total compute depends not only on the on-policy RL phase but also on the \emph{upfront}
cost of harvesting correct traces. Rejection sampling from the weaker base model increases this upfront cost,
and Figure~\ref{fig:prefix_switch_2} shows that even after paying that cost, the resulting prefixes yield smaller
downstream gains for Qwen. Together with Figure~\ref{fig:back-gen-qwen}, these results suggest that cross-family
prefix sourcing is most attractive when (a) the source model is strong enough to produce diverse correct traces
at reasonable cost, and (b) the injected prefix states align with the target model sufficiently well to allow
back-generalization to propagate to the no-prefix distribution.

\subsection{Computing the gradient norm and standard deviation metrics}
\label{app:grad_norm_std}

To quantify training signal-to-noise, we track (i) the \emph{norm of the expected gradient} and (ii) the
\emph{standard deviation of the sampled gradient} throughout RL training, using the same procedure described at
the end of Section~\ref{sec:experiments}. 

Let $g_t \in \mathbb{R}^N$ denote the (flattened) stochastic policy gradient computed at iteration $t$ from the
current minibatch (including any on-policy and/or off-policy contributions, depending on the method).
We maintain exponential moving averages (EMA) of the first and second moments \emph{coordinate-wise}:
\begin{align}
m_t &= \beta m_{t-1} + (1-\beta)\, g_t, \\
v_t &= \beta v_{t-1} + (1-\beta)\, (g_t \odot g_t),
\end{align}
where $\odot$ denotes elementwise multiplication and $\beta \in (0,1)$ is a fixed smoothing constant.

\highlight{Gradient norm.}
We report the norm of the mean gradient estimate as
\begin{align}
\textsc{GradNorm}_t \;=\; \lVert m_t \rVert_2.
\end{align}

\highlight{Gradient standard deviation.}
We estimate the (coordinate-wise) variance as $s_t = v_t - m_t \odot m_t$ and report
\begin{align}
\textsc{GradStd}_t
\;=\;
\sqrt{\sum_{i=1}^{N} \max\{(s_t)_i,\,0\}}
\;=\;
\sqrt{\mathrm{tr}\!\left(\widehat{\mathrm{Cov}}(g_t)\right)}.
\end{align}
Equivalently, this standard-deviation metric corresponds to estimating the trace of the gradient covariance
matrix via first/second moments.

\subsection{Constructing Prefixed Problems: An Example}
\label{subsec:example_prefixed}

\begin{promptbox}{Original Problem (No-Prefix)}
\small 
$<|\text{im}\_\text{start}|>$user

Consider the cube whose vertices are the eight points $(x, y, z)$ for which each of $x, y$, and $z$ is either 0 or 1 . How many ways are there to color its vertices black or white such that, for any vertex, if all of its neighbors are the same color then it is also that color? Two vertices are neighbors if they are the two endpoints of some edge of the cube. Please reason step by step, and put your final answer within boxed.

$<|\text{im}\_\text{end}|>$
\end{promptbox}

\begin{promptbox}{Prefixed Problem}
\small 
$<|\text{im}\_\text{start}|>$user

Consider the cube whose vertices are the eight points $(x, y, z)$ for which each of $x, y$, and $z$ is either 0 or 1 . How many ways are there to color its vertices black or white such that, for any vertex, if all of its neighbors are the same color then it is also that color? Two vertices are neighbors if they are the two endpoints of some edge of the cube. Please reason step by step, and put your final answer within boxed. 
$<|\text{im}\_\text{end}|>$
$<|\text{im}\_\text{start}|>$assistant

<think>
This is a complex or challenging question, and it is difficult to provide a direct and correct answer. I need to think about it.
Well, so we have a cube with vertices at all 0-1 triples, so 8 vertices total. Each vertex has degree 3, right? Because in a cube, each vertex is connected to three others (differing in one coordinate). The problem says: color vertices black or white such that for any vertex, if all its neighbors are the same color, then it's also that color. So let's parse that condition.

Let's rephrase the condition as a logical implication to make it clearer. For a vertex v, let N(v) be its neighbors (3 neighbors). Let c(v) be color, say 0=white, 1=black for simplicity (though labels might not matter, just two colors). The condition is: if c(u) = c(w) for all u, w in N(v), then c(v) = c(u). So in other words, it's not allowed for a vertex to be a different color from all its neighbors when all neighbors are the same color. Wait, actually, "if all neighbors are same color, then it's that color" – so the only way this fails is if there's a vertex where all neighbors are, say, black, but the vertex is white. Similarly, all neighbors white but vertex black is bad.

So forbidden configurations: a vertex with 3 black neighbors being white, or 3 white neighbors being black. All other configurations are allowed? Let's check: if a vertex has neighbors of different colors, then the condition is vacuously true? Wait, no: the condition is a conditional. The antecedent is "all neighbors same color", so if the antecedent is false (i.e., neighbors are not all same color), then the implication is true regardless of the vertex color. Yes! That's key. So the only forbidden things are vertices where all three neighbors are white and the vertex is black, or all three neighbors are black and the vertex is white. For any other coloring (where for every vertex, either it has mixed neighbors or if it has monochromatic neighbors, it's the same color as those neighbors), it's okay.

So maybe another way: for each vertex, when we look at its neighborhood (the three adjacent vertices), if the neighborhood is monochromatic (all one color), then the vertex must be that color. If the neighborhood is not monochromatic (has both colors), no restriction on the vertex color (well, it's already colored, but the condition doesn't impose anything here).

So our goal is to count all 256 colorings satisfying this condition for every vertex.

Maybe it's helpful to model this as a graph. The cube graph Q3, 8 vertices, each degree 3. We need colorings f: V(Q3) → {0,1} such that for all v in V, if f(N(v)) is constant, then f(v) = f(N(v)). So equivalently, for all v, if f(N(v)) not in f(neighbor) for some neighbor (i.e., N(v) not monochromatic), then no problem; if N(v) monochromatic, then f(v) must equal that color.

Let me try to list all possible colorings or find a pattern. Maybe start with colorings that are "monochromatic": all black or all white. Let's check if these work.

All black: take any vertex, all neighbors are black, so condition says it must be black – which it is. Good. Similarly all white: same thing, all neighbors white, so vertex must be white – works. So 2 colorings here.

Now what about colorings with some black, some white. Let's think about smaller cubes? Wait, maybe not. Let's consider linear extensions or something, but maybe better to think about the structure of the cube.

Another idea: in graph theory, such conditions can sometimes be related to "fixed points" under some operation, but maybe not. Alternatively, let's consider the possible colorings by their "stability" or by what constraints they impose.

Wait, let's take an example. Suppose we have a single black vertex, say the origin (0,0,0). Then its neighbors are (1,0,0), (0,1,0), (0,0,1), all white (by assumption). Now check w = (1,0,0), which is white. What are its neighbors? (0,0,0) [black], (1,1,0), (1,0,1). So w has one black neighbor, two unknowns so far. Let's say we color (1,1,0) and (1,0,1) white. Then w has neighbors: black, white, white – so not monochromatic, so condition is satisfied for w. Now check the other endpoints: (1,1,0) is white, neighbors: (1,0,0) [white], (0,1,0) [white], (1,1,1) [unknown]. If we color (1,1,1) black, then (1,1,0) has neighbors: white, white, black – not monochromatic, okay. (1,0,1) similar: neighbors white, white, black (at (1,1,1)), so okay. Now (1,1,1) is black, neighbors: (1,1,0) [white], (1,0,1) [white], (0,1,1) [unknown]. If we color (0,1,1) white, then (1,1,1) has neighbors white, white, white – all white, but it's black! Uh-oh, contradiction! 
\end{promptbox}

%% file: sections_v2/appendix_thm_proof_consistency.tex
Let the standard (no-prefix) RL objective be
\begin{align}
J(\pi)\;=\;\sum_{\mathbf{x}\in \mathcal{D}}\;\E_{\mathbf{y}\sim \pi(\cdot\mid \mathbf{x})}\bigl[r(\mathbf{x},\mathbf{y})\bigr],
\qquad
J^\star \;=\;\max_{\pi\in\Pi} J(\pi).  
\end{align}
For each $\mathbf{x}\in\mathcal{D}$, Assumption~\ref{assm:realizable} gives a single correct trace
$\mathbf{y}^{\mathbf{x}}$ such that $(\mathbf{x},\mathbf{y}^{\mathbf{x}})\in \mathcal{D}_{\mathrm{off}}$.
Given any cut index $h$, define the prefixed problem
\begin{align}
\mathbf{x}_{\mathrm{pre}} \;=\; \mathrm{concat}(\mathbf{x},(\mathbf{y}^{\mathbf{x}})_{1:h}),
\end{align}
and define its reward by evaluating the full transcript:
\begin{align*}
r(\mathbf{x}_{\mathrm{pre}},\mathbf{z})
\;:=\;
r\!\bigl(\mathbf{x},\,(\mathbf{y}^{\mathbf{x}})_{1:h}\circ \mathbf{z}\bigr).
\end{align*}
The PrefixRL objective is
\begin{align}
J_{\mathrm{pre}}(\pi)
\;=\;
\sum_{\mathbf{x}_{\mathrm{pre}}\in \mathcal{D}_{\mathrm{pre}}}\E_{\mathbf{z}\sim \pi(\cdot\mid \mathbf{x}_{\mathrm{pre}})}\bigl[r(\mathbf{x}_{\mathrm{pre}},\mathbf{z})\bigr]
\;+\;
\sum_{\mathbf{x}\in \mathcal{D}}\E_{\mathbf{y}\sim \pi(\cdot\mid \mathbf{x})}\bigl[r(\mathbf{x},\mathbf{y})\bigr].
\end{align}

\highlight{A uniform upper bound.}
Fix any $\mathbf{x}\in\mathcal{D}$ and any prefix $(\mathbf{y}^{\mathbf{x}})_{1:h}$ used to form a prefixed problem.
For any policy $\pi\in\Pi$ defined on such prefixed problems, construct a policy $\tilde\pi\in\Pi$ for the no-prefix problem $\mathbf{x}$ that deterministically emits the prefix $(\mathbf{y}^{\mathbf{x}})_{1:h}$ and then samples the suffix from $\pi(\cdot\mid \mathrm{concat}(\mathbf{x},(\mathbf{y}^{\mathbf{x}})_{1:h}))$.
By the reward definition,
\begin{align*}
\E_{\mathbf{y}\sim \tilde\pi(\cdot\mid \mathbf{x})}\bigl[r(\mathbf{x},\mathbf{y})\bigr]
\;=\;
\E_{\mathbf{z}\sim \pi(\cdot\mid \mathrm{concat}(\mathbf{x},(\mathbf{y}^{\mathbf{x}})_{1:h}))}\bigl[r(\mathrm{concat}(\mathbf{x},(\mathbf{y}^{\mathbf{x}})_{1:h}),\mathbf{z})\bigr].
\end{align*}
Therefore, for every $\pi$ and every such prefixed problem $\mathbf{x}_{\mathrm{pre}}=\mathrm{concat}(\mathbf{x},(\mathbf{y}^{\mathbf{x}})_{1:h})$,
\begin{align}
\E_{\mathbf{z}\sim \pi(\cdot\mid \mathbf{x}_{\mathrm{pre}})}\bigl[r(\mathbf{x}_{\mathrm{pre}},\mathbf{z})\bigr]
\;\le\;
\max_{\pi'\in\Pi}\E_{\mathbf{y}\sim \pi'(\cdot\mid \mathbf{x})}\bigl[r(\mathbf{x},\mathbf{y})\bigr].
\end{align}
Let $m(\mathbf{x})$ be the number of prefixed problems in $\mathcal{D}_{\mathrm{pre}}$ derived from $\mathbf{x}$.
Summing the above inequality over all prefixed problems and grouping by their originating $\mathbf{x}$ gives
\begin{align}
\sum_{\mathbf{x}_{\mathrm{pre}}\in \mathcal{D}_{\mathrm{pre}}}\E\bigl[r(\mathbf{x}_{\mathrm{pre}},\cdot)\bigr]
\;\le\;
\sum_{\mathbf{x}\in\mathcal{D}} m(\mathbf{x})\,
\max_{\pi'\in\Pi}\E_{\mathbf{y}\sim \pi'(\cdot\mid \mathbf{x})}\bigl[r(\mathbf{x},\mathbf{y})\bigr].
\end{align}
Define the constant
\begin{align}
C \;:=\; \sum_{\mathbf{x}\in\mathcal{D}} m(\mathbf{x})\,
\max_{\pi'\in\Pi}\E_{\mathbf{y}\sim \pi'(\cdot\mid \mathbf{x})}\bigl[r(\mathbf{x},\mathbf{y})\bigr],
\end{align}
which is independent of $\pi$.
Then for every $\pi\in\Pi$ we have the uniform upper bound
\begin{align}
J_{\mathrm{pre}}(\pi) \;\le\; C + J(\pi).
\label{eq:bound}
\end{align}

\highlight{Tightness using Assumption~\ref{assm:realizable}.}
By realizability in Assumption~\ref{assm:realizable}, there exists a policy $\mu\in\Pi$ such that
\begin{align}
\mu(\mathbf{y}^{\mathbf{x}}\mid \mathbf{x}) = 1, \qquad \forall \mathbf{x}\in\mathcal{D}.
\label{eq:mu_det}
\end{align}
In particular, for any cut $h$, conditioning $\mu$ on the prefix $(\mathbf{y}^{\mathbf{x}})_{1:h}$ yields the deterministic continuation:
\begin{align}
\mu\!\bigl(\mathbf{z}\mid \mathrm{concat}(\mathbf{x},(\mathbf{y}^{\mathbf{x}})_{1:h})\bigr)
\;=\;
\mathbf{1}\{\mathbf{z}=(\mathbf{y}^{\mathbf{x}})_{h+1:|\mathbf{y}^{\mathbf{x}}|}\}.
\label{eq:mu_suffix}
\end{align}
Therefore, on every prefixed problem $\mathbf{x}_{\mathrm{pre}}=\mathrm{concat}(\mathbf{x},(\mathbf{y}^{\mathbf{x}})_{1:h})$,
\begin{align}
\E_{\mathbf{z}\sim \mu(\cdot\mid \mathbf{x}_{\mathrm{pre}})}\bigl[r(\mathbf{x}_{\mathrm{pre}},\mathbf{z})\bigr]
\;=\;
r\!\bigl(\mathbf{x},\mathbf{y}^{\mathbf{x}}\bigr)
\;=\;1,
\end{align}
where the last equality uses correctness of $\mathcal{D}_{\mathrm{off}}$ in Assumption~\ref{assm:realizable}.
Hence,
\begin{align}
\sum_{\mathbf{x}_{\mathrm{pre}}\in \mathcal{D}_{\mathrm{pre}}}
\E_{\mathbf{z}\sim \mu(\cdot\mid \mathbf{x}_{\mathrm{pre}})}\bigl[r(\mathbf{x}_{\mathrm{pre}},\mathbf{z})\bigr]
\;=\;
\sum_{\mathbf{x}\in\mathcal{D}} m(\mathbf{x}).
\label{eq:pref_mu}
\end{align}
Moreover, by \eqref{eq:mu_det} and correctness,
\begin{align}
J(\mu)
\;=\;
\sum_{\mathbf{x}\in \mathcal{D}} \E_{\mathbf{y}\sim \mu(\cdot\mid \mathbf{x})}\bigl[r(\mathbf{x},\mathbf{y})\bigr]
\;=\;
\sum_{\mathbf{x}\in\mathcal{D}} r(\mathbf{x},\mathbf{y}^{\mathbf{x}})
\;=\;
|\mathcal{D}|.
\end{align}
Therefore $J^\star = |\mathcal{D}|$ and $J(\mu)=J^\star$.
Since rewards are in $[0,1]$, for each $\mathbf{x}$ we have
$\max_{\pi'\in\Pi}\E_{\mathbf{y}\sim \pi'(\cdot\mid \mathbf{x})}[r(\mathbf{x},\mathbf{y})]=1$,
and thus the constant simplifies to $C=\sum_{\mathbf{x}\in\mathcal{D}} m(\mathbf{x})$.
Combining with \eqref{eq:pref_mu} yields
\begin{align}
J_{\mathrm{pre}}(\mu)\;=\; C + J(\mu)\;=\; C + J^\star.
\label{eq:tight}
\end{align}

\highlight{Step 3: Concluding the consistency of PrefixRL objective.}
Let $\hat\pi\in\arg\max_{\pi\in\Pi}J_{\mathrm{pre}}(\pi)$ be any maximizer of the PrefixRL objective.
By optimality of $\hat\pi$, \eqref{eq:tight}, and the upper bound \eqref{eq:bound},
\begin{align}
C+J(\hat\pi)
\;\ge\;
J_{\mathrm{pre}}(\hat\pi)
\;\ge\;
J_{\mathrm{pre}}(\mu)
\;=\;
C+J^\star.
\end{align}
Cancelling $C$ yields $J(\hat\pi)\ge J^\star$, hence $J(\hat\pi)=J^\star$.
Therefore $\hat\pi\in\arg\max_{\pi\in\Pi}J(\pi)$, proving that any maximizer of the PrefixRL objective also maximizes the standard RL objective.

%% file: sections_v2/appendix_thm_proof_subopt.tex
\begin{proof}
We prove the guarantee against the comparator policy $\mu$ from Assumption~\ref{assm:realizable-app}.
Since $\Doff$ is correct and realizable by $\mu$ (i.e., for each $\mathbf{x}\in\mathcal{D}$ the unique correct trace
$\mathbf{y}^{\mathbf{x}}$ satisfies $\mu(\mathbf{y}^{\mathbf{x}}\mid \mathbf{x})=1$), we have $J(\mu)=J^\star$.
Thus it suffices to upper bound:
\begin{align}
\label{eq:suffices_mu}
J(\mu)-J(\bar\pi_T).
\end{align}

\highlight{State distribution induced by $\Doff$.}
Let $\mathbf{s}\in\mathcal{S}$ denote an autoregressive prefix-state.
Algorithm~\ref{alg:npgalg} samples prefix-states by drawing $(\mathbf{x},\mathbf{y})$ uniformly from $\Doff$
and then sampling a prefix of $\mathbf{y}$ (according to the algorithm's prefix-selection rule).
Because $\mu$ deterministically generates the same trace $\mathbf{y}$ in $\Doff$ for each $\mathbf{x}$, this state distribution
coincides with the state-visitation distribution of $\mu$; we denote it by $d^\mu_{\mathbf{s}}$.
\begin{align}
\label{eq:doff_state_dist}
d^\mu_{\mathbf{s}}
\;\equiv\;
d^{\mathrm{off}}_{\mathbf{s}}.
\end{align}

Let $\pi^{0}$ be the base LLM and $\{\pi^t\}_{t=1}^T$ be the iterates produced by NPG / mirror descent with stepsize $\eta$
and critic $\widehat Q^{\,t}$, and define the averaged policy:
\begin{align}
\label{eq:avg_policy_def}
\bar\pi_T \;:=\; \frac{1}{T}\sum_{t=1}^T \pi^t.
\end{align}

\highlight{ Performance difference lemma under $d^\mu_{\mathbf{s}}$.}
Applying performance difference Lemma~\ref{lem:pdl} with $(\pi,\pi^\prime)=(\mu,\pi^t)$ yields:
\begin{align}
\label{eq:pdl_mu_pit}
J(\mu)-J(\pi^t)
\;=\;
\E_{\mathbf{s}\sim d^\mu_{\mathbf{s}}}\E_{\mathbf{a}\sim \mu(\cdot\mid \mathbf{s})}
\bigl[A^{\pi^t}(\mathbf{s},\mathbf{a})\bigr].
\end{align}
Using $A^{\pi^t}(\mathbf{s},\mathbf{a}) = Q^{\pi^t}(\mathbf{s},\mathbf{a}) - V^{\pi^t}(\mathbf{s})$
and the identity $\E_{\mathbf{a}\sim \pi^t(\cdot\mid\mathbf{s})}[A^{\pi^t}(\mathbf{s},\mathbf{a})]=0$, \eqref{eq:pdl_mu_pit} can be rewritten as
\begin{align}
\label{eq:pdl_inner_product}
J(\mu)-J(\pi^t)
\;=\;
\E_{\mathbf{s}\sim d^\mu_{\mathbf{s}}}
\Bigl[
\bigl\langle Q^{\pi^t}(\mathbf{s},\cdot),\,\mu(\cdot\mid\mathbf{s})-\pi^t(\cdot\mid\mathbf{s})\bigr\rangle
\Bigr].
\end{align}

\subsubsection{Critic Estimation Error.}

Fix an iteration $t$.
The critic is fit by least squares over a finite class $\mathcal{F}$ (Assumption~\ref{assm:realizable_q})
using $N$ i.i.d.\ samples $(\mathbf{s}_k,\mathbf{a}_k,z_k)$ where $\mathbf{s}_k\sim d^\mu_{\mathbf{s}}$,
$\mathbf{a}_k\sim \rho^t(\cdot\mid\mathbf{s}_k)$ (see discussion below on $\rho^t$), and $z_k$ is an unbiased target for $Q^{\pi^t}(\mathbf{s}_k,\mathbf{a}_k)$.
Because rewards are \emph{terminal and binary} in $\{0,1\}$, we have
\begin{align}
\label{eq:Q_infty_bound}
0 \le Q^{\pi^t}(\mathbf{s},\mathbf{a}) \le 1,
\qquad
0 \le \widehat Q^{\,t}(\mathbf{s},\mathbf{a}) \le 1,
\end{align}
so we may take $R=1$ in Lemma~\ref{lem:C1-song}.
Therefore, setting $\delta_t:=\delta/(2T)$ and applying Lemma~\ref{lem:C1-song} with $\mathcal{H}=\mathcal{F}$,
with probability at least $1-\delta_t$,
\begin{align}
\label{eq:critic_mse_song}
\E_{\mathbf{s}\sim d^\mu_{\mathbf{s}},\,\mathbf{a}\sim \rho^t(\cdot\mid\mathbf{s})}
\Bigl[\bigl(\widehat Q^{\,t}(\mathbf{s},\mathbf{a})-Q^{\pi^t}(\mathbf{s},\mathbf{a})\bigr)^2\Bigr]
\;\le\;
\frac{256}{N}\log\Bigl(\frac{2|\mathcal{F}|}{\delta_t}\Bigr)
\;=\;
\frac{256}{N}\log\Bigl(\frac{4T|\mathcal{F}|}{\delta}\Bigr).
\end{align}

\highlight{Behavior distribution $\rho^t$ and pointwise domination.}
At iteration $t$, Algorithm~\ref{alg:npgalg} forms critic data by first sampling $(\bs_h,a_h^{\mathrm{off}})$ uniformly from $\Doff$ and then sampling
\begin{align}
\label{eq:rho_def}
a_h \;=\;
\begin{cases}
a_h^{\mathrm{off}} & \text{w.p. } \tfrac{1}{2},\\
a_h \sim \pi^t(\cdot\mid \bs_h) & \text{w.p. } \tfrac{1}{2}.
\end{cases}
\end{align}
Let $\mu(\cdot\mid \bs_h)$ denote the (deterministic) conditional action distribution induced by $\Doff$, i.e.,
\begin{align}
\label{eq:mu_def}
\mu(a\mid \bs_h) \;:=\; \mathbf{1}\{a=a_h^{\mathrm{off}}\}.
\end{align}
Then the induced action-sampling (behavior) distribution used for critic fitting is the mixture
\begin{align}
\label{eq:rho_mixture}
\rho^t(\cdot\mid \bs_h)
\;:=\;
\frac{1}{2}\,\mu(\cdot\mid \bs_h)
\;+\;
\frac{1}{2}\,\pi^t(\cdot\mid \bs_h).
\end{align}
Consequently, for every state $\bs$ and action $a$, we have the pointwise lower bounds
\begin{align}
\label{eq:rho_lowerbounds}
\rho^t(a\mid \bs)\;\ge\;\frac{1}{2}\mu(a\mid \bs),
\qquad
\rho^t(a\mid \bs)\;\ge\;\frac{1}{2}\pi^t(a\mid \bs),
\end{align}
and hence the pointwise domination inequalities
\begin{align}
\label{eq:rho_domination}
\mu(a\mid \bs)\;\le\;2\,\rho^t(a\mid \bs),
\qquad
\pi^t(a\mid \bs)\;\le\;2\,\rho^t(a\mid \bs).
\end{align}
In particular, \eqref{eq:rho_lowerbounds} also implies absolute continuity:
if $\rho^t(a\mid \bs)=0$ then $\mu(a\mid \bs)=\pi^t(a\mid \bs)=0$, so $\mu(\cdot\mid \bs)\ll \rho^t(\cdot\mid \bs)$ and
$\pi^t(\cdot\mid \bs)\ll \rho^t(\cdot\mid \bs)$.

Then Cauchy--Schwarz and Jensen applied to \eqref{eq:critic_mse_song} yield, for $\pi\in\{\mu,\pi^t\}$,
\begin{align}
\label{eq:critic_linear}
\Bigl|
\E_{\mathbf{s}\sim d^\mu_{\mathbf{s}},\,\mathbf{a}\sim \pi(\cdot\mid\mathbf{s})}
\bigl[\widehat Q^{\,t}(\mathbf{s},\mathbf{a})-Q^{\pi^t}(\mathbf{s},\mathbf{a})\bigr]
\Bigr|
\;\le\;
\epsilon_{\mathrm{crt}},
\qquad
\epsilon_{\mathrm{crt}}
:= 16\sqrt{2}\,\sqrt{\frac{1}{N}\log\Bigl(\frac{4T|\mathcal{F}|}{\delta}\Bigr)}.
\end{align}
Taking a union bound over $t\in[T]$, with probability at least $1-\delta/2$, \eqref{eq:critic_linear} holds for all $t$.

\subsubsection{Mirror Ascent and NPG Optimization Error.}
The mirror-descent update at state $\mathbf{s}$ is the KL-regularized maximization
\begin{align*}
\pi^{t+1}(\cdot\mid\mathbf{s})
=
\arg\max_{p(\cdot\mid\mathbf{s})}
\Bigl\{\eta\langle \widehat Q^{\,t}(\mathbf{s},\cdot),p(\cdot\mid\mathbf{s})\rangle
-\KL(p(\cdot\mid\mathbf{s})\|\pi^t(\cdot\mid\mathbf{s}))\Bigr\}.
\end{align*}
The first-order optimality condition implies that for any $p(\cdot\mid\mathbf{s})$,
\begin{align}
\label{eq:md_foc}
\Bigl\langle
-\eta\,\widehat Q^{\,t}(\mathbf{s},\cdot)
+
\nabla_r \KL(r(\cdot\mid\mathbf{s})\|\pi^t(\cdot\mid\mathbf{s}))\big|_{r=\pi^{t+1}},
\ p-\pi^{t+1}
\Bigr\rangle
\;\ge\;0.
\end{align}
Set $p=\mu(\cdot\mid\mathbf{s})$ and rearrange \eqref{eq:md_foc} to obtain
\begin{align}
\label{eq:md_rearranged}
\eta\bigl\langle \widehat Q^{\,t}(\mathbf{s},\cdot),\,\mu(\cdot\mid\mathbf{s})-\pi^{t+1}(\cdot\mid\mathbf{s})\bigr\rangle
\;\le\;
\Bigl\langle
\nabla_r \KL(r\|\pi^t)\big|_{r=\pi^{t+1}},
\ \mu-\pi^{t+1}
\Bigr\rangle.
\end{align}
Apply the KL three-point identity (Lemma~\eqref{eq:kl_three_point_grad}) with
$p=\mu(\cdot\mid\mathbf{s})$, $r=\pi^{t+1}(\cdot\mid\mathbf{s})$, and $q=\pi^t(\cdot\mid\mathbf{s})$ to rewrite the right-hand side:
\begin{align}
\label{eq:three_point_applied}
\Bigl\langle
\nabla_r \KL(r\|\pi^t)\big|_{r=\pi^{t+1}},
\ \mu-\pi^{t+1}
\Bigr\rangle
=
\KL(\mu\|\pi^t) - \KL(\mu\|\pi^{t+1}) - \KL(\pi^{t+1}\|\pi^t).
\end{align}
Combining \eqref{eq:md_rearranged} and \eqref{eq:three_point_applied} gives
\begin{align}
\label{eq:md_basic}
\eta\bigl\langle \widehat Q^{\,t}(\mathbf{s},\cdot),\,\mu-\pi^{t+1}\bigr\rangle
\;\le\;
\KL(\mu\|\pi^t)-\KL(\mu\|\pi^{t+1})-\KL(\pi^{t+1}\|\pi^t).
\end{align}
Using \eqref{eq:Q_infty_bound} and Pinsker's inequality, we bound the shift term
\begin{align}
\label{eq:md_shift}
\eta\bigl\langle \widehat Q^{\,t}(\mathbf{s},\cdot),\,\pi^{t+1}-\pi^{t}\bigr\rangle
\;\le\;
\eta\,\|\widehat Q^{\,t}(\mathbf{s},\cdot)\|_\infty\,\|\pi^{t+1}-\pi^t\|_1
\;\le\;
\frac{\eta^2}{2}+\KL(\pi^{t+1}\|\pi^t),
\end{align}
where the last inequality uses $\|\widehat Q^{\,t}\|_\infty\le 1$ and $\|\pi^{t+1}-\pi^t\|_1^2\le 2\KL(\pi^{t+1}\|\pi^t)$.
Adding \eqref{eq:md_basic} and \eqref{eq:md_shift} cancels $\KL(\pi^{t+1}\|\pi^t)$ and yields
\begin{align}
\label{eq:md_final_statewise}
\eta\bigl\langle \widehat Q^{\,t}(\mathbf{s},\cdot),\,\mu-\pi^{t}\bigr\rangle
\;\le\;
\KL(\mu\|\pi^t)-\KL(\mu\|\pi^{t+1})+\frac{\eta^2}{2}.
\end{align}
Taking expectation over $\mathbf{s}\sim d^\mu_{\mathbf{s}}$ and summing over $t=1,\dots,T$ gives
\begin{align}
\label{eq:md_summed}
\frac{1}{T}\sum_{t=1}^T
\E_{\mathbf{s}\sim d^\mu_{\mathbf{s}}}
\Bigl[\bigl\langle \widehat Q^{\,t}(\mathbf{s},\cdot),\,\mu-\pi^{t}\bigr\rangle\Bigr]
\;\le\;
\frac{D_1}{\eta T}+\frac{\eta}{2},
\end{align}
where
\begin{align}
\label{eq:D1_def}
D_1
\;:=\;
\E_{\mathbf{s}\sim d^\mu_{\mathbf{s}}}\Bigl[\KL\bigl(\mu(\cdot\mid\mathbf{s})\,\|\,\pi^1(\cdot\mid\mathbf{s})\bigr)\Bigr].
\end{align}
Using $\pi^1=\pi^{0}$ and the definition of $\KL(\mu\|\pi^{0})$ under the $\Doff$-induced state distribution (Assumption~\ref{assm:finite_KL}),
we identify
\begin{align}
\label{eq:D1_equals_KL}
D_1 \;=\; \KL(\mu\|\pi^{0}).
\end{align}

\subsubsection{Combining Critic Error and Optimization Error.}
Starting from \eqref{eq:pdl_inner_product}, add and subtract $\widehat Q^{\,t}$:
\begin{align}
\label{eq:pdl_split_final}
J(\mu)-J(\pi^t)
&=
\E_{\mathbf{s}\sim d^\mu_{\mathbf{s}}}
\Bigl[\bigl\langle \widehat Q^{\,t}(\mathbf{s},\cdot),\,\mu-\pi^t\bigr\rangle\Bigr]
\nonumber\\
&\quad+
\E_{\mathbf{s}\sim d^\mu_{\mathbf{s}}}
\Bigl[\bigl\langle Q^{\pi^t}(\mathbf{s},\cdot)-\widehat Q^{\,t}(\mathbf{s},\cdot),\,\mu-\pi^t\bigr\rangle\Bigr].
\end{align}
On the high-probability event where \eqref{eq:critic_linear} holds for both $\pi=\mu$ and $\pi=\pi^t$,
the critic-error term is bounded by
\begin{align}
\label{eq:critic_term_noH}
\E_{\mathbf{s}\sim d^\mu_{\mathbf{s}}}
\Bigl[\bigl\langle Q^{\pi^t}-\widehat Q^{\,t},\,\mu-\pi^t\bigr\rangle\Bigr]
\;\le\;
\Bigl|\E_{\mathbf{s}\sim d^\mu_{\mathbf{s}},\,\mathbf{a}\sim \mu}\bigl[\widehat Q^{\,t}-Q^{\pi^t}\bigr]\Bigr|
+
\Bigl|\E_{\mathbf{s}\sim d^\mu_{\mathbf{s}},\,\mathbf{a}\sim \pi^t}\bigl[\widehat Q^{\,t}-Q^{\pi^t}\bigr]\Bigr|
\;\le\; 2\,\epsilon_{\mathrm{crt}},
\end{align}
which introduces no extra factor of $H$.

Averaging \eqref{eq:pdl_split_final} over $t=1,\dots,T$ and using \eqref{eq:avg_policy_def}, we obtain
\begin{align}
\label{eq:combine_avg}
J(\mu)-J(\bar\pi_T)
\;\le\;
\frac{1}{T}\sum_{t=1}^T
\E_{\mathbf{s}\sim d^\mu_{\mathbf{s}}}
\Bigl[\bigl\langle \widehat Q^{\,t}(\mathbf{s},\cdot),\,\mu-\pi^t\bigr\rangle\Bigr]
\;+\; 2\,\epsilon_{\mathrm{crt}}.
\end{align}
Applying \eqref{eq:md_summed} to the first term in \eqref{eq:combine_avg} yields
\begin{align}
\label{eq:pre_eta_final}
J(\mu)-J(\bar\pi_T)
\;\le\;
\frac{D_1}{\eta T} + \frac{\eta}{2} + 2\,\epsilon_{\mathrm{crt}}.
\end{align}
Choose $\eta := \sqrt{\frac{2D_1}{T}}$ to balance the first two terms in \eqref{eq:pre_eta_final}, giving
\begin{align}
\label{eq:opt_choice_eta}
\frac{D_1}{\eta T} + \frac{\eta}{2}
\;=\;
\sqrt{\frac{2D_1}{T}}.
\end{align}
Combining \eqref{eq:D1_equals_KL}, \eqref{eq:critic_linear}, and \eqref{eq:pre_eta_final}, and recalling $J(\mu)=J^\star$,
we conclude that with probability at least $1-\delta$,
\begin{align}
\label{eq:final_bound_rewritten}
J(\pi^\star)-J(\bar\pi_T)
\;=\;
J^\star - J(\bar\pi_T)
\;\le\;
\mathcal{O}\left(
\sqrt{\frac{\KL(\mu\|\pi^{0})}{T}}
\;+\;
\sqrt{\frac{1}{N}\log\Bigl(\frac{T|\mathcal{F}|}{\delta}\Bigr)}
\right).
\end{align}
\end{proof}

%% file: sections_v2/appendix_thm_proof_lb.tex
Now, we prove our separation result in Proposition~\ref{prp:worstcase_gap} that lower bounds the performance gap between standard RL and PrefixRL. Here, standard RL runs Algorithm~\ref{alg:npgalg} but now without any access to the off-policy dataset $\Doff$. In each iteration, the critic training dataset $\mathcal{D}_t$ is now populated with $(\bs,a,r)$ tuples where both $\bs$ and $a$ are sampled from the current policy $\pit$. So, unlike PrefixRL we never sample the state or prefix from $\Doff$.  In this simple worst-case instance we present below there is a single trajectory in $\Doff$ that is also what that the optimal policy samples with probability $1$ and attains performance of $J(\pi^\star)=1$.

\begin{proof}
We present (i) the MDP instance together with a choice of \emph{base policy} that generates the off-policy trace, and then (ii) an exponential lower bound for standard on-policy RL without $\Doff$, and (iii) a horizon-independent (non-exponential) upper bound for PrefixRL with $\Doff$.

\highlight{MDP instance (hidden rewarding binary string) and base policy.}
Fix a horizon $H$ and an unknown binary string $\mathbf{b}=(b_1,\ldots,b_H)\in\{0,1\}^{H}$.
Let the state space be
\begin{align}
\label{eq:prefix_states_v2}
\mathcal{S} \;=\; \{s_0,s_1,\ldots,s_H\},
\end{align}
where $s_{h-1}$ encodes the first $h-1$ actions taken so far ($s_0$ is the start state).
The action space is $\mathcal{A}=\{0,1\}$.
Transitions are deterministic: from $s_{h-1}$, taking action $a_h\in\{0,1\}$ moves to $s_h$.
The episode ends at $s_H$ with terminal reward
\begin{align}
\label{eq:hidden_string_reward_v2}
r \;=\; \mathbf{1}\{(a_1,\ldots,a_H)=(b_1,\ldots,b_H)\}.
\end{align}
Thus, exactly one length-$H$ action sequence earns reward $1$.

Let $\pi^\star$ be the deterministic policy that selects $b_h$ at $s_{h-1}$ for each $h\in[H]$. Then $J(\pi^\star)=1$.
For the PrefixRL part, we also choose a \emph{base policy} $\mu$ and an off-policy dataset $\Doff$:
we set $\mu:=\pi^\star$ and let $\Doff$ contain the unique successful trajectory of $\mu$, equivalently the state--action pairs
\begin{align}
\label{eq:Doff_pairs_v2}
\Doff \;=\;\{(s_{h-1},b_h):h\in[H]\}.
\end{align}

\subsubsection{Exponential lower bound for standard on-policy RL (Algorithm~1 without $\Doff$).}
We analyze an on-policy variant of Algorithm~1 in which \emph{there is no off-policy dataset}: at each iteration $t$, the algorithm samples
$N$ full episodes only from its current policy $\pi^t$, observes only terminal rewards $r^{(t,i)}\in\{0,1\}$, fits a critic, and updates the policy.
Let $\widehat\pi_T$ denote the (possibly randomized) policy output after $T$ iterations (so the total number of full episodes is $TN$).
We prove that for any such algorithm, there exists an instance $\mathbf b$ for which the expected suboptimality gap is at least
$1-(TN+1)2^{-H}$.

\highlight{Yao's minimax setup.}
By Yao's minimax principle~\cite{yaominimax}, it suffices to fix an arbitrary adaptive algorithm and analyze its expected performance when the instance is random:
\begin{align}
\mathbf{b}\sim \mathrm{Unif}(\{0,1\}^H),
\end{align}
and we write $\mathbb{P}_{\mathbf{b}},\mathbb{E}_{\mathbf{b}}$ for probability/expectation over this draw (and over the algorithm's internal randomness).

\highlight{Per-rollout success probability is $2^{-H}$.}
Fix any rollout index $(t,i)$. Condition on the full interaction history up to this rollout and on the algorithm's internal randomness.
Under this conditioning, the action string $\mathbf a^{(t,i)}\in\{0,1\}^H$ is some random element (with an arbitrary distribution induced by the algorithm),
while $\mathbf b$ remains uniform and independent. Therefore,
\begin{align}
\label{eq:hit_prob_one_rollout_v2}
\mathbb{P}_{\mathbf{b}}\!\left[r^{(t,i)}=1 \mid \text{history}\right]
&=
\mathbb{P}_{\mathbf{b}}\!\left[\mathbf{a}^{(t,i)}=\mathbf{b}\mid \text{history}\right]
=
\sum_{\mathbf{a}\in\{0,1\}^H}
\mathbb{P}\!\left[\mathbf{a}^{(t,i)}=\mathbf{a}\mid \text{history}\right]\cdot
\mathbb{P}_{\mathbf{b}}[\mathbf{b}=\mathbf{a}]
=
2^{-H}.
\end{align}
Taking expectation over the history yields the unconditional version:
\begin{align}
\label{eq:hit_prob_one_rollout_uncond_v2}
\mathbb{P}_{\mathbf{b}}[r^{(t,i)}=1] \;=\; 2^{-H}.
\end{align}

\highlight{Probability of ever seeing a reward-1 rollout.}
There are exactly $TN$ rollouts in total. By a union bound and \eqref{eq:hit_prob_one_rollout_uncond_v2} we get the following upper bound on the probability of seeing a reward $1$ rollout across all $T$ steps,
\begin{align}
\label{eq:any_hit_prob_TN_v2}
\mathbb{P}_{\mathbf{b}}\!\left[\exists\,t\le T,\ i\le N:\ r^{(t,i)}=1\right]
\;\le\;
\sum_{t=1}^T\sum_{i=1}^N \mathbb{P}_{\mathbf{b}}[r^{(t,i)}=1]
\;=\;
TN\cdot 2^{-H}.
\end{align}

\highlight{Bound the expected value of the returned policy.}
Let $\mathbf a\sim\widehat\pi_T$ denote the length-$H$ string generated by rolling out $\widehat\pi_T$ from $s_0$.
On instance $\mathbf b$, $J(\widehat\pi_T)=\mathbb{P}[\mathbf a=\mathbf b]$.
We can decompose this reward on two events: whether the training interaction ever produced a reward-1 rollout or not:
\begin{align}
\label{eq:value_split_v2}
\mathbb{E}_{\mathbf{b}}[J(\widehat\pi_T)]
&\le
\mathbb{P}_{\mathbf{b}}\!\left[\exists\,t,i:\ r^{(t,i)}=1\right]\cdot 1
+
\mathbb{P}_{\mathbf{b}}\!\left[\forall\,t,i:\ r^{(t,i)}=0\right]\cdot
\sup_{\widehat\pi_T}\ \mathbb{E}_{\mathbf{b}}\!\left[J(\widehat\pi_T)\mid \forall\,t,i:\ r^{(t,i)}=0\right].
\end{align}

On the event $\{\forall\,t,i:\ r^{(t,i)}=0\}$, the algorithm has never observed the rewarding string.
Each zero-reward rollout rules out \emph{at most one} candidate string, namely the realized action string $\mathbf a^{(t,i)}$.
Hence, conditioning on $\{\forall\,t,i:\ r^{(t,i)}=0\}$ only implies that $\mathbf b$ is not in a set of at most $TN$ excluded strings.
Under the prior $\mathbf b\sim\mathrm{Unif}(\{0,1\}^H)$, the posterior is uniform over the remaining candidates, so
\begin{align}
\label{eq:posterior_max_mass_TN_v2}
\max_{\mathbf a\in\{0,1\}^H}\ 
\mathbb P_{\mathbf b}\!\left[\mathbf b=\mathbf a \mid \forall\,t,i:\ r^{(t,i)}=0\right]
\;\le\;
\frac{1}{2^H-TN},
\end{align}
and therefore
\begin{align}
\label{eq:cond_value_bound_v2_fixed}
\sup_{\widehat\pi_T}\ 
\mathbb{E}_{\mathbf{b}}\!\left[J(\widehat\pi_T)\mid \forall\,t,i:\ r^{(t,i)}=0\right]
\;=\;
\sup_{\widehat\pi_T}\ 
\mathbb{E}_{\mathbf{b}}\!\left[\mathbb P(\mathbf a=\mathbf b\mid \widehat\pi_T,\forall\,t,i:\ r^{(t,i)}=0)\right]
\;\le\;
\frac{1}{2^H-TN}.
\end{align}
Substituting \eqref{eq:any_hit_prob_TN_v2} and \eqref{eq:cond_value_bound_v2_fixed} into \eqref{eq:value_split_v2} yields
\begin{align}
\label{eq:E_value_returned_TN_v2_fixed}
\mathbb{E}_{\mathbf{b}}[J(\widehat\pi_T)]
\;\le\;
\mathbb P_{\mathbf b}\!\left[\exists\,t,i:\ r^{(t,i)}=1\right]\cdot 1
+
\mathbb P_{\mathbf b}\!\left[\forall\,t,i:\ r^{(t,i)}=0\right]\cdot \frac{1}{2^H-TN}
\;\le\;
TN\cdot 2^{-H} + \frac{1}{2^H-TN}.
\end{align}
In particular, whenever $TN\le 2^{H-1}$ we have $\frac{1}{2^H-TN}\le 2^{-(H-1)}$, and thus
\begin{align}
\label{eq:E_value_returned_simple_TN_v2_fixed}
\mathbb{E}_{\mathbf{b}}[J(\widehat\pi_T)]
\;\le\;
TN\cdot 2^{-H} + 2^{-(H-1)}
\;\le\;
(TN+2)\,2^{-(H-1)}.
\end{align}

\highlight{Suboptimality lower bound and fix an instance.}
Since $J(\pi^\star)=1$, we obtain from \eqref{eq:E_value_returned_TN_v2_fixed} the gap bound
\begin{align}
\label{eq:gap_exp_lb_onpolicy_v2_fixed}
\mathbb{E}_{\mathbf{b}}[J(\pi^\star)-J(\widehat\pi_T)]
\;\ge\;
1-(TN+2)\,2^{-(H-1)}.
\end{align}
By Yao's minimax principle~\citep{yaominimax}, there exists a fixed hidden string $\mathbf b$ for which the same bound holds for the algorithm on that instance.
This proves the exponential lower bound for standard on-policy RL.

\subsubsection{Horizon-independent (non-exponential) upper bound for PrefixRL.}

We now analyze Algorithm~\ref{alg:npgalg} on the same instance \emph{with} $\Doff$.
As discussed, the separation mechanism is that Algorithm~\ref{alg:npgalg} explicitly samples states from $\Doff$, which are precisely the prefix states
along the $\pi^\star$ trajectory, thereby forcing visitation of optimal-trajectory states.

\highlight{Choosing $\pibase$.}
Let $\pibase$ denote the initialization policy of Algorithm~\ref{alg:npgalg}, \textit{i.e.} $\pi^1=\pibase$.
We choose $\pibase$ to be the uniform policy on $\{0,1\}$ at every state \textit{(crucially, $\pibase$ is not instance dependent)}:
\begin{align}
\label{eq:pi0_uniform_rewrite}
\pi_0(0\mid s)=\pi_0(1\mid s)=\tfrac12
\qquad\text{for all } s\in\{s_0,\ldots,s_{H-1}\}.
\end{align}
Since $\mu=\pi^\star$ is deterministic on the states in $\Doff$, the KL term in Theorem~3.1 is not exponential in $H$.
Under the state-averaged convention used in Theorem~3.1,
\begin{align}
\label{eq:KL_mu_pi0_const_rewrite}
\KL(\mu\|\pi_0)
:=\E_{s\sim d^\mu_s}\Big[\KL\big(\mu(\cdot\mid s)\,\|\,\pi_0(\cdot\mid s)\big)\Big]
\;=\; \log 2,
\end{align}
while under the summed convention it is $H\log 2$; in either case it is not exponential in $H$.

\highlight{Invoking PrefixRL guarantee in Theorem~\ref{thm:suboptgap}).}
All assumptions required by Theorem~\ref{thm:suboptgap} hold on this instance with $\mu=\pi^\star$ and the corresponding $\Doff=\{\mathbf{b}\}$.
Therefore, applying Theorem~\ref{thm:suboptgap} to Algorithm~\ref{alg:npgalg} yields (with probability at least $1-\delta$) the bound
\begin{align}
\label{eq:prefixrl_invoke_rewrite}
J(\mu)-J(\bar\pi_T)
\;\le\;
\sqrt{\frac{2\KL(\mu\|\pi_0)}{T}}
\;+\;
\mathcal{O}\!\left(
\sqrt{\frac{1}{N}\log\!\Bigl(\frac{T|\mathcal{F}|}{\delta}\Bigr)}
\right),
\end{align}
where $\bar\pi_T$ is the iterate-averaged policy output by Algorithm~1.
Since $\mu=\pi^\star$ on this instance, $J(\mu)=J(\pi^\star)=1$, and \eqref{eq:prefixrl_invoke_rewrite} implies a non-exponential
suboptimality bound for PrefixRL. In particular, with the choice \eqref{eq:pi0_uniform_rewrite} we have $\KL(\mu\|\pi_0)=\log 2$ (or $H\log 2$
under the summed convention), so the bound has no $2^{-H}$ term.

\highlight{Separation mechanism.}
The on-policy exponential lower bound arises because, without $\Doff$, an algorithm only observes nonzero reward if it guesses the entire length-$H$
string correctly in a single episode.
In contrast, Algorithm~1 with $\Doff$ repeatedly samples \emph{states along the $\pi^\star$ trajectory} via $\sim\Doff$ and
trains at those states using the mixture distribution of Algorithm~\ref{alg:npgalg}.
Overall, this yields a non-exponential sample complexity, establishing a worst-case separation.

\subsubsection{Worst-Case Separation Result Between Standard RL and PrefixRL.}

The analysis in the above subsection can be viewed either for a fixed hidden string $\mathbf b$, or under a randomized instance distribution.
In particular, let$
\mathbf b \sim \mathrm{Unif}(\{0,1\}^H)$,
and note that under this randomization the off-policy dataset $\Doff=\Doff(\mathbf b)$ also changes with $\mathbf b$ since it contains the unique
successful trajectory $(s_{h-1},b_h)_{h=1}^H$.

\highlight{Lower bound for standard on-policy RL under random $\mathbf b$.}
For any on-policy algorithm that runs for $T$ iterations with $N$ full episodes per iteration (so $TN$ total episodes) and does not have access to $\Doff$,
the expected value of the suboptimality gap satisfies
\begin{align}
\label{eq:standardRL_random_b_lb}
\mathbb E_{\mathbf b}\!\left[J(\pi^\star_{\mathbf b})-J(\widehat\pi_T)\right]
\;\ge\;
1-(TN+2)\,2^{-(H-1)},
\end{align}
where $\pi^\star_{\mathbf b}$ denotes the optimal policy for instance $\mathbf b$ and $TN \leq 2^{H-1}$.

\highlight{Upper bound for PrefixRL under random $\mathbf b$.}
Now consider Algorithm~1 (PrefixRL) run on the same randomized instance, where $\Doff=\Doff(\mathbf b)$ is provided to the algorithm.
Choose the initialization $\pi_0$ to be the uniform policy (so that $\KL(\mu_{\mathbf b}\|\pi_0)$ is not exponential in $H$), with
$\mu_{\mathbf b}:=\pi^\star_{\mathbf b}$ as the base policy that generates $\Doff(\mathbf b)$.
Invoking the previously proved Theorem~3.1 yields (with probability at least $1-\delta$ over the algorithm's sampling)
\begin{align}
\label{eq:prefixRL_random_b_ub}
\mathbb E_{\mathbf b}\!\left[J(\pi^\star_{\mathbf b})-J(\bar\pi_T)\right]
\;\le\;
\sqrt{\frac{2\,\mathbb E_{\mathbf b}\!\left[\KL(\mu_{\mathbf b}\|\pi_0)\right]}{T}}
\;+\;
\mathcal{O}\!\left(
\sqrt{\frac{1}{N}\log\!\Bigl(\frac{T|\mathcal{F}|}{\delta}\Bigr)}
\right),
\end{align}
and for the uniform initialization $\pi_0$ we have $\mathbb E_{\mathbf b}[\KL(\mu_{\mathbf b}\|\pi_0)]=\log 2$ under the state-averaged convention
(or $H\log 2$ under the summed convention), which is not exponential in $H$.

\highlight{Applying Yao's minimax principle.}
Equations \eqref{eq:standardRL_random_b_lb} and \eqref{eq:prefixRL_random_b_ub} establish an \emph{average-case} separation under the uniform distribution
over instances $\mathbf b$ (with $\Doff$ coupled to $\mathbf b$ in the PrefixRL case).
By Yao's minimax principle, this implies that there exists a fixed instance $\mathbf b$ (and the above choice of initialization $\pi_0$, \textit{e.g.}\ uniform)
for which the same separation holds for any randomized standard on-policy algorithm without $\Doff$ versus Algorithm~\ref{alg:npgalg} with $\Doff$.

\end{proof}

%% file: sections_v2/appendix_auxiliary_lemmas.tex
\begin{lemma}
[Performance difference lemma;~\citep{kakade2002approximately}]
\label{lem:pdl}
For all policies $\pi, \pi^\prime$ and initial state distribution 
$\rho$,
\begin{align}
      \E_{\bs_0 \sim \rho} \brck{ V^\pi(\bs_0) -  V^{\pi^\prime}(\bs_0)}
= \E_{\bs_h \sim d_{\bs}^\pi }\E_{a_h\sim \pi(\cdot\mid \bs_h) }
\left[A^{\pi^\prime}(\bs_h,a_h)\right].
\end{align}
\end{lemma}
\begin{proof}
See proof of Lemma 6.1 in \citet{kakade2002approximately}.
\end{proof}

\begin{lemma}[Three-point identity for KL]
Let $p,q,r$ be distributions on a common measurable space such that $p \ll r \ll q$ and all quantities below are finite. Then
\begin{align}
\KL(p\|q)
=
\KL(p\|r)
+
\KL(r\|q)
+
\left\langle p-r,\ \nabla_r \KL(r\|q)\right\rangle,
\label{eq:kl_three_point_grad}
\end{align}.
For discrete distributions $p, q, r$ we have:
\begin{align}
\nabla_r \KL(r\|q)(x) \;=\; \log\frac{r(x)}{q(x)} + 1 .
\label{eq:grad_kl_rq}
\end{align}
Equivalently we can state this the three-point identity for discrete distributions as:
\begin{align}
\KL(p\|q)
=
\KL(p\|r)
+
\KL(r\|q)
+
\left\langle p-r,\ \log\frac{r}{q}\right\rangle,
\label{eq:kl_three_point_log}
\end{align}
since $\langle p-r,1\rangle = \int (p-r)=0$ (or $\sum_i (p_i-r_i)=0$).
\end{lemma}

\begin{proof}
We work in the continuous case; the discrete case is identical with integrals replaced by sums.
Recall
\begin{align*}
\KL(a\|b)=\int a(x)\log\frac{a(x)}{b(x)}\,dx.
\end{align*}
Consider the difference between the left-hand side and the first two KL terms:
\begin{align*}
\KL(p\|q) - \KL(p\|r) - \KL(r\|q)
&=
\int p \log\frac{p}{q}\,dx
-
\int p \log\frac{p}{r}\,dx
-
\int r \log\frac{r}{q}\,dx \\
&=
\int p\left(\log\frac{p}{q}-\log\frac{p}{r}\right)\,dx
-
\int r \log\frac{r}{q}\,dx \\
&=
\int p \log\frac{r}{q}\,dx
-
\int r \log\frac{r}{q}\,dx \\
&=
\int (p-r)\log\frac{r}{q}\,dx \\
&=
\left\langle p-r,\ \log\frac{r}{q}\right\rangle.
\end{align*}
This proves \eqref{eq:kl_three_point_log}. To obtain the gradient form \eqref{eq:kl_three_point_grad}, note that for $\KL(r\|q)=\int r\log(r/q)\,dx$ the pointwise functional derivative with respect to $r$ is
\begin{align*}
\nabla_r \KL(r\|q)(x)=\log\frac{r(x)}{q(x)}+1,
\end{align*}
so
\begin{align*}
\left\langle p-r,\ \nabla_r \KL(r\|q)\right\rangle
=
\left\langle p-r,\ \log\frac{r}{q}\right\rangle
+
\left\langle p-r,1\right\rangle.
\end{align*}
Finally, $\langle p-r,1\rangle=\int(p-r)\,dx = 1-1=0$, yielding \eqref{eq:kl_three_point_grad}.
\end{proof}

\begin{lemma}[Lemma 15 in \citet{song2022hybrid}]
\label{lem:C1-song}
Fix any $R>0$, $\delta \in (0,1)$, and assume we have a class of real-valued functions
$\mathcal{H}:\mathcal{X}\to[-R,R]$.
Suppose we have $K$ i.i.d.\ samples $\{(x_k,y_k)\}_{k=1}^K$ where $x_k\sim \rho$ and $y_k$ is sampled via
the conditional probability $p(\cdot\mid x_k)$:
\begin{align*}
y_k \sim p(\cdot \mid x_k) := h^\star(x_k) + \epsilon_k,
\end{align*}
where $h^\star \in \mathcal{H}$ and $\{\epsilon_k\}_{k=1}^K$ are independent random variables such that
$\mathbb{E}[y_k \mid x_k] = h^\star(x_k)$.
Additionally, suppose that $\max_k |y_k| \le R$ and $\max_x |h^\star(x)| \le R$.
Then the least squares solution
\begin{align*}
\hat h \leftarrow \arg\min_{h\in \mathcal{H}} \sum_{k=1}^K \bigl(h(x_k) - y_k\bigr)^2
\end{align*}
satisfies, with probability at least $1-\delta$,
\begin{align*}
\mathbb{E}_{x\sim\rho}\!\left[\bigl(\hat h(x) - h^\star(x)\bigr)^2\right]
\;\le\;
\frac{256 R^2 \log\!\bigl(2|\mathcal{H}|/\delta\bigr)}{K}.
\end{align*}
\end{lemma}

\noindent
The proof is the same as in \citet{song2022hybrid} and thus is omitted here.

%% file: sections_v2/appendix_additional_backgen.tex
This appendix provides some additional discussion on the back-generalization phenomenon with more results and full details for our experiments in Section~\ref{sec:back_generalization}.

\subsection{Back-generalization from prefixed to no-prefix problems goes beyond stitching.}
\label{app:backgen_is_not_stitching}

Consider a straightforward \emph{stitching} argument as an explanation for back-generalization~\citep{zhang2025bread,qu2025pope}. According to this, the model simply learns to complete from the off-policy intermediate states better without updating the next-token distributions on no-prefix problems, but still improves performance on no-prefix problems when the model happens to sample the same intermediate states on its own. Note that this argument still holds in the tabular policy setting. In contrast, we find that back-generalization indeed influences next-token distributions on untrained states (no-prefix problems), which can only occur through favorable function approximation in LLMs.

\subsection{Hard problems used in the in-context back-generalization experiment}
\label{app:backgen_icl_problems}

Section~\ref{sec:back_generalization} studies a meta-learning style setting where the policy is trained on a \emph{prefixed problem} consisting of an in-context example (a full solved hard problem) followed by a target hard problem. We compare transfer when the in-context hard problem is structurally similar to the target problem versus when it is unrelated (Figure~\ref{fig:icl_expt}). We use the following three problems:

\begin{AIbox}{Hard Problem P1 (Pass@16 for base model is 0.119 and is  different from P3)}
A sphere tangent to the $xy$-plane has center having $z$-coordinate $>0$. If it is projected from $P=(0,b,a)$ to the $xy$-plane, it gives the conic section $y=x^2$. If $a=\nicefrac{p}{q}$ for integers $p,q$ what is $p+q$? Answer: 3.
\end{AIbox}

\begin{AIbox}{Hard Problem P2 (Pass@16 for base model is 0.074 and is similar to P3)}
League has 30 teams (East 16, West 14). Inside each division everyone has played others once. If we add interleague games, what is smallest
$k$ for which every team gets exactly 
$k$ games? Answer: 29.
\end{AIbox}

\begin{AIbox}{Hard Problem P3 (Pass@16 for base model is 0.063)}
Amongst 300 people, no one has more than $k-1$ friends. What is the smallest $k$ for which it might be impossible to create some new friendships so that everyone has exactly $k$ friends? Answer: 151.
\end{AIbox}

\highlight{Relatedness criterion (P2 is similar to P3; P1 is different from P3).}
P2 and P3 are both naturally expressed as graph feasibility problems with degree constraints, and their solutions rely on reasoning about global constraints induced by local degree requirements (regularity, parity, and obstruction arguments). In contrast, P1 is a geometry and conic projection problem whose solution structure does not share this graph-theoretic scaffold. Figure~\ref{fig:icl_expt} uses this controlled notion of similarity to separate two effects: when the prefix and suffix share a compatible representation (P2$\rightarrow$P3), training on prefixed problems yields substantially stronger transfer than when they do not (P1$\rightarrow$P3).

\subsection{Back-generalization under model-family mismatch}
\label{app:backgen_model_mismatch}

\begin{figure*}[!h]
    \centering
    \begin{subfigure}[t]{0.32\linewidth}
        \centering
        \includegraphics[width=\linewidth]{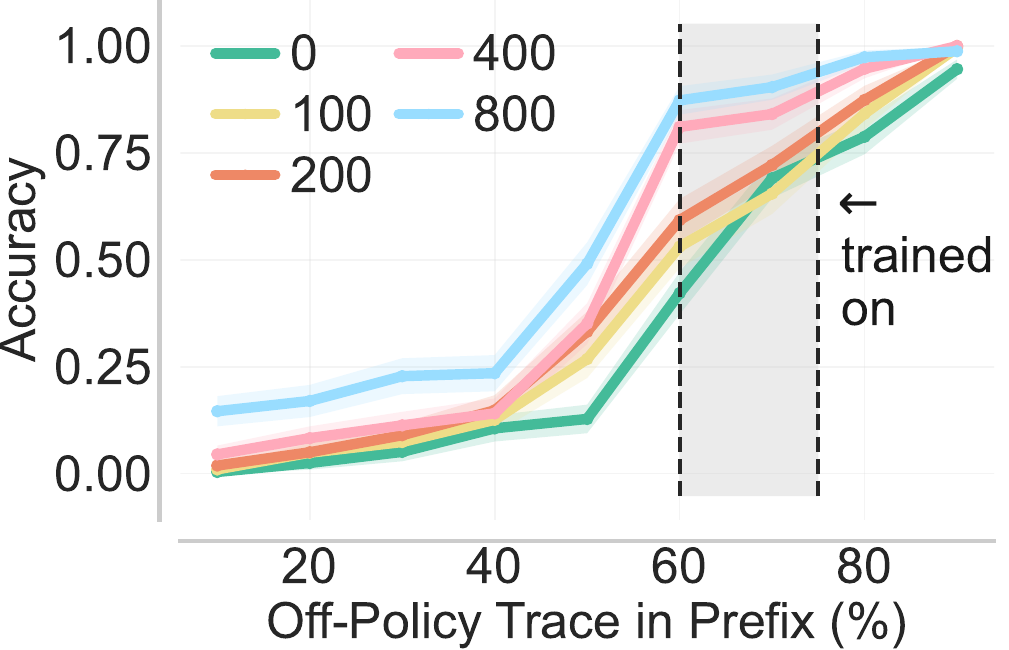}
        \caption*{(a)}
    \end{subfigure}\hfill
    \begin{subfigure}[t]{0.32\linewidth}
        \centering
        \includegraphics[width=\linewidth]{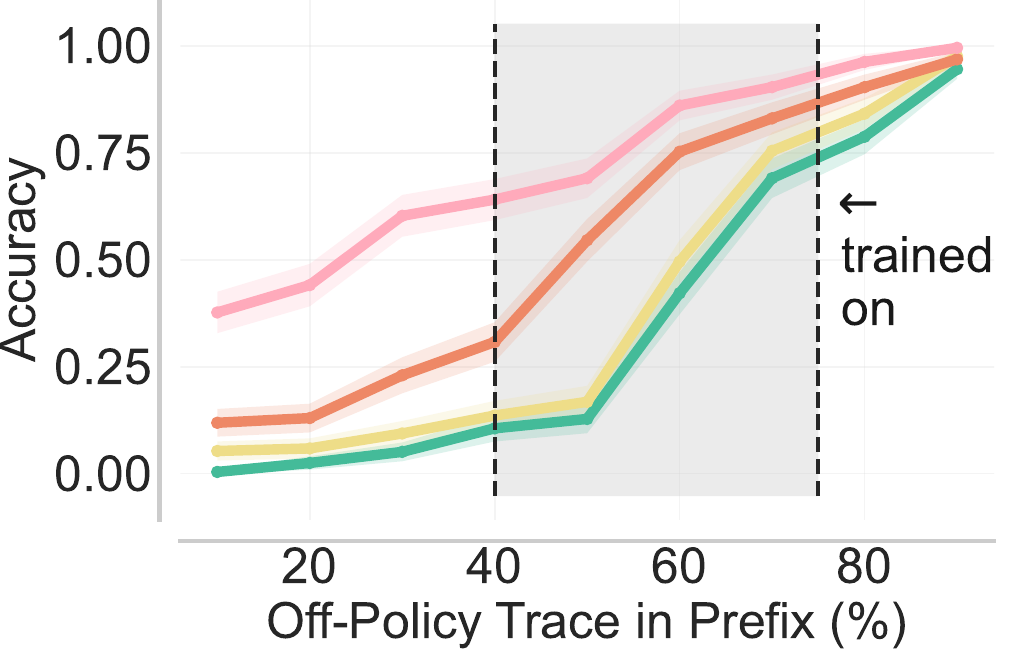}
        \caption*{(b)}
    \end{subfigure}\hfill
    \begin{subfigure}[t]{0.32\linewidth}
        \centering
        \includegraphics[width=\linewidth]{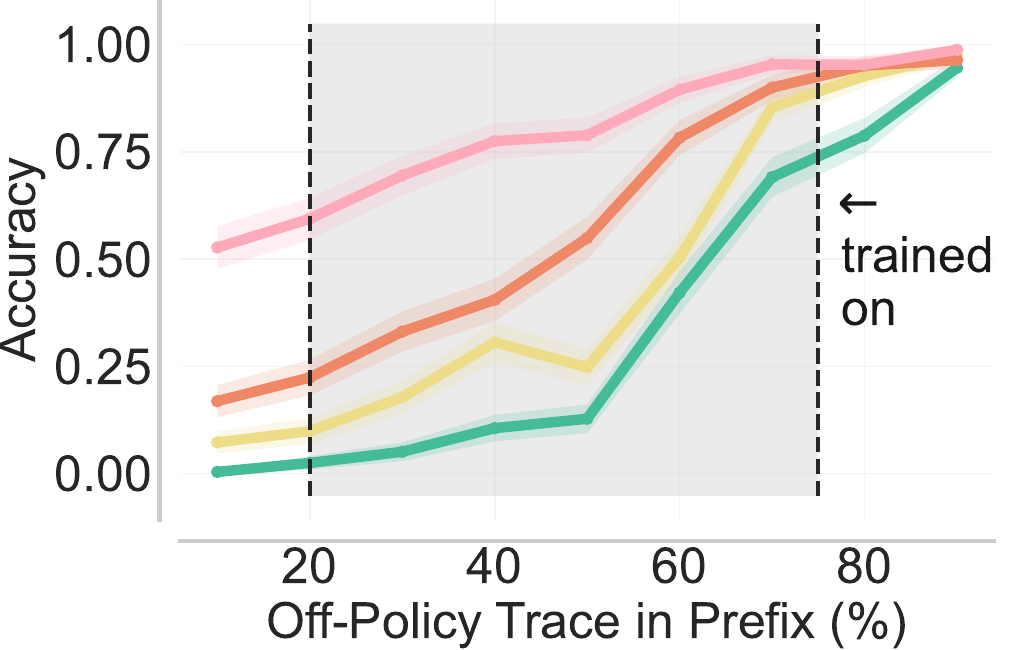}
        \caption*{(c)}
    \end{subfigure}
    \vspace{-0.1in}
    \caption{\textbf{Back-generalization (train-test mismatch):} On Llama3.1-8b-instruct we run RL only on prefixed problems sourced from Qwen3-4b-instruct whose prefix length (percent of the off-policy trace) lies in the shaded interval. At evaluation, we test across the full range of prefix lengths, including no-prefix problems (0\% prefix). The performance at different RL training iterations (0, 100, 200, 400 and 800) is represented with different colors. Similar to Figure~\ref{fig:back-gen}, where the prefixes are also sourced from Llama-3.1-8b When the mismatch is moderate, training on longer prefixes improves performance on shorter prefixes and can eventually improve no-prefix, indicating back-generalization (b, c). But different from Figure~\ref{fig:back-gen}, we find that when the prefix distribution is skewed towards long prefixes back-generalization is weaker despite running RL for 800 iterations.}
    \label{fig:back-gen-qwen}
    \vspace{-0.1in}
\end{figure*}

This section reports an additional train test mismatch experiment analogous to Figure~\ref{fig:back-gen}, but where the off-policy prefixes used to form prefixed problems are sourced from a different model family. Concretely, we train Llama3.1-8B-Instruct with PrefixRL while constructing prefixed problems from correct off-policy traces generated by Qwen3-4B-Instruct.

\highlight{Setup.}
We first collect a pool of correct off-policy traces from Qwen3-4B-Instruct. Each trace induces a family of \emph{prefix states} by truncating the trace to a specified prefix length (reported as a percentage of the full trace). A prefixed problem is formed by conditioning the training policy on such a prefix state and then asking it to complete the solution for the same underlying problem. During training, we restrict prefix lengths to lie in a band (the shaded interval in Figure~\ref{fig:back-gen-qwen}), and we run on-policy RL \emph{only} on these prefixed problems. At evaluation, we sweep prefix lengths across the full range, including the no-prefix endpoint (0\% prefix), and report accuracy at multiple RL training iterations.

\highlight{Cross-family prefixes can still induce back-generalization.}
Figure~\ref{fig:back-gen-qwen} shows that cross-family prefixes can still induce back-generalization, but the effect is less robust than in the same-family setting of Figure~\ref{fig:back-gen}. When the mismatch is moderate (training includes prefixes that are not concentrated exclusively at the very end of traces), improvements appear near the trained prefix region and then propagate to shorter prefixes, eventually improving no-prefix performance (Figure~\ref{fig:back-gen-qwen}b,c). In contrast, when the training prefix distribution is skewed toward long prefixes, back-generalization is weaker, and the transfer to shorter prefixes and to no-prefix remains limited even after long training (800 iterations) (Figure~\ref{fig:back-gen-qwen}a).

\highlight{Choosing the right prefix distribution.}
A key difficulty in this setting is that the training prefixes are not sampled from the current policy. They are injected from an external model (Qwen), and therefore correspond to states that can be very unlikely under the evolving Llama policy. When training concentrates on very late prefixes, the policy can improve primarily on a narrow slice of the prefixed state distribution without sufficiently shaping behavior on earlier states that dominate no-prefix rollouts. This makes cross-family prefixing most effective when the training mixture covers a sufficiently broad range of prefix lengths, rather than concentrating only on long, heavily conditioned prefixes.

 \subsection{Mental Model: Function Approximation and Back-Generalization.} 
\label{app:mental_model}

\begin{figure*}[!t]
    \centering
    \begin{subfigure}[t]{0.45\linewidth}
        \centering
        \includegraphics[width=\linewidth]{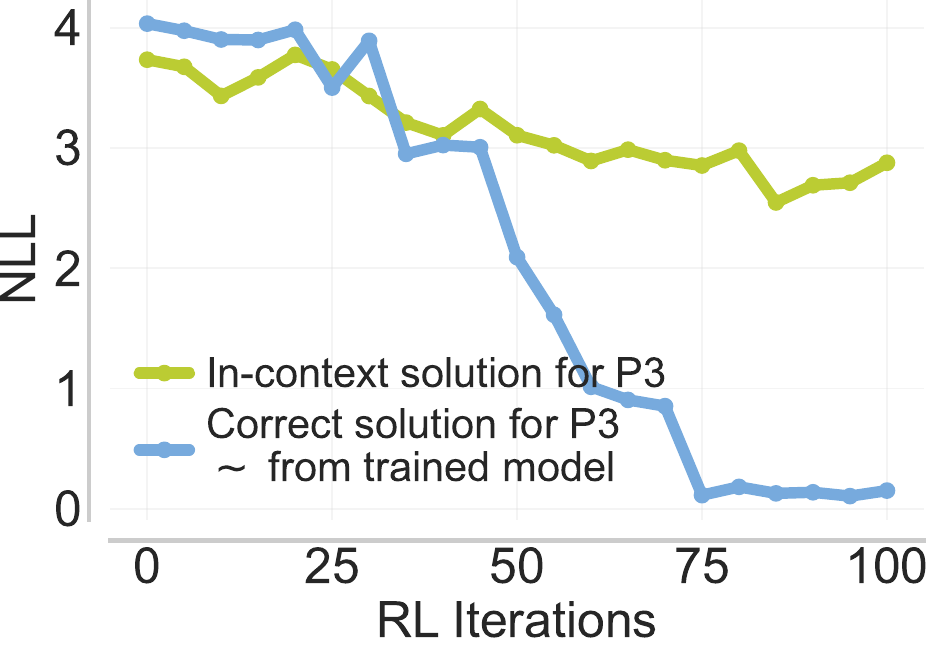}
        \caption*{\vspace{-0.05in}(a)}
    \end{subfigure}
    \vspace{-0.1in}
    \caption{\textbf{(a) Likelihood of the in-context solution:} When asked to solve problem P3, we measure the negative log-likelihood (NLL) of the in-context solution for P3 provided in the context for P2, \textit{i.e.}, when we run PrefixRL on P2 $\mid$ P3 (setup in Section~\ref{subsec:backgen_icl}). Surprisingly, we find that the likelihood of the in-context trace drops less compared to a correct trace for P3 we sample from the final checkpoint. This indicates that back-generalization does not exactly clone the in-context prefix to improve performance on the no-prefix counterpart.}
    \label{fig:nll}
    \vspace{-0.2in}
\end{figure*}

 \highlight{NLL of the in-context solution changes little.} A natural hypothesis is that improvements on the untrained but in-context problem P3 come from memorizing the in-context trace and replaying it at test time. Instead, because this trace is extremely unlikely and never directly trained on, the model does not learn to imitate it (we also saw an example of this in Figure~\ref{fig:synthetic_expt}(b)). In Figure~\ref{fig:nll}, the negative log-likelihood of the in-context P3 solution barely decreases under RL on P2$\mid$P3; the final policy instead prefers a different token sequence that still yields the correct answer. Together with the correlation in Figure~\ref{fig:synthetic_expt}, this suggests strong similarity between prefixed and no-prefix solutions, but weak similarity to the specific off-policy prefix used in the prefixed problem.
 
 Although the model does not clone the in-context off-policy trace, performance still transfers to the in-context problem. Our speculation is that for long chain-of-thought rollouts, “state” is better viewed as the model’s internal representation induced by the prefix: because the model self-corrects and backtracks, many distinct token sequences can map to similar latent states. Thus, while solving the prefixed problem, the policy can backtrack into representations close to those encountered when solving the original problem directly, but now with positive reward. If the history-conditioned and non-history-conditioned representations are close, rewards observed in the former will shift the next-token distribution in the latter. This explains the overlap between prefixed and no-prefix responses (Figure~\ref{fig:synthetic_expt}), why transfer is stronger for related pairs like P2$\mid$P3 than unrelated ones like P2$\mid$P1 (Figure~\ref{fig:icl_expt}), and why the model still may not learn to reproduce the literal off-policy prefix when its rephrased representation is far from the original context (Figure~\ref{fig:back-gen-qwen}(a)).